%% file: main.tex
\documentclass[10pt,twocolumn,letterpaper]{article}

\usepackage{tikz}
\usepackage{graphicx}
\usepackage{caption}
\usepackage{multirow} 
\usepackage{pifont}
\usepackage{ifsym}
\usepackage{amssymb,amsthm,mathrsfs}
\usepackage{geometry}
\newtheorem{theo}{Theorem}
\newtheorem{defi}{Definition}

\newtheorem{lem}{Lemma}

\usepackage{listings}

\definecolor{codegreen}{rgb}{0,0.6,0}
\definecolor{codegray}{rgb}{0.5,0.5,0.5}
\definecolor{codepurple}{rgb}{0.58,0,0.82}
\definecolor{backcolour}{rgb}{0.95,0.95,0.92}
\usepackage{color, colortbl}

\definecolor{lightgray}{gray}{0.9}

\usepackage{algorithm}
\usepackage[noend]{algpseudocode}
\usepackage{subfigure}
\usepackage{booktabs} 
\usepackage{booktabs,tabularx}
\usepackage{wrapfig}
\usepackage{amsmath}

\lstdefinestyle{mystyle}{
    backgroundcolor=\color{backcolour},   
    commentstyle=\color{codegreen},
    keywordstyle=\color{magenta},
    numberstyle=\tiny\color{codegray},
    stringstyle=\color{codepurple},
    basicstyle=\ttfamily\footnotesize,
    breakatwhitespace=false,         
    breaklines=true,                 
    captionpos=b,                    
    keepspaces=true,                 
    numbers=left,                    
    numbersep=5pt,                  
    showspaces=false,                
    showstringspaces=false,
    showtabs=false,                  
    tabsize=2
}
\usepackage{tcolorbox}

\makeatletter
\newcommand*{\citelinktext}[2]{%
  \hyper@@link[cite]{}{cite.#1}{#2}}
\makeatother

\usepackage{cvpr}              


\definecolor{cvprblue}{rgb}{0.21,0.49,0.74}
\usepackage[pagebackref,breaklinks,colorlinks,allcolors=cvprblue]{hyperref}

\def\paperID{} 
\def\confName{CVPR}
\def\confYear{2025}

\title{NoT: Federated Unlearning via Weight Negation}

\author{Yasser H. Khalil$^{1*}$
\hspace*{1em}
Leo Brunswic$^{1*}$
\hspace*{1em}
Soufiane Lamghari$^{1}$
\hspace*{1em}
Xu Li$^{2}$
\hspace*{1em}
Mahdi Beitollahi$^{1}$
\hspace*{1em}
Xi Chen$^{1}$
\\
$^{1}$Huawei Noah's Ark Lab, Montreal, Canada
\hspace*{1em}
$^{2}$Huawei Technologies Canada Inc., Ottawa, Canada
\\
}

\begin{document}
\maketitle

\renewcommand{\thefootnote}{\fnsymbol{footnote}}
\footnotetext[1]{Equal contribution.}
\footnotetext{Correspondence to: Yasser H. Khalil (yasser.khalil1@huawei.com).}
\footnotetext{Accepted at the 42nd IEEE/CVF Conference on Computer Vision and Pattern Recognition, Nashville TN, US. 2025.}

\input{sec/0_abstract}    
\input{sec/1_intro}
\input{sec/2_related_work}
\input{sec/3_preliminaries}
\input{sec/4_unlearning_framework}
\input{sec/5_not}
\input{sec/6_experiments}
\input{sec/7_limitations}
\input{sec/8_conclusion}
{
    \small
    \bibliographystyle{ieeenat_fullname}
    \bibliography{main}
}

\input{sec/X_suppl}

\end{document}

%% file: sec/0_abstract.tex
\begin{abstract}
Federated unlearning (FU) aims to remove a participant’s data contributions from a trained federated learning (FL) model, ensuring privacy and regulatory compliance. Traditional FU methods often depend on auxiliary storage on either the client or server side or require direct access to the data targeted for removal—a dependency that may not be feasible if the data is no longer available. To overcome these limitations, we propose \textbf{NoT}, a novel and efficient FU algorithm based on weight negation (multiplying by -1), which circumvents the need for additional storage and access to the target data. We argue that effective and efficient unlearning can be achieved by \textit{perturbing model parameters away from the set of optimal parameters, yet being well-positioned for quick re-optimization}. This technique, though seemingly contradictory, is theoretically grounded: we prove that the weight negation perturbation effectively disrupts inter-layer co-adaptation, inducing unlearning while preserving an approximate optimality property, thereby enabling rapid recovery. Experimental results across three datasets and three model architectures demonstrate that NoT significantly outperforms existing baselines in unlearning efficacy as well as in communication and computational efficiency. \end{abstract}

%% file: sec/1_intro.tex
\section{Introduction}
\label{sec:intro}

Federated learning (FL) enables decentralized machine learning across distributed devices, allowing models to be trained collaboratively without sharing raw data, thus enhancing privacy and security \citep{ nguyen2021federated, zhang2022federated, khalil2024dfml}. However, growing concerns over privacy and data security have emphasized the need for unlearning techniques to meet evolving regulatory standards \citep{xu2024machine, liu2024threats}. Federated unlearning (FU) addresses this by enabling removal of individual data contributions from trained FL models \citep{romandini2024federated,liu2023survey}. This capability is essential for privacy preservation and compliance with regulations such as GDPR \citep{union2023complete}, which mandates the “right to be forgotten.” FU is also critical when data is outdated, compromised, or subject to data poisoning attacks \citep{wan2024data}. 

\begin{figure}
\centering
\includegraphics[trim={0.25cm 0.19cm 0.2cm 0.17cm},clip,scale=0.7]{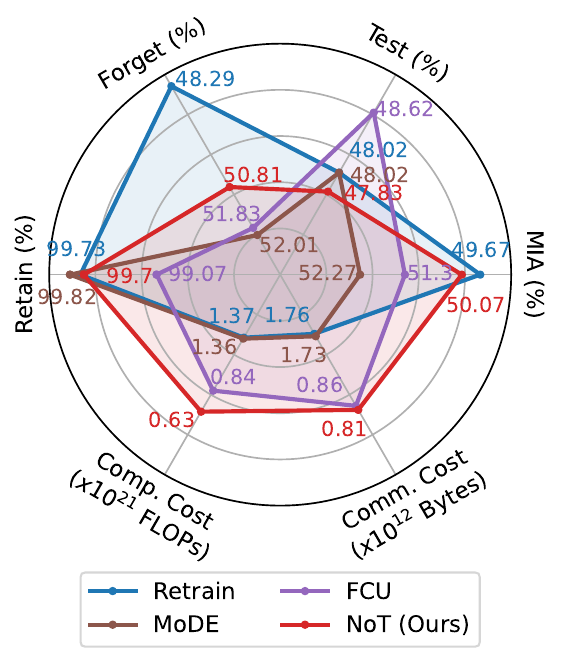}
\caption{\textbf{Performance comparison of NoT with baselines} using ViT-B/16 on Caltech-101 in a 10-client setup, where one client requests unlearning. The ideal federated unlearning algorithm should closely approximate the performance of the “gold standard” (Retrain) across key accuracy metrics: \textit{retain}, \textit{forget}, \textit{test}, and \textit{MIA}, while minimizing communication and computation overhead. As illustrated, NoT's performance closely matches that of Retrain across all metrics with minimal added costs, underscoring NoT’s efficacy and efficiency in federated unlearning. Experimental details and further comparisons can be found in Section~\ref{sec_experimens}.}
\label{fig_sample_result}
\end{figure}

\begin{figure*}
\centering
\includegraphics[trim={0.7cm 0.17cm 0.7cm 0.14cm},clip,scale=0.79]{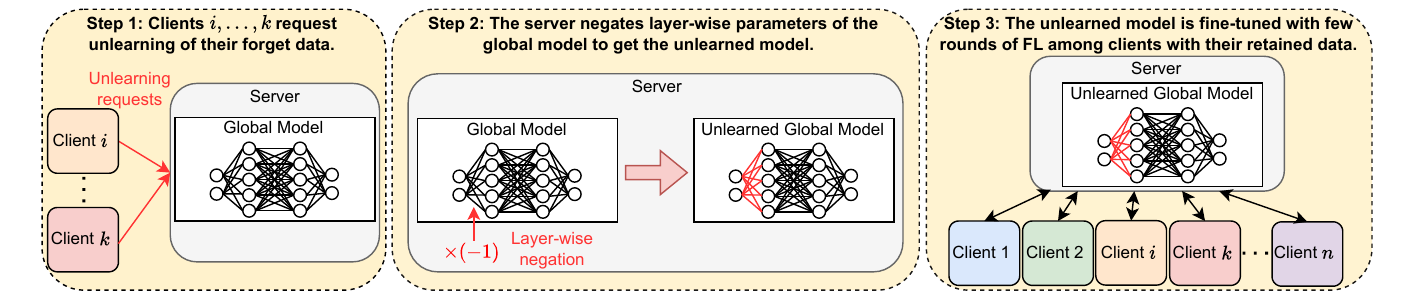}
\vspace*{-6mm}
\caption{\textbf{NoT overview.} Upon receiving unlearning requests from target clients, the server initiates the unlearning process by applying layer-wise parameter negation to the global model. This negation disrupts inter-layer co-adaptation, effectively inducing unlearning. Subsequent fine-tuning rounds restore essential knowledge. If a client wishes to forget all its data (i.e., client-wise forgetting), it does not participate in fine-tuning. Conversely, if a client wants partial data forgetting (i.e., class-wise or instance-wise forgetting), it fine-tunes the global model using its retained data.}
\label{fig_NoT_pipeline}
\end{figure*}

While various FU techniques have been proposed \citep{liu2021federaser, wu2022federated, halimi2022federated, zhao2023federated}, they face significant challenges in FL environments. For instance, exact unlearning methods such as retraining from scratch guarantee thorough data removal but are impractical due to high communication and computational demands. Other FU approaches require additional storage for model updates, which may be infeasible and could pose additional security risks. Additionally, many existing methods depend on having access to the target data, which may no longer be available or permissible for use.

To address these limitations, we propose NoT, a novel and efficient FU algorithm based on weight negation, requiring neither auxiliary storage nor access to the target data. NoT operates by negating (multiplying by -1) the parameters of specific layers in the global model, as depicted in Figure~\ref{fig_NoT_pipeline}. The name \textbf{NoT} reflects the Boolean “NOT” operation. Negation breaks inter-layer co-adaptation (\ie, dependencies between network parameters \cite{hinton2012improving,sato2019breaking}, see Appendix~\ref{appendix:coadapt}), resulting generically in high loss, which effectively “forgets” the targeted data. Fine-tuning on retained data subsequently allows for recovery of essential knowledge. Together, these two phases induce significant parameter changes that remove knowledge of the data to be unlearned. To formalize this approach, we incorporate NoT within a novel theoretical unlearning framework, establishing that effective and efficient unlearning can be achieved by \textit{perturbing model parameters away from the optimal set of parameters, yet being well-positioned for re-optimization}.
Our empirical results confirm that NoT achieves effective unlearning while minimizing communication and computation compared to state-of-the-art methods. As illustrated in Figure~\ref{fig_sample_result}, NoT achieves performance close to the “gold standard” Retrain method across several accuracy metrics, while significantly reducing communication and computation costs. Moreover, since NoT does not rely on access to the forgot data, it naturally supports client-wise, class-wise, and instance-wise forgetting. Our main \textbf{contributions} are summarized as follows:
\begin{itemize}
\item We propose NoT, an efficient FU algorithm leveraging weight negation, operating without requiring additional storage or access to target data.
\item We present a theoretical framework describing how unlearning is achieved through weight perturbation and fine-tuning. We provide an effective bound controlling the unlearning via fine-tuning, introduce the notion of layer-wise optimality which enables fast-recovery, and prove that weight negation conforms to this framework. Further, we empirically validate our theoretical predictions.
\item We conduct an extensive experimental evaluation of NoT, benchmarking it against seven FU methods across three datasets and three model architectures. Our evaluation includes experiments addressing backdoor attacks, tests in centralized settings with eight baselines, and an ablation study.

\end{itemize}

%% file: sec/2_related_work.tex
\section{Related Work}
\label{sec_releated_works}

\paragraph{Federated Unlearning (FU).} Two main approaches dominate FU \cite{liu2021federaser, wu2022federated, ren2024advances, gao2024verifi, halimi2022federated, zhao2023federated, deng2024enable, zhang2023fedrecovery, pan2024federated}: \ding{182} \textbf{Storing historical updates}: These methods save prior model updates for later use in unlearning specific data. For instance, FedEraser \citep{liu2021federaser} retrieves the global model state before a client joins the federation, sharing it with the remaining clients to remove the target client's influence. FUKD \citep{wu2022federated} subtracts the target client's updates and uses knowledge distillation to restore model performance. However, these approaches pose privacy risks due to potential model update leakage \citep{ren2024advances}, face storage limitations, and often require unlabeled data, which may not always be available. Despite these challenges, they do not require the target client's participation during unlearning. \ding{183} \textbf{Gradient modification}: These methods alter model gradients during training to suppress the impact of target data. For example, PGD \citep{halimi2022federated} reverses the learning process for the target client, constraining the update within an $\ell_2$-norm ball around a reference model, which is then fine-tuned by the remaining clients. MoDE \citep{zhao2023federated} uses a randomly initialized degradation model for unlearning, while FCU \citep{deng2024enable} simulates a model that has never seen the forgotten data by applying contrastive loss and preserving low-frequency components of the global model. These methods, however, often require the target client's involvement and impose high computational costs. In contrast, \textit{NoT eliminates the need for extra storage or access to the target data}.

\paragraph{Unlearning via Weight Perturbation.} Several works focus on unlearning through weight modification. Golatkar \etal \cite{golatkar2020eternal} proposes adding Gaussian noise, computed using Fisher information, to disrupt weights for unlearning. In a follow-up, Golatkar \etal \cite{golatkar2020forgetting} employ the Neural Tangent Kernel (NTK) to address the null space of model weights. Tarun \etal \cite{tarun2023fast} introduces an error-maximizing noise matrix trained on a pretrained model to corrupt weights associated with specific target classes. A more targeted approach, SSD \citep{foster2024fast}, dampens parameters deemed sensitive to forgotten data using Fisher information.  Most weight perturbation methods that focus on class-based forgetting, however, struggle with random data forgetting and are computationally expensive due to Fisher information calculations. \textit{To our knowledge, no prior work explores unlearning through weight negation.} In our results, we compare NoT’s negation approach with other perturbation methods.  

%% file: sec/3_preliminaries.tex
\section{Preliminaries}
\label{sec_background}
We define a {\it model} as a parameterized family $\mathcal N^\theta : \mathbb R^{d_{\mathrm{in}}}\rightarrow \mathbb R^{d_{\mathrm{out}}}$, where $\theta\in \Theta\subset \mathbb R^d$, such that the map $(x,\theta)\mapsto \mathcal N^{\theta}(x)$ is continuous and piecewise twice continuously differentiable. We focus on models provided by neural networks, the parameters tensor may then be written as $\theta= (\theta_{\ell})_{\ell \in \mathscr L}$ where $\mathscr L$ is the set of layers. We assume a dataset-dependent non-negative loss function, $\mathcal L_D(\mathcal N):=\mathbb E_{(x,y)\sim D} L(\mathcal N(x),y)$, and that the model is typically trained via gradient descent to minimize this loss. For simplicity, we denote the loss as $\mathcal L_D(\theta)$ instead of $\mathcal L_D(\mathcal N^{\theta})$, when the context clearly refers to the model $\mathcal N^\theta$. 

In this paper, we consider an FL system with $n$ clients, where each client $k\in \mathcal{P}=\{1,\ldots,n\}$ has local training data $D^k$. Clients collaboratively train a global model until convergence $\mathcal N^{\theta^*}$ using a standard FL algorithm, such as FedAvg. After training, each client $k$ may request the server to unlearn a subset of its data $D_u^k \subseteq D^k$, referred to as the target or forget data, while $D_r^k:=D^k\setminus D_u^k$ represents its retained data. The client requesting unlearning is called the target client. A straightforward solution to unlearn $D_u=\bigcup_{k}D^k_u$ is to retrain the model from scratch using distributed $D_r=\bigcup_{k}D^k_r$, resulting in a retrained model. However, this approach is computationally and communication-intensive. \textit{The challenge is to efficiently find an unlearned model that closely approximates the performance of the retrain model.}

%% file: sec/4_unlearning_framework.tex
\section{Unlearning Framework}
\label{sec:framework}

Given a trained model $\mathcal N^{\theta^*}$, our objective is to efficiently unlearn data without retraining from scratch. We propose a two-step approach: first, \textbf{perturbing} model parameters to obtain perturbed parameters $\theta'$, followed by \textbf{fine-tuning} using gradient descent starting at $\theta^0=\theta'$ to minimize $\mathcal L_{D_r}$, resulting in the unlearned model. The underlying intuition behind this approach is that perturbing the model induces not only immediate unlearning but also large gradients for the subsequent fine-tuning phase, hence substantial alteration of the model's internal configuration and promoting further unlearning. However, the perturbation should be designed to avoid an excessive fine-tuning phase.
\vspace*{3mm}
\begin{tcolorbox}[width=\linewidth, before skip=-2.1mm, after skip=0.2cm, boxsep=0.0cm, middle=0.1cm, top=0.1cm, bottom=0.1cm]
Informally, we seek a perturbation that is:

\noindent\ding{70} \textbf{(C1)} \textit{\textbf{Strong}: significantly pushes model parameters away from optimal configurations.} \\
\noindent\ding{70} \textbf{(C2)} \textit{\textbf{Resilient}: enables fast re-optimization.}
\end{tcolorbox}
\subsection{The Need for a Strong Perturbation}
To motivate condition \textbf{C1}, we introduce the concept of \textbf{loss gap} as a measure for unlearning. 

\begin{defi}[Loss gap] Let $\mathcal N^{\theta}$ be a model, and let $(D_r,D_u)$ be a pair of datasets. The \textbf{loss gap} is defined as: 
\begin{equation}
\delta (\theta) := |\mathcal L_{D_r}(\theta) - \mathcal L_{D_u}(\theta) |.
\end{equation}
\end{defi}
\vspace{-1.7mm}
In unlearning, the goal is to increase the loss gap by a target amount, ensuring that $\mathcal L_{D_r}$ is minimized while $\mathcal L_{D_u}$ is not. The rationale behind \textbf{C1} is supported by the following theorem:

\begin{theo}\label{theo:unlearning_time_constraint} Let $\mathcal N^{\theta}$ be a model, and let $(D_r,D_u)$ be a pair of datasets. Given an initial parameter set $\theta^0\in \Theta$, assume $\mathcal N^{\theta^0}$ is trained using Stochastic Gradient Langevin Descent\footnote{SGLD may be seen as an approximation of SGD, see \citep{stephan2017stochastic}.} to minimize $\mathcal L_{D_r}$ starting from $\theta^0$. The parameter evolution is given by: $d\theta^t  = -\nabla_{\theta^t} \mathcal L dt+\Sigma(\theta^t,t)\cdot dW$. At any training time $t\geq 0$, the following holds:\footnote{Expectations taken over the randomness of the stochastic process $\theta^t$.}
\begin{equation}\label{eq:unlearning_time}
t \geq \frac{\mathbb E( \delta ({\theta^t}) - \delta(\theta^0) )^2 }{L^2\left[|\mathcal L_{D_r}(\theta^0)-\mathbb E\mathcal L_{D_r}(\theta^t)|+A\right]},
\end{equation}
with $L:=\sup_{\theta_1\neq \theta_2}\frac{|\delta(\theta_1)-\delta(\theta_2)|}{\|\theta_1-\theta_2\|}$ and $A$ depends explicitly on $\Sigma$ and the Hessian of $\mathcal L_{D_r}$ along the solution, and $A=0$ when $\Sigma\equiv 0$ (see Appendix \ref{unelarning_time_lowerbound} for proof and details).\end{theo}

Here, the left side of the inequality represents the time $t$ required for unlearning via gradient descent. The numerator on the right indicates the unlearning target, while the denominator depends on datasets characteristics, loss function and training stochasticity. This theorem suggests that unlearning may be {\it slow} if some conditions are not met. For instance, natural forgetting ($\theta^0=\theta'=\theta^*$) leads to small $L$ due to minor statistical differences between $D_r$ and $D_u$, and a negligible $|\mathcal L_{D_r}(\theta^0)-\mathbb E\mathcal L_{D_r}(\theta^t)|$ as $\mathcal N^{\theta^0}=\mathcal N^{\theta^*}$ is already converged. Therefore, Theorem~\ref{theo:unlearning_time_constraint} applied to the following fine-tuning on $D_r$ predicts extended unlearning time, hence slow natural forgetting. See Appendix~\ref{sec:unlearning_time_quantitative} for quantitative estimation. For fast unlearning, factors such as a large Hessian spectrum, high loss, or high stochasticity during descent are essential, hence a \textbf{strong} perturbation that increases loss is beneficial.

\subsection{What is a Resilient Perturbation?}
\label{sec:optmizable_state}
There are various strong perturbations (\eg, randomization of $\theta$), but they may require intensive fine-tuning. Therefore, condition \textbf{C2} seeks to ensure that the perturbed set of parameters $\theta'$ is in a good optimization state, allowing efficient recovery during fine-tuning. A good optimization state could be defined {\it a posteriori} as one where $t^*_{\varepsilon}:=\min\{t~|~ \mathcal L(\theta^t)-\mathcal L(\theta^*)\leq \varepsilon\}$ is small, relative to a baseline (such as retraining from scratch). However, we favor {\it a priori} properties of the neural network at training time $t=0$ yielding (statistically) a small $t^*_{\varepsilon}$. Existing evidence suggests that certain properties ease the optimization of a model:

\noindent\tikz[baseline=(char.base)] \node[draw, circle, fill=black, text=white, inner sep=0.75pt](char){a}; \textbf{Jacobian Control}: typically analyzed through the spectrum of $(\nabla_x\mathcal N^\theta)^T( \nabla_x \mathcal N^{\theta})$ or $( \nabla_\theta\mathcal N^\theta)^T( \nabla_\theta \mathcal N^{\theta})$, it controls gradient back-propagation. Accumulation of the eigenvalues of the former around 1 is coined dynamical isometry \citep{hochreiter1998vanishing, saxe2013exact}. The spectrum of the latter relates to the spectrum of the Fisher Information Matrix (FIM) \citep{amari1997information, rattray1998natural}. Dynamical isometry is expected in networks with residual connections \citep{tarnowski2019dynamical,noci2022signal,bachlechner2021rezero} or well-chosen initializations \citep{pmlr-v80-xiao18a,burkholz2019initialization,chen2018dynamical,pennington2017resurrecting}. It controls the FIM spectrum \citep{pmlr-v89-karakida19a,liao2018approximate} which in turn impacts gradient dynamics via control over the stochasticity of gradient descent \citep{stephan2017stochastic} and isospectrality to the Neural Tangent Kernel \citep{jacot2022theory} in the mean field limit.

\noindent\tikz[baseline=(char.base)] \node[draw, circle, fill=black, text=white, inner sep=0.1pt](char){b}; \textbf{Model Pretraining} is empirically shown to accelerate fine-tuning \citep{pmlr-v5-erhan09a, erhan2010does, yosinski2014transferable, he2019rethinking, yao2021understanding}. While this is not fully theoretically accounted for, a pretrained model is expected to require less learning of low-level features, reducing the fine-tuning search space compared to training from scratch.

In our framework, denoting $J(\theta,x):= \left.\nabla_\theta \mathcal N^{\theta}\right|_{x}$ and $X$ a random variable sampled from $D_r$, we translate Jacobian control as a control over distributions of random matrices $J(\theta',X)$  and $J(\theta^*,X)$. If, for example, the Wasserstein distance $\mathcal W(J(\theta^*,X),J(\theta',X))$ is sufficiently small \cite{villani2021topics}, gradients should behave comparably during fine-tuning of $\mathcal N^{\theta'}$ relative to the original training which yielded $\mathcal N^{\theta^*}$. A perturbation has {\it Jacobian control} if it satisfies a bound on this Wasserstein distance. Also, since the pre-perturbation model $\mathcal N^{\theta^*}$ is converged, it is effectively pretrained on the dataset $D_r\cup D_u$.
{\it Layer-wise optimality} formalizes the preservation of part of the effective pretraining and ensures a smaller dimension of the fine-tuning search space. 

\begin{defi}[Layer-Wise Optimality] A model $\mathcal N^\theta$ is \textit{layer-wise optimal} (LWO) if, for every chosen layer $\ell$, freezing $\ell$, randomizing all other layers $\ell' \neq \ell$, and fine-tuning still yield an acceptable optimum.

If $\mathcal N^{\theta'}$ remains LWO when $\mathcal N^{\theta}$ is LWO, then the perturbation $\theta \mapsto \theta'$ is \textit{layer-wise optimality preserving} (LWOP). \end{defi}

Together, Jacobian control and layer-wise optimality guarantee accelerated fine-tuning. Thus, a perturbation that is LWOP and has Jacobian control is considered \textbf{Resilient}.
\vspace*{3mm}
\begin{tcolorbox}[width=\linewidth, before skip=-2.1mm, after skip=0.2cm, boxsep=0.0cm, middle=0.1cm, top=0.1cm, bottom=0.1cm]
\textit{To conclude Section~\ref{sec:framework}, achieving both conditions requires perturbing to maximize loss (\textbf{C1}) while controlling the model's Jacobian (\textbf{C2a}) and preserving layer-wise optimality  (\textbf{C2b}).}
\end{tcolorbox}

%% file: sec/5_not.tex
\section{NoT - The Unlearning Algorithm}
\label{sec:negation_theory}

In this section, we present the NoT algorithm and discuss its role as a federated unlearning solution. First, we outline the algorithm's design, process, and advantages. Second, we position NoT within our theoretical framework of perturbation and fine-tuning for unlearning. Finally, we discuss the selection of layers for negation.

\def\Yneg{Y_-}
\def\Lneg{\mathscr{L}_{\mathrm{neg}}}

\subsection{Algorithm Overview} 
\label{subsec:description}

When an unlearning request is received, the server initiates the unlearning process by \textbf{negating} the parameters of specified layers $\mathscr L_{\mathrm{neg}}$ in the converged global model $\mathcal{N}^{\theta^*}$. This produces a perturbed model $\mathcal N^{\theta'}$ with parameters $\theta'$ as follows: \begin{equation}\theta':=(-\theta^*_{\ell})_{\ell \in \mathscr L_{\mathrm{neg}}}\oplus (\theta^*_\ell)_{\ell\in\mathscr L\setminus \mathscr L_{\mathrm{neg}}}.\end{equation} 
NoT then fine-tunes $\mathcal N^{\theta'}$ on the retained data $D_r$, resulting in a final model that excludes target data contributions while preserving essential knowledge. Algorithm~\ref{alg_NoT} in Appendix~\ref{not_algo} details our proposed method via weight negation. A PyTorch implementation is included in Appendix~\ref{code_not}.

NoT presents several a priori advantages: \ding{182} Negation is \textit{computationally negligible}. \ding{183} The target client only needs to signal an unlearning request, thus \textit{communication cost is minimal}. \ding{184} \textit{No additional storage required} on the client or server side. \ding{185} \textit{No access to $D_u$} is needed, allowing target clients to delete it immediately after requesting unlearning. While there are some costs incurred during fine-tuning, they are relatively low as we will show empirically in our experiments. 

\subsection{Negation as a Strong and Resilient Perturbation}  
\label{subsection:negation_perturbation}
The NoT algorithm follows a “perturb then fine-tune” approach. Now, we theoretically ensure that NoT’s negation-based perturbation meets the conditions \textbf{(C1)} and \textbf{(C2)} for effective unlearning. Let $\mathcal N^\theta$ be a neural network model and let $X$ be a random vector following a dataset distribution $\mathcal D$, we denote by $Y_\ell$ the pre-nonlinearity activations of layer $\ell\in \mathscr L$ given input $X$ and define $\Yneg:=\bigoplus_{\ell \in \mathscr L_{\mathrm{neg}}} Y_\ell$. 

\noindent\ding{70} \textbf{(C1)} Under mild assumptions, weight negation is the strongest perturbation as it maximizes changes in the activations of the perturbed layers:

\begin{theo} \label{theo:maximize_dist}
Denote $\sigma(x):=\max(x,0)$ and let $Y_\ell '$ be the output of layer $\ell$ for some perturbation of $\ell$. Assume $\mathbb E\left|\|\sigma(Y_\ell)\|^2-\|\sigma(-Y_\ell)\|^2\right|\leq \varepsilon$ and $\mathbb E\left|\|\sigma(Y_\ell)\|^2-\|\sigma(Y_\ell')\|^2\right|\leq \varepsilon$, then: \begin{equation}
    \mathbb E \|\sigma(Y_\ell)-\sigma(-Y_\ell)\|^2\geq  \mathbb E \|\sigma(Y_\ell)-\sigma(Y_\ell')\|^2 - 2\varepsilon.
\end{equation}
\end{theo}
While the loss may occasionally remain unchanged (\eg, if the next layer is zero), the distinct distribution of $(X,\Yneg)$ from $(X,-\Yneg)$ implies that \textbf{C1} is generally satisfied. See Appendix~\ref{sec:negation_bounds} for proof and additional discussion.

\noindent\ding{70} \textbf{(C2)} The following Theorems ensure conditions \textbf{C2a} and \textbf{C2b}, ensuring resilience for NoT.
\begin{theo}\label{theo:jacobian} Assume $\mathcal N^{\theta}$ is a feedforward \footnote{By feedforward, we mean that the computational graph is a DAG. RNNs are not feedforward, but residual connections are allowed.} neural network of layer poset $(\mathscr L,\leq)$ ordered by the computational graph. Let $J^{\theta}_\ell:= \nabla_{\theta_\ell}\mathcal N^{\theta}(X)$ with $X\sim \mathcal D$ for layer $\ell \in \mathscr L$.  Then:
\begin{align} 
 \mathcal W\left(J^{\theta^*}_\ell;J^{\theta'}_\ell\right) \leq& A_\ell ~  {\mathrm{TV}}(\Yneg; -\Yneg),& \forall\ell> \mathscr L_{\mathrm{neg}}\nonumber\\ 
\mathcal W\left(\epsilon J^{\theta^*}_\ell;J^{\theta'}_\ell\right)\leq& A_\ell ~ \mathrm{TV} ( (X, \Yneg); (X,-\Yneg )),& \forall\ell\leq \mathscr L_{\mathrm{neg}}
\end{align}
%
where $\epsilon=(-1)^{\ell\notin  \mathscr L_{\mathrm{neg}} }$,   $\mathcal W$ and  ${\mathrm{TV}}$ are the  Wasserstein and total variation distances, respectively, and  $(A_\ell)_{\ell \in \mathscr L}$ are positive values. See Appendix~\ref{sec:spectrum} for the proof, technical assumptions, and details on $A_\ell$.
\end{theo}
\begin{theo}\label{theo:negate_C2} 
The negation perturbation is LWOP  if $\mathscr L_{\mathrm{neg}}$ is an antichain of the poset $\mathscr L$ containing no maximal element, and each $\ell\in \mathscr L_{\mathrm{neg}}$ is activated by sigmoid-like, odd, or even functions (\eg, $\mathbf 1_{>0}$, $\tanh$, $\sin$, $x^2$). 
See Appendix~\ref{sec:appendix_pretraining} for details and proof.
\end{theo}  
In the wide network limit, we expect ${\mathrm{TV}}(Y_\ell,-Y_\ell)\ll 1$ for all hidden $\ell$, meaning that Theorem~\ref{theo:jacobian} ensures the preservation of each $J_\ell$ spectrum for $\ell > \mathscr L_{\mathrm{neg}}$ through negation. For $\ell\leq \mathscr L_{\mathrm{neg}}$, gradients are non-exploding but possibly vanishing. Hence, \textbf{C2a}. Theorem~\ref{theo:negate_C2} provides \textbf{C2b} in many cases but does not cover ReLU-like activations, we conjecture that negation is “approximately” LWOP in this case. 

\subsection{Selecting Layers for Negation}
\label{subsec:hyperparameter}
Generally, we only negate the weights of the first layer, as this is sufficient to induce changes in low-level feature representations, leading to significant parameter updates in deeper layers during fine-tuning. As a result, high-level features containing user-specific information are effectively forgotten. Additionally, while negating multiple layers can strengthen unlearning, it also slows down recovery. For instance, consider negating two layers $\ell_1$ and $\ell_2$, where $\ell_1<\ell_2$ in the computational graph. On one hand, Theorem \ref{theo:negate_C2} suggests that the simplest recovery path involves modifying layers $\ell> \ell_1$. On the other hand, Theorem~\ref{theo:jacobian} indicates that layers $\ell\leq \ell_2$ may suffer from vanishing gradients. Consequently, layers in the range $\{\ell_1<\ell\leq \ell_2\}\neq \emptyset$ should be fine-tuned but are likely to have small gradients, making recovery slower. Additionally, negating both convolution and normalization layers sequentially is ineffective, as the negations cancel out\footnote{Negating both the scale $\sigma$ and mean $\mu$ in a normalization layer $N$ along with the convolution $C$ (without non-linearity) results in: $(-N)(-C)(x) = ((-Cx)-(-\mu))/(-\sigma) = (Cx-\mu)/\sigma = NC(x)$.}.

%% file: sec/6_experiments.tex
\section{Experiments} 
\label{sec_experimens}
In this section, we evaluate NoT across three datasets and three model architectures within federated settings, benchmarking it against seven baseline methods. The main experiments cover random data forgetting, backdoor attack mitigation, and empirical validation of our theoretical predictions. Additionally, we evaluate NoT within a centralized setting, benchmarking it against eight baselines. Results indicate that NoT achieves superior unlearning performance with low communication and computation costs. Additionally, we present an ablation study examining the impact of negating different layers of a model, various perturbations, and changing the ratio of data to forget.

\subsection{Experimental Setup}
\paragraph{Datasets and Models.} We evaluate NoT using CIFAR-10/100 \citep{krizhevsky2014cifar} and Caltech-101 \citep{li_andreeto_ranzato_perona_2022}, with three architectures: CNN (two convolution layers with layer normalization), ResNet-18 \citep{he2016deep}, and Vision Transformer (ViT-B/16) \citep{dosovitskiy2020image}.

\paragraph{Implementation Details.} For each dataset and architecture, we train a global model using FedAvg until convergence, then apply the unlearning algorithm to the converged model. Each communication round involves the participation of all clients, with data distributed IID among them unless stated otherwise. Each client’s data is divided into training and validation sets (80:20), while the test set is used to assess model accuracy. For Caltech-101, we initially split 80\% of samples for training and 20\% for testing, then partition the training data among clients. Further details are provided in Appendix~\ref{further_imp_details}.

\paragraph{Baselines.} We compare NoT with several baselines: \ding{182} \textit{Retrain}: Retraining from scratch with the retain data $D_r$, serving as the gold standard model due to exact unlearning; \ding{183} \textit{FT}: Fine-tuning the original converged model solely with $D_r$, relying on natural forgetting of $D_u$; \ding{184} \textit{FedEraser} \citep{liu2021federaser}; \ding{185} \textit{FUKD} \citep{wu2022federated}; \ding{186} \textit{PGD} \citep{halimi2022federated}; \ding{187} \textit{MoDE} \citep{zhao2023federated}; and \ding{188} \textit{FCU} \citep{deng2024enable}. 

\begin{table*}[t!]
\centering
\caption{\textbf{Client-wise federated unlearning in an IID setting.} Performance comparison of NoT with baselines in a 10-client setup, where the first client requests unlearning. The best average gap is marked in \textcolor{red}{red}.}
\vspace{-3mm}
\resizebox{17.5cm}{!}{
\begin{tabular}{llllllllll}
\toprule
\multirow{2}{*}{\shortstack[c]{\textbf{Dataset} \\ \textbf{ \& Model}}} & \multirow{2}{*}{\shortstack[c]{\textbf{Method}}} & \multicolumn{3}{c}{\textbf{Accuracy (\%)}} & \multicolumn{1}{c}{\textbf{Privacy (\%)}} & \multicolumn{1}{c}{\multirow{2}{*}{\shortstack[c]{\textbf{Avg.} \\ \textbf{Gap $\downarrow$}}}} & \multicolumn{2}{c}{\textbf{Cost (Bytes \& FLOPs)}} \\
\cmidrule(r){3-5} \cmidrule(r){6-6} \cmidrule(r){8-9}  
&& \multicolumn{1}{c}{\textbf{Retain (\textcolor{blue}{$\Delta \downarrow$})}} & \multicolumn{1}{c}{\textbf{Forget (\textcolor{blue}{$\Delta \downarrow$})}} & \multicolumn{1}{c}{\textbf{Test (\textcolor{blue}{$\Delta \downarrow$})}} & \multicolumn{1}{c}{\textbf{MIA (\textcolor{blue}{$\Delta \downarrow$})}} & & \multicolumn{1}{c}{\textbf{Comm. $\downarrow$}} & \multicolumn{1}{c}{\textbf{Comp. $\downarrow$}} \\
\midrule
\multirow{8}{*}{\shortstack[c]{CIFAR-10 \\ CNN}} & Retrain & 91.66\tiny{$\pm$ 0.12} \normalsize{(\textcolor{blue}{0.00}}) & 83.05\tiny{$\pm$ 0.23} \normalsize{(\textcolor{blue}{0.00}}) & 82.32\tiny{$\pm$ 0.30} \normalsize{(\textcolor{blue}{0.00}}) & 50.23\tiny{$\pm$ 0.39} \normalsize{(\textcolor{blue}{0.00}}) & \normalsize{\textcolor{blue}{0.00}} & 1.35$e^{10}$ & 5.81$e^{16}$ \\
\cline{2-9}
& FT & 92.48\tiny{$\pm$ 0.20} \normalsize{(\textcolor{blue}{0.82}}) & 85.56\tiny{$\pm$ 0.36} \normalsize{(\textcolor{blue}{2.51}}) & 82.36\tiny{$\pm$ 0.08} \normalsize{(\textcolor{blue}{0.04}}) & 50.90\tiny{$\pm$ 0.71} \normalsize{(\textcolor{blue}{0.67}}) & \normalsize{\textcolor{blue}{1.01}} & 9.39$e^{09}$ & 4.06$e^{16}$ \\
& FedEraser & 88.19\tiny{$\pm$ 0.16} \normalsize{(\textcolor{blue}{3.47}}) & 81.71\tiny{$\pm$ 0.23} \normalsize{(\textcolor{blue}{1.34}}) & 80.87\tiny{$\pm$ 0.37} \normalsize{(\textcolor{blue}{1.45}}) & 50.17\tiny{$\pm$ 0.26} \normalsize{(\textcolor{blue}{0.06}}) & \normalsize{\textcolor{blue}{1.58}} & 1.34$e^{10}$ & 5.79$e^{16}$ \\
& FUKD & 82.69\tiny{$\pm$ 0.05} \normalsize{(\textcolor{blue}{8.97}}) & 79.31\tiny{$\pm$ 0.12} \normalsize{(\textcolor{blue}{3.74}}) & 78.71\tiny{$\pm$ 0.12} \normalsize{(\textcolor{blue}{3.61}}) & 50.17\tiny{$\pm$ 0.49} \normalsize{(\textcolor{blue}{0.06}}) & \normalsize{\textcolor{blue}{4.09}} & 1.33$e^{10}$ & 5.77$e^{16}$ \\
& PGD & 92.62\tiny{$\pm$ 0.13} \normalsize{(\textcolor{blue}{0.96}}) & 85.36\tiny{$\pm$ 0.30} \normalsize{(\textcolor{blue}{2.31}}) & 82.50\tiny{$\pm$ 0.02} \normalsize{(\textcolor{blue}{0.18}}) & 50.70\tiny{$\pm$ 0.45} \normalsize{(\textcolor{blue}{0.47}}) & \normalsize{\textcolor{blue}{0.98}} & 1.19$e^{10}$ & 5.13$e^{16}$ \\
& MoDE & 92.56\tiny{$\pm$ 0.13} \normalsize{(\textcolor{blue}{0.90}}) & 85.25\tiny{$\pm$ 0.62} \normalsize{(\textcolor{blue}{2.20}}) & 82.31\tiny{$\pm$ 0.35} \normalsize{(\textcolor{blue}{0.01}}) & 50.70\tiny{$\pm$ 0.41} \normalsize{(\textcolor{blue}{0.47}}) & \normalsize{\textcolor{blue}{0.90}} & 1.10$e^{10}$ & 4.77$e^{16}$ \\
& FCU & 92.46\tiny{$\pm$ 0.11} \normalsize{(\textcolor{blue}{0.80}}) & 84.84\tiny{$\pm$ 0.22} \normalsize{(\textcolor{blue}{1.79}}) & 82.48\tiny{$\pm$ 0.21} \normalsize{(\textcolor{blue}{0.16}}) & 50.70\tiny{$\pm$ 0.36} \normalsize{(\textcolor{blue}{0.47}}) & \normalsize{\textcolor{blue}{0.81}} & 1.33$e^{10}$ & 5.75$e^{16}$ \\
\cline{2-9}
\rowcolor{lightgray!50} \cellcolor{white} & NoT (Ours) & 91.69\tiny{$\pm$ 0.02} \normalsize{(\textcolor{blue}{0.03}}) & 83.86\tiny{$\pm$ 0.17} \normalsize{(\textcolor{blue}{0.81}}) & 82.65\tiny{$\pm$ 0.14} \normalsize{(\textcolor{blue}{0.33}}) & 50.23\tiny{$\pm$ 0.21} \normalsize{(\textcolor{blue}{0.00}}) & \normalsize{\textcolor{red}{0.29}} & 7.09$e^{09}$ & 3.06$e^{16}$ \\
\midrule
\multirow{8}{*}{\shortstack[c]{CIFAR-100 \\ CNN}} & Retrain & 72.32\tiny{$\pm$ 0.11} \normalsize{(\textcolor{blue}{0.00}}) & 53.31\tiny{$\pm$ 0.87} \normalsize{(\textcolor{blue}{0.00}}) & 54.28\tiny{$\pm$ 0.25} \normalsize{(\textcolor{blue}{0.00}}) & 49.70\tiny{$\pm$ 0.64} \normalsize{(\textcolor{blue}{0.00}}) & \normalsize{\textcolor{blue}{0.00}} & 1.38$e^{10}$ & 5.96$e^{16}$ \\
\cline{2-9}
 & FT & 73.68\tiny{$\pm$ 0.06} \normalsize{(\textcolor{blue}{1.36}}) & 56.11\tiny{$\pm$ 0.45} \normalsize{(\textcolor{blue}{2.80}}) & 55.46\tiny{$\pm$ 0.08} \normalsize{(\textcolor{blue}{1.18}}) & 49.77\tiny{$\pm$ 1.11} \normalsize{(\textcolor{blue}{0.07}}) & \normalsize{\textcolor{blue}{1.35}} & 1.33$e^{10}$ & 5.77$e^{16}$ \\
 & FedEraser & 67.25\tiny{$\pm$ 0.44} \normalsize{(\textcolor{blue}{5.07}}) & 51.02\tiny{$\pm$ 0.05} \normalsize{(\textcolor{blue}{2.29}}) & 51.51\tiny{$\pm$ 0.62} \normalsize{(\textcolor{blue}{2.77}}) & 49.60\tiny{$\pm$ 0.57} \normalsize{(\textcolor{blue}{0.10}}) & \normalsize{\textcolor{blue}{2.56}} & 1.38$e^{10}$ & 5.96$e^{16}$ \\
 & FUKD & 55.99\tiny{$\pm$ 0.03} \normalsize{(\textcolor{blue}{16.33}}) & 45.20\tiny{$\pm$ 0.04} \normalsize{(\textcolor{blue}{8.11}}) & 47.32\tiny{$\pm$ 0.12} \normalsize{(\textcolor{blue}{6.96}}) & 51.13\tiny{$\pm$ 0.24} \normalsize{(\textcolor{blue}{1.43}}) & \normalsize{\textcolor{blue}{8.21}} & 1.38$e^{10}$ & 5.95$e^{16}$ \\
 & PGD & 73.68\tiny{$\pm$ 0.11} \normalsize{(\textcolor{blue}{1.36}}) & 56.00\tiny{$\pm$ 0.47} \normalsize{(\textcolor{blue}{2.69}}) & 55.21\tiny{$\pm$ 0.07} \normalsize{(\textcolor{blue}{0.93}}) & 49.83\tiny{$\pm$ 0.95} \normalsize{(\textcolor{blue}{0.13}}) & \normalsize{\textcolor{blue}{1.28}} & 1.21$e^{10}$ & 5.22$e^{16}$ \\
 & MoDE & 73.36\tiny{$\pm$ 0.45} \normalsize{(\textcolor{blue}{1.04}}) & 55.64\tiny{$\pm$ 0.37} \normalsize{(\textcolor{blue}{2.33}}) & 55.16\tiny{$\pm$ 0.19} \normalsize{(\textcolor{blue}{0.88}}) & 49.67\tiny{$\pm$ 0.94} \normalsize{(\textcolor{blue}{0.03}}) & \normalsize{\textcolor{blue}{1.07}} & 1.20$e^{10}$ & 5.19$e^{16}$ \\
 & FCU & 73.40\tiny{$\pm$ 0.11} \normalsize{(\textcolor{blue}{1.08}}) & 56.68\tiny{$\pm$ 0.08} \normalsize{(\textcolor{blue}{3.37}}) & 55.37\tiny{$\pm$ 0.08} \normalsize{(\textcolor{blue}{1.09}}) & 50.03\tiny{$\pm$ 0.19} \normalsize{(\textcolor{blue}{0.33}}) & \normalsize{\textcolor{blue}{1.47}} & 1.04$e^{10}$ & 4.49$e^{16}$ \\
\cline{2-9}
\rowcolor{lightgray!50} \cellcolor{white} & NoT (Ours) & 72.25\tiny{$\pm$ 0.08} \normalsize{(\textcolor{blue}{0.07}}) & 55.22\tiny{$\pm$ 0.61} \normalsize{(\textcolor{blue}{1.91}}) & 55.23\tiny{$\pm$ 0.39} \normalsize{(\textcolor{blue}{0.95}}) & 49.63\tiny{$\pm$ 0.97} \normalsize{(\textcolor{blue}{0.07}}) & \normalsize{\textcolor{red}{0.75}} & 1.33$e^{10}$ & 5.73$e^{16}$ \\
\midrule
\multirow{6}{*}{\shortstack[c]{CIFAR-10 \\ ResNet-18}} & Retrain & 100.00\tiny{$\pm$ 0.00} \normalsize{(\textcolor{blue}{0.00}}) & 87.66\tiny{$\pm$ 0.64} \normalsize{(\textcolor{blue}{0.00}}) & 87.73\tiny{$\pm$ 0.35} \normalsize{(\textcolor{blue}{0.00}}) & 49.37\tiny{$\pm$ 0.29} \normalsize{(\textcolor{blue}{0.00}}) & \normalsize{\textcolor{blue}{0.00}} & 1.23$e^{12}$ & 5.66$e^{18}$ \\
\cline{2-9}
& FT & 99.45\tiny{$\pm$ 0.77} \normalsize{(\textcolor{blue}{0.55}}) & 97.36\tiny{$\pm$ 2.22} \normalsize{(\textcolor{blue}{9.70}}) & 86.61\tiny{$\pm$ 1.54} \normalsize{(\textcolor{blue}{1.12}}) & 56.87\tiny{$\pm$ 0.84} \normalsize{(\textcolor{blue}{7.50}}) & \normalsize{\textcolor{blue}{4.72}} & 5.94$e^{11}$ & 2.73$e^{18}$ \\
& PGD & 99.51\tiny{$\pm$ 0.69} \normalsize{(\textcolor{blue}{0.49}}) & 97.67\tiny{$\pm$ 2.48} \normalsize{(\textcolor{blue}{10.01}}) & 86.82\tiny{$\pm$ 1.89} \normalsize{(\textcolor{blue}{0.91}}) & 56.70\tiny{$\pm$ 1.31} \normalsize{(\textcolor{blue}{7.33}}) & \normalsize{\textcolor{blue}{4.68}} & 6.04$e^{11}$ & 2.78$e^{18}$ \\
& MoDE & 99.80\tiny{$\pm$ 0.20} \normalsize{(\textcolor{blue}{0.20}}) & 91.18\tiny{$\pm$ 0.68} \normalsize{(\textcolor{blue}{3.52}}) & 87.11\tiny{$\pm$ 0.65} \normalsize{(\textcolor{blue}{0.62}}) & 52.37\tiny{$\pm$ 0.66} \normalsize{(\textcolor{blue}{3.00}}) & \normalsize{\textcolor{blue}{1.83}} & 5.88$e^{11}$ & 2.71$e^{18}$ \\
& FCU & 100.00\tiny{$\pm$ 0.00} \normalsize{(\textcolor{blue}{0.00}}) & 87.51\tiny{$\pm$ 0.45} \normalsize{(\textcolor{blue}{0.15}}) & 85.93\tiny{$\pm$ 0.07} \normalsize{(\textcolor{blue}{1.80}}) & 50.80\tiny{$\pm$ 0.28} \normalsize{(\textcolor{blue}{1.43}}) & \normalsize{\textcolor{red}{0.84}} & 5.73$e^{11}$ & 2.64$e^{18}$ \\
\cline{2-9}
\rowcolor{lightgray!50} \cellcolor{white} & NoT (Ours) & 99.77\tiny{$\pm$ 0.31} \normalsize{(\textcolor{blue}{0.23}}) & 91.62\tiny{$\pm$ 2.06} \normalsize{(\textcolor{blue}{3.96}}) & 87.63\tiny{$\pm$ 1.61} \normalsize{(\textcolor{blue}{0.10}}) & 52.20\tiny{$\pm$ 0.83} \normalsize{(\textcolor{blue}{2.83}}) & \normalsize{\textcolor{blue}{1.78}} & 5.42$e^{11}$ & 2.49$e^{18}$ \\

\midrule

\multirow{6}{*}{\shortstack[c]{CIFAR-100 \\ ResNet-18}} & Retrain & 99.96\tiny{$\pm$ 0.00} \normalsize{(\textcolor{blue}{0.00}}) & 59.96\tiny{$\pm$ 0.61} \normalsize{(\textcolor{blue}{0.00}}) & 60.66\tiny{$\pm$ 0.63} \normalsize{(\textcolor{blue}{0.00}}) & 50.30\tiny{$\pm$ 0.30} \normalsize{(\textcolor{blue}{0.00}}) & \normalsize{\textcolor{blue}{0.00}} & 7.34$e^{11}$ & 3.38$e^{18}$ \\
\cline{2-9}
& FT & 99.85\tiny{$\pm$ 0.12} \normalsize{(\textcolor{blue}{0.11}}) & 88.80\tiny{$\pm$ 3.22} \normalsize{(\textcolor{blue}{28.84}}) & 60.41\tiny{$\pm$ 1.77} \normalsize{(\textcolor{blue}{0.25}}) & 64.33\tiny{$\pm$ 1.10} \normalsize{(\textcolor{blue}{14.03}}) & \normalsize{\textcolor{blue}{10.81}} & 7.28$e^{11}$ & 3.35$e^{18}$ \\
& PGD & 99.49\tiny{$\pm$ 0.41} \normalsize{(\textcolor{blue}{0.47}}) & 75.30\tiny{$\pm$ 1.35} \normalsize{(\textcolor{blue}{15.34}}) & 61.12\tiny{$\pm$ 0.42} \normalsize{(\textcolor{blue}{0.46}}) & 57.33\tiny{$\pm$ 0.73} \normalsize{(\textcolor{blue}{7.03}}) & \normalsize{\textcolor{blue}{5.83}} & 6.67$e^{11}$ & 3.07$e^{18}$ \\
& MoDE & 99.69\tiny{$\pm$ 0.29} \normalsize{(\textcolor{blue}{0.27}}) & 73.53\tiny{$\pm$ 3.62} \normalsize{(\textcolor{blue}{13.57}}) & 60.74\tiny{$\pm$ 2.33} \normalsize{(\textcolor{blue}{0.08}}) & 57.00\tiny{$\pm$ 2.02} \normalsize{(\textcolor{blue}{6.70}}) & \normalsize{\textcolor{blue}{5.16}} & 6.76$e^{11}$ & 3.12$e^{18}$ \\
& FCU & 99.92\tiny{$\pm$ 0.03} \normalsize{(\textcolor{blue}{0.04}}) & 67.14\tiny{$\pm$ 0.20} \normalsize{(\textcolor{blue}{7.18}}) & 60.52\tiny{$\pm$ 0.26} \normalsize{(\textcolor{blue}{0.14}}) & 52.97\tiny{$\pm$ 0.76} \normalsize{(\textcolor{blue}{2.67}}) & \normalsize{\textcolor{red}{2.51}} & 3.70$e^{11}$ & 1.71$e^{18}$ \\
\cline{2-9}
\rowcolor{lightgray!50} \cellcolor{white} & NoT (Ours) & 99.35\tiny{$\pm$ 0.64} \normalsize{(\textcolor{blue}{0.61}}) & 72.03\tiny{$\pm$ 4.20} \normalsize{(\textcolor{blue}{12.07}}) & 61.98\tiny{$\pm$ 2.47} \normalsize{(\textcolor{blue}{1.32}}) & 55.20\tiny{$\pm$ 0.80} \normalsize{(\textcolor{blue}{4.90}}) & \normalsize{\textcolor{blue}{4.73}} & 6.09$e^{11}$ & 2.80$e^{18}$ \\
\midrule
\multirow{6}{*}{\shortstack[c]{Caltech-101 \\ ViT}} & Retrain & 99.73\tiny{$\pm$ 0.04} \normalsize{(\textcolor{blue}{0.00}}) & 48.29\tiny{$\pm$ 0.44} \normalsize{(\textcolor{blue}{0.00}}) & 48.02\tiny{$\pm$ 0.72} \normalsize{(\textcolor{blue}{0.00}}) & 49.67\tiny{$\pm$ 3.47} \normalsize{(\textcolor{blue}{0.00}}) & \normalsize{\textcolor{blue}{0.00}} & 1.76$e^{12}$ & 1.37$e^{21}$ \\
\cline{2-9}
& FT & 99.96\tiny{$\pm$ 0.00} \normalsize{(\textcolor{blue}{0.23}}) & 94.23\tiny{$\pm$ 0.51} \normalsize{(\textcolor{blue}{45.94}}) & 48.75\tiny{$\pm$ 0.27} \normalsize{(\textcolor{blue}{0.73}}) & 73.80\tiny{$\pm$ 0.94} \normalsize{(\textcolor{blue}{24.13}}) & \normalsize{\textcolor{blue}{17.76}} & 1.63$e^{12}$ & 1.28$e^{21}$ \\
& PGD & 73.34\tiny{$\pm$ 14.31} \normalsize{(\textcolor{blue}{26.39}}) & 61.44\tiny{$\pm$ 12.13} \normalsize{(\textcolor{blue}{13.15}}) & 44.22\tiny{$\pm$ 2.41} \normalsize{(\textcolor{blue}{3.80}}) & 61.43\tiny{$\pm$ 6.47} \normalsize{(\textcolor{blue}{11.76}}) & \normalsize{\textcolor{blue}{13.78}} & 4.78$e^{10}$ & 3.94$e^{19}$ \\
& MoDE & 99.82\tiny{$\pm$ 0.09} \normalsize{(\textcolor{blue}{0.09}}) & 52.01\tiny{$\pm$ 4.13} \normalsize{(\textcolor{blue}{3.72}}) & 48.02\tiny{$\pm$ 0.59} \normalsize{(\textcolor{blue}{0.00}}) & 52.27\tiny{$\pm$ 2.90} \normalsize{(\textcolor{blue}{2.60}}) & \normalsize{\textcolor{blue}{1.60}} & 1.73$e^{12}$ & 1.36$e^{21}$ \\
& FCU & 99.07\tiny{$\pm$ 0.10} \normalsize{(\textcolor{blue}{0.66}}) & 51.83\tiny{$\pm$ 0.61} \normalsize{(\textcolor{blue}{3.54}}) & 48.62\tiny{$\pm$ 0.33} \normalsize{(\textcolor{blue}{0.60}}) & 51.30\tiny{$\pm$ 0.22} \normalsize{(\textcolor{blue}{1.63}}) & \normalsize{\textcolor{blue}{1.61}} & 8.56$e^{11}$ & 8.36$e^{20}$ \\
\cline{2-9}
\rowcolor{lightgray!50} \cellcolor{white} & NoT (Ours) & 99.70\tiny{$\pm$ 0.02} \normalsize{(\textcolor{blue}{0.03}}) & 50.81\tiny{$\pm$ 0.73} \normalsize{(\textcolor{blue}{2.52}}) & 47.83\tiny{$\pm$ 0.27} \normalsize{(\textcolor{blue}{0.19}}) & 50.07\tiny{$\pm$ 2.04} \normalsize{(\textcolor{blue}{0.40}}) & \normalsize{\textcolor{red}{0.79}} & 8.08$e^{11}$ & 6.31$e^{20}$ \\
\bottomrule
\end{tabular}
}
\label{table_dif_dataset_arch}
\end{table*}

\begin{table*}[t!]
\centering
\caption{\textbf{Client-wise federated unlearning in a non-IID setting.} Performance comparison of NoT with baselines in a 10-client setup, where the first client requests unlearning. Client data follows a Dirichlet distribution with $\beta=0.1$. The best average gap is marked in \textcolor{red}{red}.}
\vspace{-3mm}
\resizebox{17.5cm}{!}{
\begin{tabular}{llllllllll}
\toprule
\multirow{2}{*}{\shortstack[c]{\textbf{Dataset} \\ \textbf{ \& Model}}} & \multirow{2}{*}{\shortstack[c]{\textbf{Method}}} & \multicolumn{3}{c}{\textbf{Accuracy (\%)}} & \multicolumn{1}{c}{\textbf{Privacy (\%)}} & \multicolumn{1}{c}{\multirow{2}{*}{\shortstack[c]{\textbf{Avg.} \\ \textbf{Gap $\downarrow$}}}} & \multicolumn{2}{c}{\textbf{Cost (Bytes \& FLOPs)}} \\
\cmidrule(r){3-5} \cmidrule(r){6-6} \cmidrule(r){8-9}  
&& \multicolumn{1}{c}{\textbf{Retain (\textcolor{blue}{$\Delta \downarrow$})}} & \multicolumn{1}{c}{\textbf{Forget (\textcolor{blue}{$\Delta \downarrow$})}} & \multicolumn{1}{c}{\textbf{Test (\textcolor{blue}{$\Delta \downarrow$})}} & \multicolumn{1}{c}{\textbf{MIA (\textcolor{blue}{$\Delta \downarrow$})}} & & \multicolumn{1}{c}{\textbf{Comm. $\downarrow$}} & \multicolumn{1}{c}{\textbf{Comp. $\downarrow$}} \\
\midrule
\multirow{6}{*}{\shortstack[c]{Caltech-101 \\ ViT}} & Retrain & 98.42\tiny{$\pm$ 0.45} \normalsize{(\textcolor{blue}{0.00}}) & 43.80\tiny{$\pm$ 2.15} \normalsize{(\textcolor{blue}{0.00}}) & 45.56\tiny{$\pm$ 0.61} \normalsize{(\textcolor{blue}{0.00}}) & 52.20\tiny{$\pm$ 1.79} \normalsize{(\textcolor{blue}{0.00}}) & \normalsize{\textcolor{blue}{0.00}} & 1.77$e^{12}$ & 9.88$e^{21}$ \\
\cline{2-9}
 & FT & 99.80\tiny{$\pm$ 0.06} \normalsize{(\textcolor{blue}{1.38}}) & 80.04\tiny{$\pm$ 2.18} \normalsize{(\textcolor{blue}{36.24}}) & 47.75\tiny{$\pm$ 0.08} \normalsize{(\textcolor{blue}{2.19}}) & 67.87\tiny{$\pm$ 1.73} \normalsize{(\textcolor{blue}{15.67}}) & \normalsize{\textcolor{blue}{13.87}} & 1.48$e^{12}$ & 8.24$e^{21}$ \\
 & PGD & 99.72\tiny{$\pm$ 0.08} \normalsize{(\textcolor{blue}{1.30}}) & 75.94\tiny{$\pm$ 2.54} \normalsize{(\textcolor{blue}{32.14}}) & 47.25\tiny{$\pm$ 0.50} \normalsize{(\textcolor{blue}{1.69}}) & 65.73\tiny{$\pm$ 0.82} \normalsize{(\textcolor{blue}{13.53}}) & \normalsize{\textcolor{blue}{12.16}} & 1.48$e^{12}$ & 8.25$e^{21}$ \\
 & MoDE & 97.87\tiny{$\pm$ 1.19} \normalsize{(\textcolor{blue}{1.19}}) & 47.69\tiny{$\pm$ 2.53} \normalsize{(\textcolor{blue}{6.39}}) & 46.81\tiny{$\pm$ 0.67} \normalsize{(\textcolor{blue}{1.71}}) & 51.70\tiny{$\pm$ 2.35} \normalsize{(\textcolor{blue}{3.00}}) & \normalsize{\textcolor{blue}{3.07}} & 1.64$e^{12}$ & 9.18$e^{21}$ \\
 & FCU & 99.34\tiny{$\pm$ 0.02} \normalsize{(\textcolor{blue}{0.92}}) & 52.01\tiny{$\pm$ 0.57} \normalsize{(\textcolor{blue}{8.21}}) & 45.53\tiny{$\pm$ 0.57} \normalsize{(\textcolor{blue}{0.03}}) & 54.80\tiny{$\pm$ 1.26} \normalsize{(\textcolor{blue}{2.60}}) & \normalsize{\textcolor{blue}{2.94}} & 1.48$e^{12}$ & 9.44$e^{21}$ \\
\cline{2-9}
\rowcolor{lightgray!50} \cellcolor{white} & NoT (Ours) & 97.71\tiny{$\pm$ 0.41} \normalsize{(\textcolor{blue}{0.71}}) & 45.20\tiny{$\pm$ 2.40} \normalsize{(\textcolor{blue}{1.40}}) & 45.93\tiny{$\pm$ 0.54} \normalsize{(\textcolor{blue}{0.37}}) & 51.23\tiny{$\pm$ 1.49} \normalsize{(\textcolor{blue}{0.97}}) & \normalsize{\textcolor{red}{0.86}} & 8.90$e^{11}$ & 4.95$e^{21}$ \\
\bottomrule
\end{tabular}
}
\label{table_dif_dataset_arch_noniid}
\end{table*}

\paragraph{Evaluation Metrics.} Following prior works \cite{jia2023model, fan2024salun}, we assess NoT’s effectiveness and efficiency with the following metrics: \ding{182} \textit{Retain}, \ding{183} \textit{Forget}, and \ding{184} \textit{Test Accuracies (\%)}, measuring the model's performance on $D_r$, $D_u$, and test data, respectively. \ding{185} \textit{MIA} (Membership Inference Attack \cite{hu2022membership}) (\%): indicates the extent to which $D_u$ remains recognizable in the model \ding{186} \textit{Communication} and \ding{187} \textit{Computation Costs}: quantify the total communication (in bytes) and FLOPs needed for unlearning and recovery. Retain and test accuracies measure how well the model retains knowledge about $D_r$, while forget accuracy and MIA evaluate how effectively the model forgets $D_u$. We calculate delta (\textcolor{blue}{$\Delta$}) values and the \textcolor{blue}{average gap} compared to Retrain, with lower values indicating better performance. Achieving a balance across all metrics is essential to demonstrate NoT’s effectiveness and efficiency.

\begin{table*}[t!]
\centering
\caption{\textbf{Class-wise federated unlearning.} Performance comparison of NoT with baselines in a 10-client setup, where samples of class 0 are to be forgotten from all clients. The data distribution among clients is \textbf{IID}. The best average gap is marked in \textcolor{red}{red}.}
\vspace{-3mm}
\resizebox{17.5cm}{!}{
\begin{tabular}{llllllllll}
\toprule
\multirow{2}{*}{\shortstack[c]{\textbf{Dataset} \\ \textbf{ \& Model}}} & \multirow{2}{*}{\shortstack[c]{\textbf{Method}}} & \multicolumn{3}{c}{\textbf{Accuracy (\%)}} & \multicolumn{1}{c}{\textbf{Privacy (\%)}} & \multicolumn{1}{c}{\multirow{2}{*}{\shortstack[c]{\textbf{Avg.} \\ \textbf{Gap $\downarrow$}}}} & \multicolumn{2}{c}{\textbf{Cost (Bytes \& FLOPs)}} \\
\cmidrule(r){3-5} \cmidrule(r){6-6} \cmidrule(r){8-9}  
&& \multicolumn{1}{c}{\textbf{Retain (\textcolor{blue}{$\Delta \downarrow$})}} & \multicolumn{1}{c}{\textbf{Forget (\textcolor{blue}{$\Delta \downarrow$})}} & \multicolumn{1}{c}{\textbf{Test (\textcolor{blue}{$\Delta \downarrow$})}} & \multicolumn{1}{c}{\textbf{MIA (\textcolor{blue}{$\Delta \downarrow$})}} & & \multicolumn{1}{c}{\textbf{Comm. $\downarrow$}} & \multicolumn{1}{c}{\textbf{Comp. $\downarrow$}} \\
\midrule
\multirow{4}{*}{\shortstack[c]{Caltech-101 \\ ViT}} & Retrain & 99.32\tiny{$\pm$ 0.11} \normalsize{(\textcolor{blue}{0.00}}) & 0.00\tiny{$\pm$ 0.00} \normalsize{(\textcolor{blue}{0.00}}) & 46.85\tiny{$\pm$ 0.35} \normalsize{(\textcolor{blue}{0.00}}) & 92.07\tiny{$\pm$ 1.47} \normalsize{(\textcolor{blue}{0.00}}) & \normalsize{\textcolor{blue}{0.00}} & 1.48$e^{12}$ & 8.24$e^{21}$ \\
\cline{2-9}
& FT & 99.93\tiny{$\pm$ 0.03} \normalsize{(\textcolor{blue}{0.61}}) & 77.39\tiny{$\pm$ 2.37} \normalsize{(\textcolor{blue}{77.39}}) & 46.69\tiny{$\pm$ 0.24} \normalsize{(\textcolor{blue}{0.16}}) & 71.53\tiny{$\pm$ 0.74} \normalsize{(\textcolor{blue}{20.54}}) & \normalsize{\textcolor{blue}{24.68}} & 1.19$e^{12}$ & 6.60$e^{21}$ \\
& MoDE & 98.74\tiny{$\pm$ 0.89} \normalsize{(\textcolor{blue}{0.73}}) & 0.00\tiny{$\pm$ 0.00} \normalsize{(\textcolor{blue}{0.00}}) & 46.97\tiny{$\pm$ 0.42} \normalsize{(\textcolor{blue}{0.18}}) & 84.70\tiny{$\pm$ 6.16} \normalsize{(\textcolor{blue}{8.20}}) & \normalsize{\textcolor{blue}{2.28}} & 1.35$e^{12}$ & 7.54$e^{21}$ \\
\cline{2-9}
\rowcolor{lightgray!50} \cellcolor{white} & NoT (Ours) & 99.36\tiny{$\pm$ 0.10} \normalsize{(\textcolor{blue}{0.04}}) & 0.70\tiny{$\pm$ 0.49} \normalsize{(\textcolor{blue}{0.70}}) & 46.54\tiny{$\pm$ 0.17} \normalsize{(\textcolor{blue}{0.31}}) & 84.83\tiny{$\pm$ 1.37} \normalsize{(\textcolor{blue}{7.24}}) & \normalsize{\textcolor{red}{2.07}} & 8.90$e^{11}$ & 4.95$e^{21}$ \\
\bottomrule
\end{tabular}
}
\label{table_dif_dataset_arch_classforgetting}
\end{table*}

\begin{table*}[t!]
\centering
\caption{\textbf{Instance-wise federated unlearning.} Performance comparison of NoT with baselines in a 10-client setup, where 10\% of each client's data is randomly selected to be forgotten. The data distribution among clients is \textbf{IID}. The best average gap is marked in \textcolor{red}{red}.}
\vspace{-3mm}
\resizebox{17.5cm}{!}{
\begin{tabular}{llllllllll}
\toprule
\multirow{2}{*}{\shortstack[c]{\textbf{Dataset} \\ \textbf{ \& Model}}} & \multirow{2}{*}{\shortstack[c]{\textbf{Method}}} & \multicolumn{3}{c}{\textbf{Accuracy (\%)}} & \multicolumn{1}{c}{\textbf{Privacy (\%)}} & \multicolumn{1}{c}{\multirow{2}{*}{\shortstack[c]{\textbf{Avg.} \\ \textbf{Gap $\downarrow$}}}} & \multicolumn{2}{c}{\textbf{Cost (Bytes \& FLOPs)}} \\
\cmidrule(r){3-5} \cmidrule(r){6-6} \cmidrule(r){8-9}  
&& \multicolumn{1}{c}{\textbf{Retain (\textcolor{blue}{$\Delta \downarrow$})}} & \multicolumn{1}{c}{\textbf{Forget (\textcolor{blue}{$\Delta \downarrow$})}} & \multicolumn{1}{c}{\textbf{Test (\textcolor{blue}{$\Delta \downarrow$})}} & \multicolumn{1}{c}{\textbf{MIA (\textcolor{blue}{$\Delta \downarrow$})}} & & \multicolumn{1}{c}{\textbf{Comm. $\downarrow$}} & \multicolumn{1}{c}{\textbf{Comp. $\downarrow$}} \\
\midrule
\multirow{4}{*}{\shortstack[c]{Caltech-101 \\ ViT}} & Retrain & 99.71\tiny{$\pm$ 0.01} \normalsize{(\textcolor{blue}{0.00}}) & 47.58\tiny{$\pm$ 1.09} \normalsize{(\textcolor{blue}{0.00}}) & 47.98\tiny{$\pm$ 0.23} \normalsize{(\textcolor{blue}{0.00}}) & 49.70\tiny{$\pm$ 1.44} \normalsize{(\textcolor{blue}{0.00}}) & \normalsize{\textcolor{blue}{0.00}} & 1.77$e^{12}$ & 9.88$e^{21}$ \\
\cline{2-9}
& FT & 99.96\tiny{$\pm$ 0.02} \normalsize{(\textcolor{blue}{0.25}}) & 96.61\tiny{$\pm$ 0.91} \normalsize{(\textcolor{blue}{49.03}}) & 48.91\tiny{$\pm$ 0.19} \normalsize{(\textcolor{blue}{0.93}}) & 75.73\tiny{$\pm$ 1.39} \normalsize{(\textcolor{blue}{26.03}}) & \normalsize{\textcolor{blue}{19.06}} & 1.48$e^{12}$ & 8.24$e^{21}$ \\
& MoDE & 97.48\tiny{$\pm$ 1.58} \normalsize{(\textcolor{blue}{1.02}}) & 51.94\tiny{$\pm$ 5.86} \normalsize{(\textcolor{blue}{2.49}}) & 48.16\tiny{$\pm$ 0.40} \normalsize{(\textcolor{blue}{0.35}}) & 51.33\tiny{$\pm$ 5.27} \normalsize{(\textcolor{blue}{0.17}}) & \normalsize{\textcolor{blue}{1.01}} & 1.05$e^{12}$ & 5.90$e^{21}$ \\
\cline{2-9}
\rowcolor{lightgray!50} \cellcolor{white} & NoT (Ours) & 99.85\tiny{$\pm$ 0.03} \normalsize{(\textcolor{blue}{0.14}}) & 50.25\tiny{$\pm$ 1.24} \normalsize{(\textcolor{blue}{2.67}}) & 47.97\tiny{$\pm$ 0.23} \normalsize{(\textcolor{blue}{0.01}}) & 50.73\tiny{$\pm$ 1.75} \normalsize{(\textcolor{blue}{1.03}}) & \normalsize{\textcolor{red}{0.96}} & 9.19$e^{11}$ & 5.50$e^{21}$ \\

\bottomrule
\end{tabular}
}
\label{table_dif_dataset_arch_instanceforgetting}
\end{table*}

\subsection{Results}
\label{sec_results}

\begin{figure}[t!]
\vspace{-10pt}
\centering
\includegraphics[width=\linewidth]{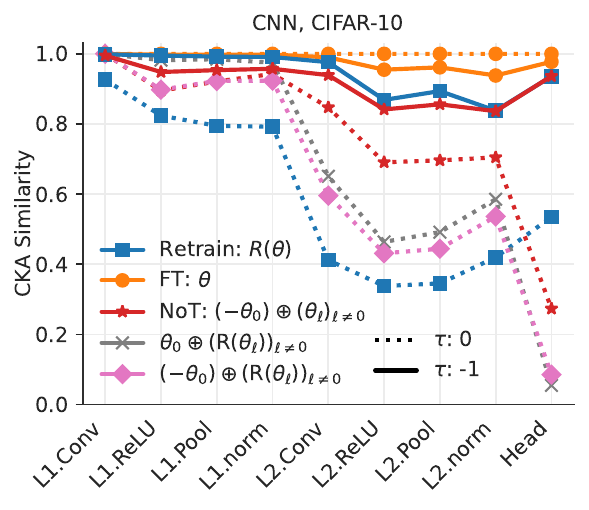}
\vspace*{-7mm}
\caption{\textbf{CKA of layer activations} for various models compared to the original model ($\theta=\theta^*$) before fine-tuning (FT@$\tau$:0). The first and last communication rounds are denoted by $\tau$:0 and -1. $(-\theta_0)\oplus \left( \mathrm{R}(\theta_\ell)\right)_{\ell\neq 0}$ denotes a model with negated first-layer weights ($\ell$:0) and randomized rest $R(\cdot)$.
}
\label{fig:cka_cnn}
\end{figure}

\paragraph{Federated Unlearning: Comparing NoT with Baselines.} Table~\ref{table_dif_dataset_arch} shows that NoT performance outperforms other baselines in client-wise FU across architectures and datasets, achieving a low average gap and competitive communication and computation costs. While FCU shows lower average gap in some cases, such as with ResNet-18, its performance varies across architectures, unlike NoT’s consistent balance between forgetting effectiveness (Forget Accuracy, MIA), model fidelity (Retain, Test Accuracy), and competitively low costs. Although Table~\ref{table_dif_dataset_arch} does not report storage, FedEraser and FUKD require considerable storage that grows with communication rounds, making them impractical for ResNet-18 and ViT due to storage constraints. In contrast, NoT requires no additional storage. FedEraser and FUKD also perform worse than Retrain because, in our setting, the target client joins the federation from the first round, and as those baselines need to refer back to the first checkpoint, this makes full retraining more efficient. Table \ref{table_dif_dataset_arch_noniid} presents results for client-wise FU in a non-IID setting, where data across clients follow a Dirichlet distribution with $\beta=0.1$. Furthermore, results for class-wise and instance-wise forgetting are provided in Tables~\ref{table_dif_dataset_arch_classforgetting} and \ref{table_dif_dataset_arch_instanceforgetting}. These results highlight NoT’s superior performance across different settings, demonstrating its effectiveness and efficiency compared to baseline methods.

\begin{table*}[t!]
\centering
\caption{\textbf{Backdoor attack in FL.} Performance comparison of NoT with baselines in a 10-client setup, where one client with 80\% backdoored samples requests unlearning. The best average gap is marked in \textcolor{red}{red}.}
\vspace{-3mm}
\resizebox{17.5cm}{!}{
\begin{tabular}{llllllllll}
\toprule
\multirow{2}{*}{\shortstack[c]{\textbf{Dataset} \\ \textbf{ \& Model}}} & \multirow{2}{*}{\shortstack[c]{\textbf{Method}}} & \multicolumn{3}{c}{\textbf{Accuracy (\%)}} & \multicolumn{1}{c}{\textbf{Privacy (\%)}} & \multicolumn{1}{c}{\multirow{2}{*}{\shortstack[c]{\textbf{Avg.} \\ \textbf{Gap $\downarrow$}}}} & \multicolumn{2}{c}{\textbf{Cost (Bytes \& FLOPs)}} \\
\cmidrule(r){3-5} \cmidrule(r){6-6} \cmidrule(r){8-9}  
&& \multicolumn{1}{c}{\textbf{Retain (\textcolor{blue}{$\Delta \downarrow$})}} & \multicolumn{1}{c}{\textbf{Forget (\textcolor{blue}{$\Delta \downarrow$})}} & \multicolumn{1}{c}{\textbf{Test (\textcolor{blue}{$\Delta \downarrow$})}} & \multicolumn{1}{c}{\textbf{MIA (\textcolor{blue}{$\Delta \downarrow$})}} & & \multicolumn{1}{c}{\textbf{Comm. $\downarrow$}} & \multicolumn{1}{c}{\textbf{Comp. $\downarrow$}} \\
\midrule
\multirow{6}{*}{\shortstack[c]{Caltech-101 \\ ViT}} & Retrain & 99.73\tiny{$\pm$ 0.08} \normalsize{(\textcolor{blue}{0.00}}) & 15.74\tiny{$\pm$ 0.34} \normalsize{(\textcolor{blue}{0.00}}) & 48.31\tiny{$\pm$ 0.24} \normalsize{(\textcolor{blue}{0.00}}) & 73.17\tiny{$\pm$ 1.55} \normalsize{(\textcolor{blue}{0.00}}) & \normalsize{\textcolor{blue}{0.00}} & 1.76$e^{12}$ & 1.37$e^{21}$ \\
\cline{2-9}
& FT & 99.93\tiny{$\pm$ 0.02} \normalsize{(\textcolor{blue}{0.20}}) & 23.24\tiny{$\pm$ 0.96} \normalsize{(\textcolor{blue}{7.50}}) & 47.85\tiny{$\pm$ 0.07} \normalsize{(\textcolor{blue}{0.46}}) & 60.60\tiny{$\pm$ 0.57} \normalsize{(\textcolor{blue}{12.57}}) & \normalsize{\textcolor{blue}{5.18}} & 1.74$e^{12}$ & 1.36$e^{21}$ \\
& PGD & 99.41\tiny{$\pm$ 0.37} \normalsize{(\textcolor{blue}{0.32}}) & 62.82\tiny{$\pm$ 29.77} \normalsize{(\textcolor{blue}{47.08}}) & 45.95\tiny{$\pm$ 1.28} \normalsize{(\textcolor{blue}{2.36}}) & 70.07\tiny{$\pm$ 6.27} \normalsize{(\textcolor{blue}{3.10}}) & \normalsize{\textcolor{blue}{13.21}} & 1.56$e^{12}$ & 1.22$e^{21}$ \\
& MoDE & 98.70\tiny{$\pm$ 0.92} \normalsize{(\textcolor{blue}{1.03}}) & 30.81\tiny{$\pm$ 13.87} \normalsize{(\textcolor{blue}{15.07}}) & 45.54\tiny{$\pm$ 2.04} \normalsize{(\textcolor{blue}{2.77}}) & 63.55\tiny{$\pm$ 10.35} \normalsize{(\textcolor{blue}{9.62}}) & \normalsize{\textcolor{blue}{7.12}} & 1.76$e^{12}$ & 1.38$e^{21}$ \\
& FCU & 99.51\tiny{$\pm$ 0.08} \normalsize{(\textcolor{blue}{0.22}}) & 15.50\tiny{$\pm$ 0.53} \normalsize{(\textcolor{blue}{0.24}}) & 48.46\tiny{$\pm$ 0.21} \normalsize{(\textcolor{blue}{0.15}}) & 73.10\tiny{$\pm$ 1.13} \normalsize{(\textcolor{blue}{0.07}}) & \normalsize{\textcolor{blue}{0.17}} & 1.06$e^{12}$ & 9.98$e^{20}$ \\
\cline{2-9}
\rowcolor{lightgray!50} \cellcolor{white} & NoT (Ours) & 99.60\tiny{$\pm$ 0.02} \normalsize{(\textcolor{blue}{0.13}}) & 15.62\tiny{$\pm$ 0.17} \normalsize{(\textcolor{blue}{0.12}}) & 47.95\tiny{$\pm$ 0.07} \normalsize{(\textcolor{blue}{0.36}}) & 73.17\tiny{$\pm$ 1.38} \normalsize{(\textcolor{blue}{0.00}}) & \normalsize{\textcolor{red}{0.15}} & 1.43$e^{12}$ & 1.11$e^{21}$ \\
\bottomrule
\end{tabular}
}
\label{table_backdoor_attack}
\end{table*}

\paragraph{Backdoor Attack.} We further evaluate unlearning effectiveness by removing the influence of backdoor triggers \citep{gu2019badnets}. Following prior works \citep{wu2022federated, halimi2022federated, zhao2023federated},  a fraction of the target client images is poisoned with a 3×3 pixel pattern trigger (using the Adversarial Robustness Toolbox \citep{nicolae2019adversarial}) and relabeled as ‘airplane’ (excluding pre-existing airplane labels). ViT and Caltech-101 are used in this experiment. A global model is then trained from scratch using all clients, including the poisoned one, until convergence. Due to the target client, the global FL model becomes vulnerable to the backdoor trigger. A successful unlearning method should eliminate the influence of the poisoned client. Table~\ref{table_backdoor_attack} shows that NoT achieves the lowest average gap with minimal communication and computation costs, outperforming baselines in removing backdoor influence.

\paragraph{Negation is LWOP and Breaks Co-Adaptation.} We assess the impact of negation through LWOP and the disruption of co-adaptation by comparing CKA similarities \citep{pmlr-v97-kornblith19a} across three models—FT, Retrain, and NoT—analyzed before ($\tau$:0) and after ($\tau$:-1) fine-tuning (see Figure~\ref{fig:cka_cnn}). The FT model at $\tau$:0 serves as the reference. Initially comparing NoT to FT at $\tau$:0; we observe a high CKA similarity at the negated layer (L1.conv), which is consistent with LWO of NoT model at $\tau$:0, hence LWOP of negation. However, as layer depth increases, divergence grows, indicating a breakdown in co-adaptation. Post fine-tuning, both Retrain and NoT at $\tau$:-1 exhibit similar CKA, diverging from FT but closely aligning with each other, underscoring the resemblance of NoT to Retrain. CKA comparisons between models obtained by randomizing (using a consistent seed) all layers except the first in both FT and NoT at $\tau$:0 are also performed: the observed high CKA similarity further supports LWO. See Appendix~\ref{appendix:CKA} for extra CKA comparisons. Direct validation of LWO (Appendix~\ref{appendix:LWO}) involves freezing and negating the first layer, while randomizing the others, then fine-tuning; this yields a model with test accuracy comparable to training from scratch, consistent with our predictions. Finally, the dimensionality reduction of the gradient descent search space is confirmed using PCA (Appendix~\ref{appendix:dim_reduction}). 

\begin{table*}[t!]
\centering
\caption{\textbf{Centralized unlearning.} Performance comparison of NoT with baselines, with the best average gap marked in \textcolor{red}{red}. The forget data is randomly selected and constitutes 10\% of the train data.}
\vspace{-3mm}
\resizebox{17.5cm}{!}{
\begin{tabular}{llllllll}
\toprule
\multirow{2}{*}{\shortstack[c]{\textbf{Dataset} \\ \textbf{ \& Model}}} & \multirow{2}{*}{\shortstack[c]{\textbf{Method}}} & \multicolumn{3}{c}{\textbf{Accuracy (\%)}} & \multicolumn{1}{c}{\textbf{Privacy (\%)}} & \multicolumn{1}{c}{\multirow{2}{*}{\shortstack[c]{\textbf{Avg.} \\ \textbf{Gap $\downarrow$}}}} & \multirow{2}{*}{\shortstack[c]{\textbf{Comp. Cost} \\ \textbf{(FLOPs) $\downarrow$}}} \\
\cmidrule(r){3-5} \cmidrule(r){6-6} 
&& \multicolumn{1}{c}{\textbf{Retain (\textcolor{blue}{$\Delta \downarrow$})}} & \multicolumn{1}{c}{\textbf{Forget (\textcolor{blue}{$\Delta \downarrow$})}} & \multicolumn{1}{c}{\textbf{Test (\textcolor{blue}{$\Delta \downarrow$})}} & \multicolumn{1}{c}{\textbf{MIA (\textcolor{blue}{$\Delta \downarrow$})}} & & \\
\midrule
\multirow{9}{*}{\shortstack[c]{CIFAR-10 \\ ResNet-18}} & Retrain & 100.00\tiny{$\pm$ 0.00} \normalsize{(\textcolor{blue}{0.00}}) & 91.72\tiny{$\pm$ 0.14} \normalsize{(\textcolor{blue}{0.00}}) & 92.03\tiny{$\pm$ 0.12} \normalsize{(\textcolor{blue}{0.00}}) & 49.32\tiny{$\pm$ 0.23} \normalsize{(\textcolor{blue}{0.00}}) & \normalsize{\textcolor{blue}{0.00}} & 9.58$e^{15}$ \\ 
\cline{2-8}
& FT & 98.63\tiny{$\pm$ 0.24} \normalsize{(\textcolor{blue}{1.37}}) & 95.92\tiny{$\pm$ 0.53} \normalsize{(\textcolor{blue}{4.20}}) & 89.84\tiny{$\pm$ 0.27} \normalsize{(\textcolor{blue}{2.19}}) & 53.49\tiny{$\pm$ 0.45} \normalsize{(\textcolor{blue}{4.17}}) & \normalsize{\textcolor{blue}{2.98}} & 7.31$e^{14}$ \\
& RandL & 94.34\tiny{$\pm$ 0.32} \normalsize{(\textcolor{blue}{5.66}}) & 88.85\tiny{$\pm$ 0.46} \normalsize{(\textcolor{blue}{2.87}}) & 89.53\tiny{$\pm$ 0.21} \normalsize{(\textcolor{blue}{2.50}}) & 54.60\tiny{$\pm$ 0.43} \normalsize{(\textcolor{blue}{5.28}}) & \normalsize{\textcolor{blue}{4.07}} & 2.07$e^{15}$ \\
& GA \citelinktext{thudi2022unrolling}{\small{(EuroS\&P, 2022)}} 
 & 98.76\tiny{$\pm$ 0.07} \normalsize{(\textcolor{blue}{1.24}}) & 93.39\tiny{$\pm$ 0.19} \normalsize{(\textcolor{blue}{1.67}}) & 89.73\tiny{$\pm$ 0.14} \normalsize{(\textcolor{blue}{2.30}}) & 52.21\tiny{$\pm$ 0.24} \normalsize{(\textcolor{blue}{2.89}}) & \normalsize{\textcolor{blue}{2.02}} & 2.09$e^{15}$  \\
& BadT \citelinktext{chundawat2023can}{\small{(AAAI, 2023)}} & 99.89\tiny{$\pm$ 0.03} \normalsize{(\textcolor{blue}{0.11}}) & 98.52\tiny{$\pm$ 0.30} \normalsize{(\textcolor{blue}{6.80}}) & 90.14\tiny{$\pm$ 0.17} \normalsize{(\textcolor{blue}{1.89}}) & 48.82\tiny{$\pm$ 0.12} \normalsize{(\textcolor{blue}{0.50}}) & \normalsize{\textcolor{blue}{2.32}} & 5.69$e^{14}$ \\
& $\ell_1$-sparse \citelinktext{jia2023model}{\small{(NIPS, 2023)}} & 99.98\tiny{$\pm$ 0.00} \normalsize{(\textcolor{blue}{0.02}}) & 91.94\tiny{$\pm$ 0.03} \normalsize{(\textcolor{blue}{0.22}}) & 92.15\tiny{$\pm$ 0.21} \normalsize{(\textcolor{blue}{0.12}}) & 49.63\tiny{$\pm$ 0.07} \normalsize{(\textcolor{blue}{0.31}}) & \normalsize{\textcolor{red}{0.19}} & 2.98$e^{15}$   \\
& SSD \citelinktext{foster2024fast}{\small{(AAAI, 2024)}} & 98.82\tiny{$\pm$ 0.85} \normalsize{(\textcolor{blue}{1.18}}) & 99.05\tiny{$\pm$ 0.72} \normalsize{(\textcolor{blue}{7.33}}) & 89.70\tiny{$\pm$ 1.40} \normalsize{(\textcolor{blue}{2.33}}) & 58.25\tiny{$\pm$ 1.51} \normalsize{(\textcolor{blue}{8.93}}) & \normalsize{\textcolor{blue}{4.94}} & 7.46$e^{13}$ \\
& SalUn \citelinktext{fan2024salun}{\small{(ICLR, 2024)}}& 98.29\tiny{$\pm$ 0.19} \normalsize{(\textcolor{blue}{1.71}}) & 94.01\tiny{$\pm$ 0.45} \normalsize{(\textcolor{blue}{2.29}}) & 90.78\tiny{$\pm$ 0.13} \normalsize{(\textcolor{blue}{1.25}}) & 48.18\tiny{$\pm$ 1.34} \normalsize{(\textcolor{blue}{1.14}}) & \normalsize{\textcolor{blue}{1.64}} & 1.22$e^{15}$  \\
\cline{2-8}
\rowcolor{lightgray!50} \cellcolor{white} & \textbf{NoT (Ours)} & 99.69\tiny{$\pm$ 0.05} \normalsize{(\textcolor{blue}{0.31}}) & 92.41\tiny{$\pm$ 0.12} \normalsize{(\textcolor{blue}{0.69}}) & 92.18\tiny{$\pm$ 0.07} \normalsize{(\textcolor{blue}{0.15}}) & 49.52\tiny{$\pm$ 0.11} \normalsize{(\textcolor{blue}{0.20}}) & \normalsize{\textcolor{blue}{0.34}} & 2.19$e^{15}$  \\
\bottomrule
\end{tabular}
}
\label{tbl_centralized_results}
\end{table*}

\begin{figure*}[!t]
    \begin{minipage}[t]{0.31\textwidth}
     \centering
     \includegraphics[trim={0.25cm 0.19cm 0.3cm 0.17cm},clip,width=\linewidth]{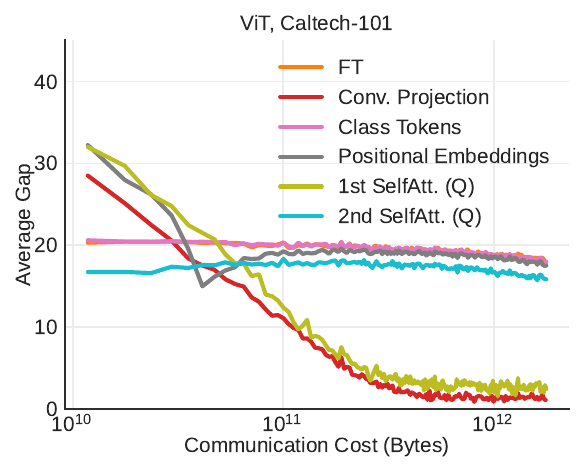}
     \vspace*{-6mm}
    \caption{\textbf{Effect of negating different ViT layers.} Negating the convolution projection layer resulted in the best unlearning performance.}
    \label{fig_negate_diff_layers}
   \end{minipage} \hfill
    \begin{minipage}[t]{0.31\textwidth}
        \centering
        \includegraphics[trim={0.25cm 0.19cm 0.3cm 0.17cm},clip,width=\linewidth]{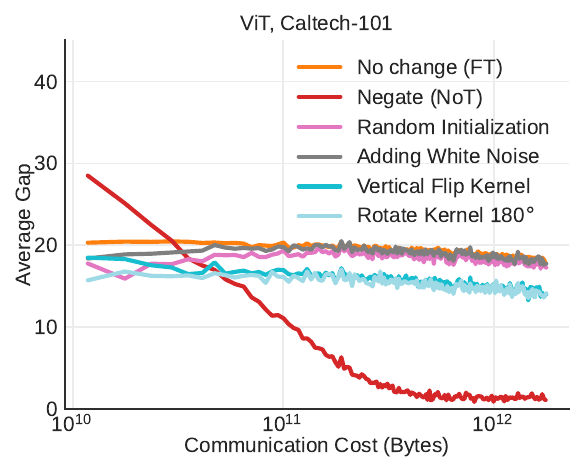}
        \vspace*{-6mm}
        \caption{\textbf{Effect of different perturbations on ViT convolution projection layer.} Applying weight negation is the best perturbation for inducing unlearning.}
    \label{fig_diff_perturbations}
    \end{minipage} \hfill
\begin{minipage}[t]{0.31\textwidth}
        \centering
        \includegraphics[trim={0.25cm 0.19cm 0.3cm 0.17cm},clip,width=\linewidth]{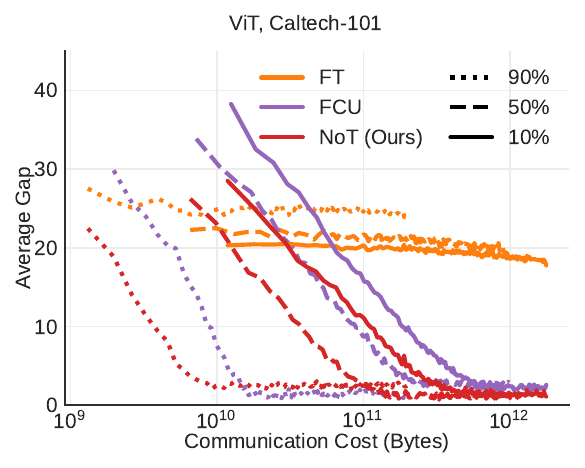}
        \vspace*{-6mm}
        \caption{\textbf{Effect of varying forget data ratios (\ie, target clients).} NoT attains the best unlearning with the least cost compared to baselines under different ratios.}
        \label{fig_diff_forgetdata50}
    \end{minipage} \hfill
\end{figure*}

\paragraph{Centralized Unlearning: Comparing NoT with Baselines.} To validate NoT in a centralized setting, we compare it to baselines: \ding{182} \textit{Retrain}; \ding{183} \textit{FT}; \ding{184} \textit{RandL} (Random Label): The forget set ($D_u$) is randomly relabeled, followed by fine-tuning the model on the updated dataset; \ding{185} \textit{GA} (Gradient Ascent) \citep{thudi2022unrolling}: Fine-tuning on $D_u$ by increasing loss; \ding{186} \textit{BadT} \citep{chundawat2023can}; \ding{187} \textit{$\ell_1$-sparse} \citep{jia2023model}; \ding{188} \textit{SSD} \citep{foster2024fast}; and \ding{189} \textit{SalUn} \citep{fan2024salun}. Table~\ref{tbl_centralized_results} confirms that NoT achieves the best balance of efficiency and accuracy without extra storage or access to $D_u$. While $\ell_1$-sparse\footnote{Our theory applies to $\ell_1$-sparse \citep{jia2023model} as its sparsification step is a perturbation satisfying C1 and C2.} achieves lower average gap, it requires more than 30\% more computations. These results affirm NoT’s broader applicability beyond federated settings.

\subsection{Ablation Study} 
\label{sec_ablation}
\paragraph{Negation of Different Layers.} Figure~\ref{fig_negate_diff_layers} shows that negating ViT’s convolution projection layer achieves the best performance. Negating the positional embedding layer initially raises average gap but eventually recovers, resembling natural forgetting (FT). Negating the query weights of the first self-attention head is effective but falls short of the projection layer. These findings suggest that early-layer negation is most effective, as the largest forgetting occurs in subsequent layers in the computational graph\footnote{With residual connections, the last layer may ‘follow’ the first layer.} due to a \textit{break in co-adaptation}. This is evident in Figure~\ref{fig_negate_diff_layers}, where negating the second self-attention head’s query weights results in significantly less unlearning than the first. 

\paragraph{Different Perturbations.}\label{sec_perturbation} Figure~\ref{fig_diff_perturbations} evaluates various perturbations on ViT’s convolution projection layer. Weight negation achieves the lowest average gap, while other perturbations (\eg, random initialization, adding noise) perform similarly to FT (no perturbation). Kernel flips or rotations are likely LWOP and provide slight improvement but remain less effective than negation.
Future research could investigate alternative LWOP perturbations with the potential for further enhancing unlearning performance.

\paragraph{Different Forget Data ($D_u$) Ratios.} Figure~\ref{fig_diff_forgetdata50} examines NoT's performance across varying forget ratios (\ie, more target clients). Forget ratios of 10\%, 50\%, and 90\% represent 1, 5, and 9 clients (out of 10) requesting unlearning, respectively. For all ratios, a Retrain model is computed using the new corresponding retain data. Across all ratios, NoT closely matches Retrain's performance with minimal communication costs. This is due to the reduction in retain data as forget data increases. In contrast, FCU’s computation cost rises significantly as forget ratios increase, while NoT remains efficient due to not needing access to $D_u$.

%% file: sec/7_limitations.tex
\section{Limitations}
\label{sec_limitations}
While NoT effectively unlearns targeted data contributions, it also leads to some loss of useful knowledge, necessitating access to retain data for recovery—similar to other strong baseline methods. This access to retain data is essential for optimal model performance.  Moreover, NoT has been evaluated only on vision-centric classification tasks with smaller datasets (CIFAR-10/100, Caltech-101), and has not been tested on larger-scale datasets, such as Imagenet \citep{russakovsky2015imagenet}.

%% file: sec/8_conclusion.tex
\section{Conclusion}
\label{sec_conclusion}
We proposed NoT, a novel and efficient federated unlearning (FU) algorithm that requires no additional storage or access to target data. NoT achieves unlearning by negating specific layer-wise parameters, disrupting co-adaptation across layers. We also proposed a theoretical framework, supported by empirical results, that demonstrates that layer-wise negation effectively induces unlearning, while the model’s good optimization state post-negation enables rapid recovery during fine-tuning. NoT has shown strong unlearning and recovery performance across diverse architectures and datasets, surpassing existing FU baselines and outperforming traditional machine unlearning methods.

%% file: sec/X_suppl.tex
\clearpage
\setcounter{page}{1}
\maketitlesupplementary

\section{Co-Adaptation}
\label{appendix:coadapt}
The term \textit{co-adaptation} originates from evolutionary biology and describes phenomena where living entities—from proteins to entire species—are not merely adapted in isolation
but within an environment composed of other living entities. The entities are adapted to an ecological niche, which itself evolves through the adaptations of other entities. As a result, multiple living entities adapt together through coevolution, leading to a state of co-adaptation \cite{FUTUYMA201370}.  

In computer science, the term has been adopted more loosely to describe scenarios where different agents—such as human-machine interfaces \cite{merel2013multi}, robots \cite{shi2012survey}, software components \cite{10.1007/978-1-4471-0509-1_20}, network layers \cite{sato2019breaking}, and artificial neurons \cite{hinton2012improving,wang2022stnet}—are not considered optimal in isolation but as part of a multi-agent system.

Our use of this terminology requires further clarification beyond this intuition. A typical neural network trained via gradient descent to minimize a loss function is considered optimal with respect to the given loss. This optimality is defined over the parameter space of a fixed architecture. However, in our work, we perturb models layer-wise rather than as a whole. The optimality of an individual layer is not meaningful in isolation because it does not minimize the loss independently; instead, it contributes to loss minimization in conjunction with other layers. 

In our work, we manipulate parameters from multiple optimal models, introducing the following definitions:
\begin{defi}[Grafting]
    Let $\mathcal N^{\theta}$ denote a model with a  set of layers $\mathscr L$, where $\theta^{(i)}$ represents a set of parameters.
    
    A grafting of $\left(\mathcal N^{\theta^{(i)}}\right)_{i\in I}$ is a set of parameters $\overline \theta$ such that: 
    \begin{equation}
        \forall \ell\in\mathscr L, \exists i\in I, \overline \theta_\ell = \theta_\ell^{(i)}.
    \end{equation} 
\end{defi}

\begin{defi}[Co-Adapted Layers] Let $\mathcal L$ be a loss, $\mathcal N^\theta$ a model with layer set $\mathscr L$, and $\overline \theta$ a grafting of models. 

Layer parameters $\overline \theta_\ell$ for $\ell\in \mathscr L$ are said to be co-adapted if $\mathcal N^{\overline \theta}$ is optimal with respect to $\mathcal{L}$. 

By abuse of language, we refer to the layers themselves as co-adapted if their parameters are co-adapted.
\end{defi}

If the grafting originates from a singleton of an optimal model ($|I|=1$), the layers are trivially co-adapted. In our study, we perturb the layers of a model to ensure layer-wise optimality is preserved, while creating a non-optimal grafting of multiple optimal models. This process leads to layers losing their co-adaptation.

\section{NoT Algorithm} \label{not_algo}
Algorithm~\ref{alg_NoT} details our proposed unlearning method, which leverages weight negation.
\vspace{-0.9em}
\begin{algorithm}[H]
   \caption{NoT}
   \label{alg_NoT}
\begin{algorithmic}[1]
   \State {\bfseries Input:} Initialize global model $\mathcal N$. Each client $k\in \mathcal{P} := \{1,...,n\}$ has data $D^k$. Current round  is denoted as $\tau$, number of local iterations $I$, and learning rate $\eta$.
 \State \textcolor{blue}{\textbackslash\textbackslash \ Client Side}
         \For{client $k \in \mathcal{P}$}
            \State Client $k$ decides a target set $D^k_u\subset D^k$ and defines $D_r^k := D^k \setminus D^k_u$. 
            \If{client $k$ has $D_u^k\neq \emptyset$}
                \State Send unlearning request to server.
            \EndIf
        \EndFor
        \State \textcolor{purple}{\textbackslash\textbackslash \ Server Side}
        \If{unlearning request from clients $\mathcal{K} \subset \mathcal{P}$}
            \State Negate parameters of selected layers in $\mathcal N$.
        
   \For{communication round $\tau = 1, \cdots, \mathcal{T}$}
        \State $\mathcal{M} = \{\}$
        \For{client $k \in \mathcal{P}\setminus \mathcal{K}$}
            \State $\mathcal{M}[k] \gets \Call{ClientUpdate}{\mathcal N, k, I}$.
        \EndFor
        \State $\mathcal N \gets$ Aggregate $\mathcal{M}$
       \EndFor
        \EndIf
       
        \State \textcolor{blue}{\textbackslash\textbackslash \ Client Side}
        \Function{ClientUpdate}{$\mathcal N, k, I$}
            \For{local iteration $i = 1, ..., I$} \Comment{Local training}
                \For{minibatch $B^k \in$ local data $D^k_r$}
                    \State $\mathcal N \gets \mathcal N - \eta \nabla \mathcal{L}_{B^k}(\mathcal N)$   
                \EndFor
            \EndFor
            
            \State \Return $\mathcal N$
        
        \EndFunction
        
\end{algorithmic}
\end{algorithm}

\section{PyTorch Code} \label{code_not}
The PyTorch implementation of the main component of NoT, which applies layer-wise negation to a model.  
\begin{lstlisting}[language=Python]
import torch

def negate_layers(model, layer_indicies=[0]):
    '''
    Negates the layer-wise parameters of a model by layer indicies
    Args:
        model: The model that needs to be unlearned
        layer_indicies: Layer indicies to be negated
    returns 'unlearned model' 
    '''
    with torch.no_grad():
        for layer_indx, param in enumerate(model.parameters()):
            if layer_indx in layer_indicies:
                param.data = -1 * param.data
    return model
\end{lstlisting}
\section{Extra Theoretical Results and Mathematical Proofs} 
\subsection{Unlearning Time Lower Bound}
\label{unelarning_time_lowerbound}
\begin{theo}\label{theo:unlearning_time_constraint_exact} Let $\mathcal N^{\theta}$ be a model, and let $(D_r,D_u)$ denote a pair of datasets. Given an initial parameter set $\theta^0\in \Theta$, assume $\mathcal N^{\theta^0}$ is trained using Stochastic Gradient Langevin Descent (SGLD) to minimize $\mathcal L_{D_r}$ starting from $\theta^0$. The parameter evolution is described as: $d\theta^t  = -\nabla_{\theta^t} \mathcal L dt+\Sigma(\theta^t,t)\cdot dW$, where training time $t\geq 0$. Then the following holds:
\begin{equation}
 t \geq \frac{\mathbb E( \delta (\theta^t) - (\theta^0) )^2 }{L^2 \left[|\mathcal L_{D_r}(\theta^0)-\mathbb E\mathcal L_{D_r}(\theta^t)|+A\right]},    
\end{equation}
where: \begin{eqnarray}
 L&:=&\sup_{\theta_1\neq \theta_2}\frac{|\delta(\theta_1)-\delta(\theta_2)|}{\|\theta_1-\theta_2\|}\\
    A &=& \frac{1}{2}\int_0^t  \left|\mathrm{Tr}\left( \Sigma(\theta^s,s)^2 \cdot \nabla^2 \mathcal L_{D_r}(\theta^s) \right)\right|ds
\end{eqnarray}
\end{theo}
\begin{proof}
To begin with, let us prove the formula in the case of Deterministic Gradient Descent ($\Sigma = 0$). 
We have: 
\begin{align}
    d\theta^t &= -\nabla \mathcal L_{D_r}(\theta^t)dt \\
    d\mathcal L_{D_r}(\theta^t) &= -\nabla \mathcal L_{D_r}(\theta^t) \cdot \nabla \mathcal L_{D_r}(\theta^t) dt\\ &= -\|\nabla \mathcal L_{D_r}(\theta^t)\|^2dt\\
    \|\delta (\theta^t)-\delta(\theta^0)\| &\leq  \|\theta^t-\theta^0\| \|\delta \|_{\mathrm{Lip}}
\end{align}
where $ \|\delta \|_{\mathrm{Lip}} = \sup_{\theta_1\neq \theta_2} \frac{|\delta(\theta_1)-\delta(\theta_2)|}{\|\theta_1-\theta_2\|}$.

Then, 
\begin{align}
    \|\theta^t - \theta^0 \| &= (\theta^t - \theta^0) \cdot  \frac{ (\theta^t - \theta^0)}{\| (\theta^t - \theta^0)\|} \\ 
    &= \int_{0}^t (-\nabla \mathcal L_{D_r}(\theta^s)) \cdot   \frac{ (\theta^t - \theta^0)}{\| (\theta^t - \theta^0)\|} ds \\
    &\leq  \sqrt{\underbrace{\int_{0}^t \|\nabla \mathcal L_{D_r}(\theta^s)\|^2 ds}_{=|\mathcal L_{D_r} (\theta^t) -\mathcal L_{D_r} (\theta^0)| }   \underbrace{\int_0 ^t \left\|\frac{ \theta^t - \theta^0)}{\| (\theta^t - \theta^0)\|}\right\|^2 ds}_{=t}}  \nonumber \\
    \end{align}
    Therefore,
   \begin{equation}
     \|\theta^t - \theta^0 \| ^2 \leq    t |\mathcal L_{D_r} (\theta^t) -\mathcal L_{D_r} (\theta^0)|,
\end{equation}
and then: 
\begin{equation}
  \frac{   \|\delta(\theta^t) - \delta(\theta^0)\| }{\|\delta \|_{\mathrm{Lip}} |\mathcal L_{D_r} (\theta^t) -\mathcal L_{D_r} (\theta^0)| } \leq t.    
\end{equation}
The third line is obtained applying Cauchy-Schwarz inequality in the Hilbert space $L^2([0,t],\mathbb R^n)$.
The result is proved for $\Sigma =0$. When $\Sigma \neq 0$, we have: 
\begin{align}
    d\theta^t &= -\nabla \mathcal L_{D_r}(\theta^t)dt + \Sigma(\theta^t,t)\cdot dW \\
    d\mathcal L_{D_r}(\theta^t) &= -\|\nabla \mathcal L_{D_r}(\theta^t)\|^2dt + \nabla \mathcal L_{D_r} \cdot \Sigma(\theta^t,t) \cdot dW \nonumber \\&\quad 
    + \frac{1}{2}\mathrm{Tr}\left( \Sigma(\theta^t,t)^2 \cdot \nabla^2 \mathcal L_{D_r}(\theta^t) \right)dt
    \end{align}
    \begin{equation}
          \|\delta (\theta^t)-\delta  (\theta^0)\| \leq  \|\theta^t-\theta^0\| \times \|\delta\|_{\mathrm{Lip}}
    \end{equation}

Where the second line is obtained from the first applying It\^o Formula (Chapter IV in \cite{revuz2013continuous}).
The computation then unfolds the same way, then taking the expectancy and absolute value.

\end{proof}

\subsection{Activation Distance Maximization}
\label{sec:negation_bounds}
Maximizing the distance $d$ between activations serves as a straightforward heuristic for increasing loss. Specifically, for a layer $\ell$ activated by a function $\sigma$, denoting by $\ell'$ a replacement for $\ell$, we aim to maximize $d(\sigma \ell x,\sigma \ell'x)$ while ensuring that $\|\sigma\ell x\|\simeq \|\sigma\ell 'x\|$. We deem reasonable to assume that the later is approximately satisfied: for wide networks  $\ell x$ is close to a centered Gaussian distribution at initialization and our main case of interest is negation which preserves such Gaussian distributions. 
Let us consider specific cases for different activation functions.
\begin{enumerate}
\item \textbf{ReLU Activation ($\sigma : a \mapsto \max(0,a)$):}
For any $y_0 \in \mathbb R_+^n$, if at least one coordinate of $y_0$ is non-zero, then \begin{equation}
    \max_{\|y\|=\|y_0\|,y\geq0} d(y_0,y)^2=  2\|y_0\|^2.
\end{equation} If all coordinates of $y_0$ are positive, then 
\begin{equation}
     \max_{\|y\|=\|y_0\|,y\geq0} d(y_0,y)^2=  \|y_0\|^2.
\end{equation} The maximum is reached when $y$ is orthogonal to $y_0$.
Thus, the optimal scenario involves orthogonalizing the post-nonlinearity activations.

\item \textbf{Binary Step Function ($\sigma : a \mapsto \mathbf 1_{a>0}+\frac{1}{2}\mathbf 1_{a=0})$:} this function approximates sigmoid and similar variations. Here, the maximum distance is achieved by applying the Boolean negation to $y_0$. Notably, $\mathbf{not}~ \sigma(\ell x)= \sigma (-\ell x) $, setting $\ell'=-\ell$ ensures the desired property.\footnote{By extension of the boolean operator we define $\mathbf{not}(x):=1-x$}

\item \textbf{Hyperbolic Tangent (or any Odd Activation Function):} For such activations, the maximum distance is attained by $y=-y_0$. Since $\forall x, \|\sigma(\ell x)\|=\|\sigma(-\ell x)\|$ for odd functions, setting $\ell'=-\ell$ satisfies the requirement.
\end{enumerate}

Regarding the ReLU case, we have the following Lemma which implies Theorem~\ref{theo:maximize_dist}:
\begin{lem} Denote $\sigma(x):=\max(x,0)$ and let $Y \in \mathbb R^n$ be a random vector. Assume $\mathbb E\left|\|\sigma(Y)\|^2-\|\sigma(-Y)\|^2\right|\leq \varepsilon$ then: \begin{equation}
    \mathbb E \|\sigma(Y)-\sigma(-Y)\|^2\geq \mathbb E\left[\max_{y\geq 0,\|y\|=\|\sigma(Y)\|} d(\sigma(Y),y)^2\right] - \varepsilon.
\end{equation}
\end{lem}
\begin{proof} 
\begin{eqnarray}
     \mathbb E \|\sigma(Y)-\sigma(-Y)\|^2=& \mathbb E\mathbf 1_{Y\geq 0}\left[ \|\sigma(Y)\|^2+\|\sigma(-Y)\|^2\right]&\nonumber\\ \label{equ:decomposition}& + \mathbb E\mathbf 1_{Y\ngeq 0}\left[ \|\sigma(Y)\|^2+\|\sigma(-Y)\|^2\right]& \nonumber\\\end{eqnarray}
Define $I_1, I_2$ the terms on the right hand side. We have:
     \begin{eqnarray}
       I_1&=&  \mathbb E\mathbf 1_{Y\geq 0}\left[ \|\sigma(Y)\|^2+\|\sigma(-Y)\|^2\right]  \\&=&    \mathbb E\mathbf 1_{Y\geq 0} \|\sigma(Y)\|^2 \\ 
      &=& \mathbb E\mathbf 1_{Y\geq 0}\max_{\|y\|=\|\sigma(Y)\|,y\geq0} d(\sigma(Y),y)^2 \\\label{equ:positive}\\
   I_2&=&\mathbb E\mathbf 1_{Y\ngeq 0}\left[ \|\sigma(Y)\|^2+\|\sigma(-Y)\|^2\right] \\
     &=& \mathbb E\mathbf 1_{Y\ngeq 0}\left[ \|\sigma(-Y)\|^2-\|\sigma(Y)\|^2\right]\nonumber\\ &&\quad +  \mathbb E \mathbf 1_{Y\ngeq 0} 2\|\sigma(Y)\|^2 \\
     &\geq&  \mathbb E\mathbf 1_{Y\ngeq 0}\max_{\|y\|=\|\sigma(Y)\|,y\geq0} d(\sigma(Y),y)^2 -\varepsilon\label{equ:negative}
\end{eqnarray}
Substituting Equation~\eqref{equ:positive} and Equation~\eqref{equ:negative} in Equation \eqref{equ:decomposition} yields the desired result.
\end{proof}

\textbf{Alternative perturbations approaches:} 

$\bullet$ \textbf{Orthogonal Linear Transformations:} Applying an orthogonal linear transformation to $\ell$ post-activation seems promising. However, ensuring the transform works well across all activation values $\ell x$ is generally infeasible if the activation distribution contains a basis of the output vector space.

$\bullet$ \textbf{Weight Randomization:} Randomizing the weights of a dense or convolution layer ($\ell$) produces $\ell'$. If $X$ being the random variable input to the layer, the random variables $\ell X$ and $\ell 'X$ have low correlation, especially in high dimensions. Therefore, we expect the expected distance $\mathbb E d(\sigma \ell x,\sigma \ell'x)^2$ to be close to  $\mathbb E d(\sigma Y_1,\sigma Y_2)^2$, where $Y_1$ and $Y_2$ are independent. 
However, this approach is less effective, as demonstrated in the following lemma:
\begin{lem} Denote $\sigma(x):=\max(x,0)$ and let $Y_1,Y_2$ be IID standard Gaussian random arrays. 
\begin{equation}
    \mathbb E \|\sigma(Y_1)-\sigma(Y_2)\|^2= \alpha \mathbb E\left[\max_{y\geq 0,\|y\|=\|\sigma(Y_1)\|} d(\sigma(Y_1),y)^2\right],
\end{equation}
with $\alpha=\left(1-\frac{1}{\pi}\right)\in [0,1[$.
\end{lem}

\subsection{Jacobian Bound}
\label{sec:spectrum}

In this section, we only consider feedforward models consisting of dense layers, concatenations, splits and non-trainable activation functions. By feedforward, we mean that the computational graph is Directed Acyclic, not necessarily sequential.

The assumption on layers is not very restrictive: one may formally rewrite most common architectures using only dense layers with tied weights. We write bounds assuming non-tied weights, but they may be generalized, noticing that the gradient with respect to some tied weights is the sum of the tied layers of the untied gradients.

Let $E,F$ be finite dimensional normed spaces. 
Let us recall that the Sobolev space $W^{k,\infty}(E,F)$ is the set of function $E\rightarrow F$ whose derivatives up to order $k$ are $L^{\infty}$. Define  $W^{k,\infty}_1(E,F)$ the set of functions whose differential is in $W^{k-1,\infty}(E, E^* \otimes F))$.
For functions defined on a convex open subset of a finite-dimensional normed space $E$ to another $F$, we define the semi-norms:
\begin{equation}
    \|f\|_{\mathrm{Lip}}:= \sup_{x\neq y} \frac{\|f(x)-f(y)\|}{\|x-y\|},
\end{equation} 
For functions in $\mathrm{Lip} + L^{\infty}$ we define: 
\begin{equation}\|f\|_{\mathrm{Lip},\infty}:= \inf_{g+h=f} \max\left( \|g\|_{\mathrm{Lip}} ; 2\|h\|_{L^\infty}\right).
\end{equation}
All models are considered as a function of the variable $x$, so Lipchitz and $L^\infty$ norms are computed with respect to the $x$ input, not the parameters $\theta$. 
Let us begin with some elementary bounds.
\begin{lem}\label{lem:Wasserstein2} Let $E,F$ be two finite dimensional normed spaces and let $\mathcal U$ be an open subset of $E$.   Let $f:\mathcal U \rightarrow \mathbb F$  and let $\mathcal P(\mathcal U), \mathcal P(F)$ be the space of probability distributions having a first moment on $\mathcal U$ and $F$ respectively (for the Borel $\sigma$-algebra). Then $f^* : \mathcal P(\mathcal U) \rightarrow \mathcal P(F), \mu \mapsto f\# \mu$ is $\|f\|_{\mathrm{Lip}}$-Lipchitz for the metric $\mathcal W$ and 1-Lipchitz for the metric ${\mathrm{TV}}$.
\end{lem}
\begin{proof} On the one hand, for any coupling $\pi$ of $\mu,\nu$ we have:
    \begin{eqnarray}
        \mathcal W(f\# \mu, f\# \nu) &\leq& \mathbb E_{(x,y)\sim \pi} \|f(x)-f(y)\| \\ 
        &\leq& \|f\|_{\mathrm{Lip}} \mathbb E_{(x,y)\sim \pi} \|x-y\|  
    \end{eqnarray}
    We may take $\pi$ so that $ \mathbb E_{(x,y)\sim \pi} \|x-y\| = \mathcal W(\mu,\nu)$. 
    
    On the other hand, for any coupling $\pi$ of $\mu,\nu$:
    \begin{eqnarray}
      {\mathrm{TV}}(f\# \mu, f\# \nu) &\leq& \mathbb E_{(x,y)\sim \pi} \mathbf 1_{f(x)\neq f(y)} \\ &\leq& \mathbb E_{(x,y)\sim \pi} \mathbf 1_{x\neq y} 
    \end{eqnarray}
    We may take $\pi$ so that $\mathbb E_{(x,y)\sim \pi} \mathbf 1_{x\neq y} = {\mathrm{TV}}(\mu, \nu)$. 
\end{proof}

\begin{lem}\label{lem:TV-Wasserstein} Let $f\in L^{\infty}(\mathcal U,F)$ for some $\mathcal U\subset E$ open. Then 
\begin{equation}
     f^* : \left(\mathcal P(\mathcal U), {\mathrm{TV}}\right)\rightarrow \left(\mathcal P(F), \mathcal W \right)
\end{equation}
is   $2\|f\|_{\infty}$-Lipchitz.
\end{lem}
\begin{proof}
    See proof of Lemma~\ref{lem:Wasserstein1}.
\end{proof}

\begin{defi} Let $(A,\leq)$ be a partially ordered set (poset). Define the length of $A$ as the longest chain of $(A,\leq)$: 
\begin{equation}
    \mathrm{len}(A) := \max \{ |\gamma| ~:~   \text{totally ordered } \gamma \subset A\}.
\end{equation} 
We also define the diameter of $A$ as the number of distinct maximal chains in $A$: 
\begin{align}
    \Lambda(A)&:=  |\Gamma| \quad ;\\
    \Gamma(\mathscr L)&:= \{\gamma \subset A ~|~ \gamma \text{ totally ordered and maximal} \}. \nonumber
\end{align}
\end{defi}
\begin{defi}
    Let $\mathcal N^\theta$ be a feedforward neural net of layer poset $(\mathscr L,\leq)$. Define 
    \begin{equation}
        \|\theta\|_{2,\mathscr L} := \max_{\gamma \in   \Gamma(\mathscr L)} \left(\prod_{i=1}^{|\gamma|} \|\theta_{\gamma_i}\|_2 \right)^{1/|\gamma|}.
    \end{equation}
\end{defi}
\begin{lem} \label{lem:bound_infty_N}Let $\mathcal N^\theta$ be a feedforward neural net of layer poset $(\mathscr L,\leq)$ and activation function set $\Sigma\subset\mathrm{Lip}$. Assume that layers are either dense, split, or concatenation, then 
\begin{equation}
    \|\mathcal N^\theta\|_{\mathrm{Lip}} \leq \Lambda(\mathscr L)  \max\left(\|\Sigma\|_{\mathrm{Lip}}\|\theta\|_{2,\mathscr L},1\right)^{\mathrm{len}(\mathscr L)}
\end{equation}
with $\|\Sigma\|_{\mathrm{Lip}}:= \max_{\sigma\in \Sigma} \|\sigma\|_{\mathrm{Lip}}$.
\end{lem}
\begin{proof}
    Under these assumptions, the output of network may be rewritten as a projection of the concatenation of the output of sequential networks following the layers along maximal chains.  
    Therefore, with $\Gamma$ the set of maximal chains of $\mathscr L$, for any $x\in \mathbb R^{d_{\mathrm{in}}}$ 
    \begin{align}
        \|\mathcal N^\theta\|_{\mathrm{Lip}} &\leq \sum_{\gamma \in \Gamma} \|\sigma_{\gamma_{|\gamma|}} \theta_{\gamma_{|\gamma|}} \sigma_{\gamma_{|\gamma|-1}} \theta_{\gamma_{|\gamma|-1}}   \cdots \sigma_{\gamma_{1}} \theta_{\gamma_{1}} \|_{\mathrm{Lip}}\nonumber\\\\
        &\leq  \sum_{\gamma \in \Gamma} \left(\|\sigma_{\gamma_{|\gamma|}}\|_{\mathrm{Lip}} \|\theta_{\gamma_{|\gamma|}}\|_{\mathrm{Lip}} \|\sigma_{\gamma_{|\gamma|-1}} \| _{\mathrm{Lip}} \right.\nonumber\\ &\quad \left.\times  \|\theta_{\gamma_{|\gamma|-1}}\|_{\mathrm{Lip}}    \cdots \|\sigma_{\gamma_{1}} \|_{\mathrm{Lip}} \|\theta_{\gamma_{1}}\|_{\mathrm{Lip}}  \right) \\ &\leq 
          \sum_{\gamma \in \Gamma}\|\Sigma\|_{\mathrm{Lip}}^{|\gamma|}  \prod_{i=1}^{|\gamma|} \|\theta_{\gamma_i}\|_{2} \\ &\leq 
           \sum_{\gamma \in \Gamma}\|\Sigma\|_{\mathrm{Lip}}^{|\gamma|} \|\theta\|_{2,\mathscr L} ^{|\gamma|}\\
           &\leq  |\Gamma|\max \left(1,\|\Sigma\|_{\mathrm{Lip}}\|\theta\|_{2,\mathscr L}\right) ^{\mathrm{len}(\mathscr L)}.
    \end{align}
    The results follows for $\Lambda(\mathscr L):= |\Gamma|$.
\end{proof}

\begin{lem} \label{lem:bound_infty_nablaN}Let $\mathcal N^\theta$ be a feedforward neural net of layer poset $(\mathscr L,\leq)$ and activation function set $\Sigma\subset \mathrm{Lip}$. Assume $\mathcal N^{\theta}$ is defined on some open bounded domain $\mathcal U$ and assume that layers are either dense, split, or concatenation, then for any layer $\ell$:
\begin{equation}
    \|\nabla_{\theta_\ell}\mathcal N^\theta\|_{\infty} \leq \sup_{x\in \mathcal U} \|x\|\Lambda(\mathscr L)  \max\left(\|\Sigma\|_{\mathrm{Lip}}\|\theta\|_{2,\mathscr L},1\right)^{\mathrm{len}(\mathscr L)} 
\end{equation}
with $\|\Sigma\|_{\mathrm{Lip}}:= \max_{\sigma\in \Sigma} \|\sigma\|_{\mathrm{Lip}}$.
\end{lem}
\begin{proof}
    Proceed as for Lemma~\ref{lem:bound_infty_N}.
\end{proof}

\begin{theo}\label{theo:jacobian_full} Assume $\mathcal N^{\theta}$ is given by a Lipchitz feedforward neural network having also Lipchitz derivative of the layer set $\mathscr L$ and let $\leq$ be the order relation on $\mathscr L$ induced by the computational graph. Let $J^{\theta}_\ell:= \nabla_{\theta_\ell}\mathcal N^{\theta}(X)$ be its  Jacobian on $X\sim \mathcal D$ for layer $\ell \in \mathscr L$, and let $Y$ be the concatenation of the outputs of the layers $\ell \in \mathscr L_{\mathrm{neg}}$. Assume that $\mathscr L_{\mathrm{neg}}$ is a Cauchy domain of $\mathscr L$ (i.e., $\mathscr L_{\mathrm{neg}}$  intersects exactly once every  maximal totally ordered subsets of $\mathscr L$), then:
\begin{align} 
 \mathcal W\left(J^\theta_\ell;J^{\theta'}_\ell\right) \leq& A_\ell ~  {\mathrm{TV}}(\Yneg; -\Yneg),& \forall\ell> \mathscr L_{\mathrm{neg}}\\ 
\mathcal W\left(\epsilon J^\theta_\ell;J^{\theta'}_\ell\right)\leq& A_\ell ~ \mathrm{TV} ( (X, \Yneg); (X,-\Yneg )),& \forall\ell\leq \mathscr L_{\mathrm{neg}}
\end{align}
where $\epsilon=(-1)^{\ell\notin  \mathscr L_{\mathrm{neg}} }$,   $\mathcal W$ and  ${\mathrm{TV}}$ are the  Wasserstein and total variation distances, and  $(A_\ell)_{\ell \in \mathscr L}$ are positive constants depending on $\theta$ and the support of $X$.
\end{theo}
\begin{proof}

Define  $\theta_{\neq \mathrm{neg}} = (\theta_\ell)_{\ell \notin \mathscr L_{{\mathrm{neg}}}}$ and for layer subset $\mathscr L_1\leq \mathscr L_1$ denote by 
$\mathcal N^{\theta}_{\mathscr L_1 \rightarrow \mathscr L_2}$ the sub-model of $\mathcal N^\theta$ taking as input the concatenation of the outputs of layers in $\mathscr L_1$ and outputs the concatenation of the output of layers in $\mathscr L_2$. We take the convention that if $\mathscr L_1=\emptyset$ means the input is the input of the model $\mathcal N^{\theta}$ and $\mathscr L_2=\emptyset$, the output is the output of model $\mathcal N^\theta$. With this convention, $\mathcal N^\theta_{\emptyset \rightarrow \emptyset} = \mathcal N^\theta$. The submodel $\mathcal N^{\theta}_{\mathscr L_1 \rightarrow \mathscr L_2}$ is well defined if $\mathscr L_1$ is a Cauchy subset but may not be otherwise.
Note that differentiation may only be taken with respect to the weights of a dense layer since split, concatenation and activation functions don't have weights under our assumptions.

First, let $\ell > \mathscr L_{\mathrm{neg}}$. We denote by $X'$ a random variable having the same law as $X$ but with a coupling to $X$ to be chosen later. Since $\mathscr L_{\mathrm{neg}}$ is a Cauchy domain:
\begin{align}
    \|J^\theta_\ell-J^{\theta'}_\ell\|&:=\left\| \left.\nabla_{\theta_\ell} \left(\mathcal N^{\theta}_{\mathscr L_{\mathrm{neg}} \rightarrow \emptyset}    \mathcal N^{\theta}_{\emptyset \rightarrow \mathscr L_{\mathrm{neg}}}\right)\right|_{x=X} 
\right.\nonumber
\\ &\quad  \quad  
- \left.\left.\nabla_{\theta_\ell} \left( \mathcal N^{\theta'}_{\mathscr L_{\mathrm{neg}} \rightarrow \emptyset}   \mathcal N^{\theta'}_{\emptyset \rightarrow \mathscr L_{\mathrm{neg}}}  \right)\right|_{x=X'} \right\| 
    \\ &=
     \left\|\left. \left(\nabla_{\theta_\ell} \mathcal N^{\theta}_{\mathscr L_{\mathrm{neg}} \rightarrow \emptyset}   \right) \right|_{x=\mathcal N^{\theta}_{\emptyset \rightarrow \mathscr L_{\mathrm{neg}}}(X)}
\right.\nonumber
\\ &\quad  \quad 
- \left.\left.\left(\nabla_{\theta_\ell} ( \mathcal N^{\theta}_{\mathscr L_{\mathrm{neg}} \rightarrow \emptyset}  \right) \right|_{x= \mathcal N^{\theta'}_{\emptyset \rightarrow \mathscr L_{\mathrm{neg}}}(X')}  \right\| 
\end{align}
By Lemmata~\ref{lem:Wasserstein2} and \ref{lem:TV-Wasserstein} we get:
\begin{align}
 \mathcal W\left(J^\theta_\ell~;~J^{\theta'}_\ell\right)&     \leq         
2\|\nabla_{\theta_\ell}  \mathcal N^{\theta'}_{\mathscr L_{\mathrm{neg}} \rightarrow \emptyset}\|_{\infty}\\& \quad \times
 {\mathrm{TV}}( \mathcal N^{\theta}_{\emptyset \rightarrow \mathscr L_{\mathrm{neg}}}(X)~;~ \mathcal N^{\theta'}_{\emptyset \rightarrow \mathscr L_{\mathrm{neg}}}(X'))
 \\&            =            
2\|\nabla_{\theta_\ell}  \mathcal N^{\theta}_{\mathscr L_{\mathrm{neg}} \rightarrow \emptyset}\|_{\infty}{\mathrm{TV}}( \Yneg; - \Yneg)
\end{align}
Then, by Lemma~\ref{lem:bound_infty_nablaN}, we may set: \begin{equation}
    A_\ell = 2\left(\sup_{x\in \mathcal U} \|x\|\right)\Lambda(
\mathscr L_{\mathrm{neg}}^+ )  \max\left(\|\Sigma\|_{\mathrm{Lip}}\|\theta\|_{2,\mathscr L},1\right)^{\mathrm{len}(\mathscr L_{\mathrm{neg}}^+)}
\end{equation}
with $\mathscr L_{\mathrm{neg}}^+ := \{\ell' \in \mathscr L~|~\ell' > \mathscr L_{\mathrm{neg}} \}$.

Second, let $\ell < \mathscr L_{\mathrm{neg}}$ and proceed in the same way: 
\begin{align}
    \|J^\theta_\ell+J^{\theta'}_\ell\|&:=\left\| \nabla_{\theta_\ell} \left(\mathcal N^{\theta}_{\mathscr L_{\mathrm{neg}} \rightarrow \emptyset}    \mathcal N^{\theta}_{\emptyset \rightarrow \mathscr L_{\mathrm{neg}}}\right)  (X)
\right.\nonumber
\\ &\quad \quad  
+ \left.\nabla_{\theta_\ell} \left( \mathcal N^{\theta'}_{\mathscr L_{\mathrm{neg}} \rightarrow \emptyset}   \mathcal N^{\theta'}_{\emptyset \rightarrow \mathscr L_{\mathrm{neg}}}\right)   (X')  \right\| 
    \\ &=
     \left\| \left.\nabla_{x}  \mathcal N^{\theta}_{\mathscr L_{\mathrm{neg}} \rightarrow \emptyset} \right|_{x= \Yneg}  \nabla_{\theta_\ell}\mathcal N^{\theta}_{\emptyset \rightarrow \mathscr L_{\mathrm{neg}}}   (X)
\right.\nonumber
\\ &\quad \quad  
- \left.\left.\nabla_{x}  \mathcal N^{\theta}_{\mathscr L_{\mathrm{neg}} \rightarrow \emptyset} \right|_{x=- \Yneg'}  \nabla_{\theta_\ell}\mathcal N^{\theta}_{\emptyset \rightarrow \mathscr L_{\mathrm{neg}}}  (X')  \right\| 
\end{align}
Beware that this time, the terms can not be rewritten as a function of only $ \Yneg$.
We thus apply Lemma~\ref{lem:Wasserstein2} to the whole term as a function of the couple variable $(X, \Yneg)$: 
\begin{equation}(x_0,y_0)\mapsto \left(\left.\nabla_{x}  \mathcal N^{\theta}_{\mathscr L_{\mathrm{neg}} \rightarrow \emptyset} \right|_{x=y_0} \right) \left(\nabla_{\theta_\ell}\left.\mathcal N^{\theta}_{\emptyset \rightarrow \mathscr L_{\mathrm{neg}}}\right|_{x=x_0}\right).\end{equation}

Then, Lemmata~\ref{lem:bound_infty_nablaN} and \ref{lem:bound_infty_N}, allow to conclude the same way with: 
\begin{equation}A_\ell = \left(\sup_{x\in \mathcal U} \|x\|\right)\Lambda(
\mathscr L )  \max\left(\|\Sigma\|_{\mathrm{Lip}}\|\theta\|_{2,\mathscr L},1\right)^{\mathrm{len}(\mathscr L)} \end{equation}

Finally, the case $\ell\in \mathscr L_{\mathrm{neg}}$ is treated the same way.
\end{proof}

One may also prove similar bounds replacing the total variation by Wasserstein distance. However, the constant $A_\ell$ then depends on the Lipchitz norm of the derivative of the activation functions. If the activation functions do not have Lipchitz derivative, total variation is necessary as we only have the following Lemma to control the Wasserstein distance of output distributions.

\begin{lem}\label{lem:Wasserstein1} Let $f\in \mathrm{Lip}(\mathcal U,F)+L^{\infty}(\mathcal U,F)$ for some $\mathcal U\subset E$ open. Then, 
\begin{equation}f^* : \left(\mathcal P(\mathcal U), \mathcal W + {\mathrm{TV}}\right)\rightarrow \left(\mathcal P(F), \mathcal W \right)\end{equation}
is   $\|f\|_{\mathrm{Lip},\infty}$-Lipchitz.

\end{lem}

\begin{proof} For any coupling $\pi$ of $\mu,\nu$ and any decomposition $f=g+h$ with $g\in W_1^{1,\infty}$ and $h\in W_{0,\mathrm{LBV}}$, with $(x,y)\sim \pi$ we have:
\begin{align}
     \mathbb E\|f(x) - f(y)\| 
     &\leq   \mathbb E\left[\|g(x)-g(y)\| + \|h(x)-h(y)\| \right]\\
       &\leq   \mathbb E\left[ \|g\|_{\mathrm{Lip}} \|x-y\|+ \mathbf{1}_{x\neq y}\|h\|_\infty \right]
\end{align}
Since the inequality is true for any coupling $\pi$, we may choose a coupling of the form $\pi = D\# \min(\mu,\nu) + \pi'$ with $D:x\mapsto (x,x)$ and $\pi'$ a coupling between $(\mu-\nu)^+$ and $(\mu-\nu)^-$. Here $\mu\wedge\nu := \min\left(\frac{d\mu}{d(\mu+\nu)} ~;~ \frac{d\nu}{d(\mu+\nu)} \right)\times  (\mu+\nu)$ where the derivative denote the Radon-Nikodym derivative (see \cite{cohn2013measure} p125). We then have:\begin{eqnarray}
    \mathbb E\left[  \mathbf{1}_{x\neq y}\|h\|_\infty \right] &=& 2 \|h\|_\infty  {\mathrm{TV}}(\mu,\nu) 
\end{eqnarray}

and since $\inf_{\pi'}\mathbb E\left[\|x -y\|\right] =
\mathcal W((\mu-\nu)^+,(\mu-\nu)^-)$. Taking the infimum over $g+h=f$ and adding ${\mathrm{TV}}$ on both sides we get: 
\begin{align}
    \mathcal W(f\# \mu,f\#\nu) &\leq \inf_{f=g+h}\left(\|g\|_{\mathrm{Lip}} \mathcal W((\mu-\nu)^+~;~(\mu-\nu)^-)\right. \nonumber\\&\quad + \left.(1+2\|h\|_{L^\infty}) {\mathrm{TV}}(\mu,\nu) \right). 
\end{align}
Note that by cyclical monotonicity of the optimal transport plan (see \cite{villani2021topics} pp79-80), 
for all positive measures $\alpha,\beta,\gamma$ we have $\mathcal W(\alpha+\beta~;~ \alpha +\gamma) \geq \mathcal W(\alpha,\beta)$. Therefore, 
\begin{align}
    \mathcal W(\mu,\nu)&=\mathcal  \mathcal W(\mu\wedge\nu+(\mu-\nu)^+~;~\mu\wedge\nu+(\mu-\nu)^-)\nonumber \\&\geq  \mathcal W((\mu-\nu)^+~;~(\mu-\nu)^-)
\end{align}
so,
\begin{align}
    \mathcal W(f\# \mu,f\#\nu) &\leq \inf_{f=g+h}\max(\|g\|_{\mathrm{Lip}},(1+2\|h\|_{L^\infty}) )\nonumber \\&\quad \times \left(\mathcal W(\mu~;~\nu) + {\mathrm{TV}}(\mu,\nu) \right). 
\end{align}
Finally,
\begin{equation}
    \mathcal W(f\# \mu,f\#\nu) \leq \left(1+\|f\|_{\mathrm{Lip},\infty}\right) \left(\mathcal W(\mu~;~\nu) + {\mathrm{TV}}(\mu,\nu) \right). 
\end{equation}
\end{proof}

\subsection{Layer-Wise Optimality}
\label{sec:appendix_pretraining}
We introduce a new metric to quantify how far a layer of a model is from being effectively pretrained, as discussed in Section~\ref{sec:optmizable_state}: how much the layer has to be modified to become a layer with an optimal set of weights? 
\begin{defi}[Layer-wise optimality norm]
    Let $\mathcal N^{\theta}$ denote a model parameterized by $\theta\in \Theta$, and let $\ell$ be a layer of $\mathcal N^{\theta}$ with parameters $\theta_\ell$. The layer-wise optimality norm of layer $\ell$ is defined as:
        
\begin{equation}\|\theta\|_{*,\ell} := \inf_{\alpha \in \mathbb R_+,  \theta^*\in \Theta^*} \|\theta_\ell-\alpha \theta^*_\ell\|_2,  
\end{equation}
where $\theta^*$ represents the set of optimal parameters.

For a Lipchitz transformations $\sigma:\Theta\rightarrow \Theta$, acting solely on layer $\ell$, the layer-wise optimality norm of $\sigma$ is given by:  
\begin{equation}\|\sigma\|_{*,\ell}:= \inf_{A>0}\sup_{\theta\in \Theta} \left(  \|\mathcal N^{\sigma(\theta)}\|_{*,\ell}-A\|\mathcal N^{\theta}\|_{*,\ell}\right). 
\end{equation}
\end{defi}
\begin{defi} Let $\mathcal N^\theta$ be a model of layer set $\mathscr L$. A Lipchitz Lipchitz transformations $\sigma:\Theta\rightarrow \Theta$ is LWO-Lipchitz if:  
\begin{equation}
    \|\sigma\|_*:=\sum_{\ell \in \mathscr L}\|\sigma\|_{*,\ell}=0.
\end{equation}
    
\end{defi}
\paragraph{Remarks.} \ding{182} These definitions are meaningful only if  $\Theta^*\neq\emptyset$. Strictly speaking, $\|\cdot\|_{*,\ell}$ on $\Theta$ is not a norm but the  distance to a subset. \ding{183} The factor $\alpha$ ensures that the layer-wise optimality norm remains bounded, and it is reasonable since such a scaling factor is usually easy to recover via gradient descent. \ding{184} The norm $\|\sigma\|_{*,\ell}$  quantifies how well the transformation $\sigma$ preserves the optimality of layer $\ell$. LWO-Lipchitz property means $\sigma$ is Lipchitz for all the the layer-wise optimality norms. Thus, it preserves layer-wise optimality quantitatively.

A low layer-wise optimality norm does not guarantee rapid convergence but implies that the layer can be frozen while the rest of the model is trained from scratch. Even if not frozen, it is expected to reduce the effective dimensionality of the space explored by gradient descent. 

\subsubsection{Affine Compensation of Layer Negation.} A key point is that the usual sigmoid-like or odd activation functions $\psi$ satisfy an algebraic relation: \begin{equation}
    \forall x\in \mathbb R, ~\psi(-x)+\psi(x)=C.
\end{equation}
for some constant $C\in \mathbb R$.
\begin{lem}\label{lem:affine_compensate} Let $\ell_1,\ell_2$ be two linear layers, and let $\psi$ be an activation function. If $-\ell_1$ denotes the layer obtained by negating the parameters of $\ell_1$, then: 

If $\psi$ satisfies an algebraic relation of the form:     
\begin{multline}
    \exists a,b,c\in \mathbb R, \\ \quad \left(ab\neq0~\text{ and }~ \forall x\in \mathbb R, ~a\psi(x)+b\psi(-x)=c\right), 
\end{multline}  
there exists a linear layer $\ell_2'$ such that: 
\begin{equation} \ell_2\circ \psi \circ \ell_1  = \ell_2'\circ \psi \circ (-\ell_1).    
\end{equation}
Furthermore, if $\ell_2$ is convolutional, then $\ell_2'$ can also be chosen as convolutional.
\end{lem}
\begin{proof} Let $a,b,c\in \mathbb R$ satisfy the property above. Define $\ell_3(x) = c-\frac{b}{a} x$. 
Then $\ell_2' = \ell_2\circ \ell_3$ satisfies the desired properties. The layer $\ell_2'$ is linear and convolutional if $\ell_2$ is convolutional.
\end{proof}

Theorem~\ref{theo:negate_C2} is a consequence of the following Theorem.
\begin{theo}
The negation perturbation is LWO-Lipchitz  if $\mathscr L_{\mathrm{neg}}$ is an antichain of the poset $\mathscr L$ containing no maximal element, and each $\ell\in \mathscr L_{\mathrm{neg}}$ is activated by sigmoid-like, odd, or even functions (\eg, $\mathbf 1_{>0}$, $\tanh$, $\sin$, $x^2$). 
\end{theo}
\begin{proof}
 Without loss of generality we may assume that  $\Lneg$ is a singleton $\Lneg=\{\ell_1\}$.
    
    Let $\mathcal N^\theta$ be a feedforward neural network  with a set of linear or convolution layers $\mathscr L$, and let $\ell_1$ be a layer that is not the output layer. Denote the subsequent layer by $\ell_2$,  and assume the activation function $\psi$ of $\ell$ satisfies:  $\forall x\in \mathbb R, \psi(x)+ \psi(-x) = Cte$ or $\forall x\in \mathbb R, \psi(x)- \psi(-x) = Cte$.
    
    For any $\varepsilon>0$, let $(\theta^*_\ell)_{\ell\in \mathscr L}\in\Theta^*$ such that: $\| \theta_{\ell_1} -\theta^*_{\ell_1} \|\leq \|\mathcal N^{\theta}\|_{*,\ell_1}+\varepsilon$. Define $\widetilde\theta^* \in \Theta$ as follows: 
    \begin{itemize}
        \item $\widetilde \theta^*_\ell = \theta_\ell^*$ for $\ell\notin \{\ell_1,\ell_2\}$;
        \item $\widetilde \theta^*_{\ell_1} = -\theta^*_{\ell_1}$;
        \item $\widetilde \theta^*_{\ell_2}$ the parameters of the layer given by Lemma~\ref{lem:affine_compensate}.
        \end{itemize}
    Since:        
\begin{equation}\ell_2(\widetilde \theta^*_{\ell_2})\circ \psi \circ \ell(\widetilde \theta^*_{\ell_1}) = \ell_2( \theta^*_{\ell_2})\circ \psi \circ \ell_1(\theta^*_{\ell_1}),    
\end{equation}
we deduce that $\mathcal N^{\theta^*} = \mathcal N^{\widetilde\theta^*}$ as functions, and hence $\widetilde \theta^*\in \Theta^*$. Thus:     
\begin{align}
    \| \mathcal N^{\sigma(\theta)}\|_{*,\ell_1} &\leq 
    \|\sigma(\theta)_{\ell_1} - \widetilde \theta_{\ell_1}^*\|\\&=\|-\theta_{\ell_1} - (-  \theta_{\ell_1}^*)\|  \\&= \|\theta_{\ell_1} - \theta_{\ell_1}^*\|\\&\leq \| \mathcal N^\theta\|_{*,\ell_1} +\varepsilon .    
\end{align}
As this inequality holds for all $\varepsilon>0$, we conclude: $\| \mathcal N^{\sigma(\theta)}\|_{*,\ell_1} \leq  \| \mathcal N^\theta\|_{*,\ell_1}$. Consequently, $\|\sigma\|_{*,\ell_1}=0$.
\end{proof}
\section{Empirical Support of Theoretical Analysis}
This section presents empirical evidence supporting the hypotheses used in the theoretical analysis. Specifically, we examine the following:
\begin{enumerate}
    \item \textbf{CKA Analysis:} Demonstrates effective pretraining and the breaking of co-adaptation.
    \item \textbf{Post-Negation Fine-Tuning:} Validates the claim of dimensionality reduction during fine-tuning.
    \item \textbf{Unlearning Lower Bound:} Evaluates the unlearning lower bound to show that it is constraining for natural forgetting.
\end{enumerate}

\subsection{Quantitative Unlearning Time Constraint}
\label{sec:unlearning_time_quantitative}
We aim to assess the accuracy of the proposed  unlearning time lower bound. Recall the derived inequality:
\begin{equation}
t \geq \frac{\mathbb E( \delta ({\theta^t}) - \delta(\theta^0) )^2 }{\|\delta \|_{\mathrm{Lip}}^2\left[|\mathcal L_{D_r}(\theta^0)-\mathbb E\mathcal L_{D_r}(\theta^t)|+A\right]}:=t_\mathrm{unlearn} ,
\end{equation}
where $t_\mathrm{unlearn}$ represents the estimated unlearning time lower bound.

The goal is not to compute this precisely but to determine its order of magnitude. Each term in the formula can be approximated as follows:
    \begin{itemize}
        \item \textbf{Estimating $\delta$:} $\delta$ is computable for any model through a forward pass on both the remaining and forgotten data.
        \item \textbf{Estimating $\delta(\theta^t)-\delta(\theta^0)$:} For a model $\mathcal N^{\theta^0}$, we use: 
        \begin{equation}
            |\delta(\theta^t)-\delta(\theta^0) |\geq (1-\varepsilon)\left|\delta({\theta^*_\mathrm{Retrain}})-\delta(\theta^0)\right|,
        \end{equation}
        where $\theta^*_\mathrm{Retrain}$ represents the parameters of a model trained from scratch on $D_r$ and  $0<\varepsilon\ll 1$.
        \item \textbf{Estimating $\|\delta\|_{\mathrm{Lip}}$:} Since $\|\delta\|_{\mathrm{Lip}}=\|\nabla_\theta\delta\|_\infty$, we estimate it by computing the gradient of $\mathcal L_{D_r}$ and $\mathcal L_{D_u}$ at various $\theta$ and taking the supremum of the norm across for these $\theta$: \begin{equation}
            \|\delta\|_{\mathrm{Lip}}\simeq \max_{b\in \text{Batch}}  \|\nabla_{\theta} \delta(\theta^b)\|
        \end{equation} For a given training path, we choose to take the supremum of $\theta^t$ at different times. This gives a slightly sharper bound than taking the maximum across models, as NoT tends to have larger gradients than the other training.

        \begin{equation}
         \|\delta\|_{\mathrm{Lip}}\simeq 
            \max_{t \in \{t_1,\cdots,t_n\}} \|\nabla_{\theta} \delta(\theta^t)\|
        \end{equation}
        \item \textbf{Estimating $|\mathcal L_{D_r}(\theta^0)-\mathbb E\mathcal L_{D_r}(\theta^t)|$:} To estimate the available loss decrease, we use a lower bound given by: 
        \begin{equation}
            |\mathcal L_{D_r}(\theta^0)-\mathbb E\mathcal L_{D_r}(\theta^t)| \leq |\mathcal L_{D_r}(\theta^0)-\mathcal L_{D_r}(\theta^*)|,
        \end{equation}  
         where $\theta^*$ is obtained from a model trained from scratch on $D_r$ for an extended period (\eg, twice as long as the original training).
        \item \textbf{Estimating the Stochasticity Term:} For the term: 
        \begin{equation}\frac{1}{2}\int_0^t  \left|\mathrm{Tr}\left( \Sigma(\theta^s,s)^2 \cdot \nabla^2 \mathcal L_{D_r}(\theta^s) \right)\right|ds,\end{equation}
        we assume $\Sigma(\theta^s,s)$ is diagonal and use Formula 8 from \citep{stephan2017stochastic} to approximate $\Sigma(\theta^s,s)^2$. The Hessian $\nabla^2 \mathcal L_{D_r}$ is also approximated as diagonal, reducing the trace to a product of the diagonal elements.
    \end{itemize}

\paragraph{Results.} The compiled results are presented in Table~\ref{tab:unlearning_bound}, alongside computational cost estimates, for comparison with experimental results in Table~\ref{table_dif_dataset_arch}.
\begin{table}
\centering
\caption{Comparison of the estimated unlearning bound $t_{\mathrm{unlearn}}$ (right side of Equation~\eqref{eq:unlearning_time}).} 
\vspace{-3mm}
\resizebox{8.5cm}{!}{
\begin{tabular}{llll}
\toprule
\multirow{2}{*}{\shortstack[c]{\textbf{Dataset} \\ \textbf{ \& Model}}} & \multirow{2}{*}{\shortstack[c]{\textbf{Method}}}  & \multirow{2}{*}{\shortstack[c]{\textbf{$t_{\mathrm{unlearn}}$}  $\downarrow$}}  & \multirow{2}{*}{\shortstack[c]{\textbf{Est. Comp. Cost} \\ \textbf{(FLOPs) $\downarrow$}}} \\ \\
\midrule
\multirow{2}{*}{\shortstack[c]{CIFAR-10 \\ CNN}} & FT & 2.46 & 1.51$e^{14}$ \\
& NoT & 0.04 & 2.33$e^{12}$ \\
\midrule
\multirow{2}{*}{\shortstack[c]{CIFAR-10 \\ ResNet-18}} & FT & 4.77 & 3.00$e^{16}$ \\
& NoT &  0.0017 & 1.04$e^{13}$  \\
\midrule
\multirow{2}{*}{\shortstack[c]{Caltech-101 \\ ViT}} & FT & 0.299 & 1.74$e^{19}$\\
& NoT & 0.0006  & 3.56$e^{16}$ \\
\bottomrule
\end{tabular}
}
\label{tab:unlearning_bound}
\end{table}
\paragraph{Observations:}
\begin{itemize}
    \item \textbf{Stochasticity Term:} The term $\Sigma(\theta^s,s)^2$ is orders of magnitude smaller than the available loss decrease, allowing us to neglect its contribution.
    \item \textbf{Comparison of FT and NoT:} We obtain a significant difference between FT and NoT with a predicted cost ratio close to the experimental cost ratio.
    \item \textbf{Accuracy of the Bound:
} The theoretical lower bound for $t_{\mathrm{unlearn}}$ is significantly lower than empirical results, suggesting potential for refinement by incorporating details of the gradient descent trajectory.
\end{itemize}

\subsection{Further CKA Analysis}
\label{appendix:CKA}
Figure~\ref{fig:cka_cnn}, presented CKA comparisons for a CNN model, demonstrating alignment with the theoretical analysis. Here, we extend CKA analysis to ResNet-18 and ViT models, confirming that \textbf{layer negation disrupts co-adaptation while preserving similar features}. Figure~\ref{fig:cka_resnet_vit} shows the CKA similarity between the original (FT), random (Retrain), and layer-wise-negation (NoT) models, before ($\tau$:0) and after ($\tau$:-1) fine-tuning. The FT model at $\tau$:0 serves as the reference. 
\begin{itemize}
    \item \textbf{ResNet-18:} The first convolutional layer (index $\ell$:0) is negated. Activations from the ReLU following each layer and the head are analyzed.
    \item \textbf{ViT:} The convolutional projection layer (index $\ell$:1) is negated. Focus is on Conv\_Proj, Encoder\_0, and Head activations.
\end{itemize}
 
Comparing NoT to FT at round $\tau$:0 reveals greater divergence at deeper layers, reflecting disrupted co-adaptation. In ViT, we can see the significant CKA dissimilarity for the subsequent layer (Encoder\_0) and output (Heads). However, the transformer residual connections seem to bring features back to high similarity with a decreasing trend in depth. Since this behavior is observed across perturbations and even randomized network shows similar pattern, we normalize the CKA via $\frac{\mathrm{CKA}_{Q}-\mathrm{CKA}_{Retrain@\tau:0}}{1-\mathrm{CKA}_{Retrain@\tau:0}}$ to clarify these results, where $Q$ is any model. After fine-tuning, Retrain and NoT ($\tau$:-1) exhibit similar CKA but differ from FT. This demonstrates that after fine-tuning, NoT becomes closely similar to Retrain. Moreover, an indication of NoT preserving effective pretraining benefits is by comparing two random models, at $\tau$:0, where the first layer has the original weights while the other has the negation of the original weights as in FT@$\tau$=0. The high CKA similarity between those two models confirms the effective pretraining of negating the first layer.

\begin{figure*}[t!]
\setlength\tabcolsep{1pt}
\begin{tabularx}{\textwidth}{c}
\includegraphics[width=\linewidth]{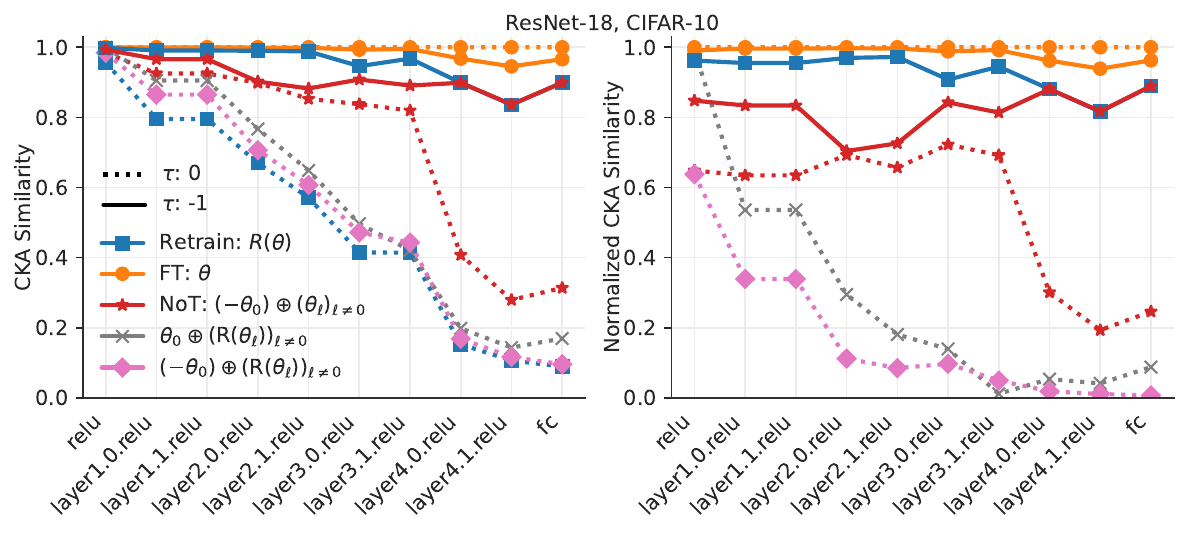}\\
\includegraphics[width=\linewidth]{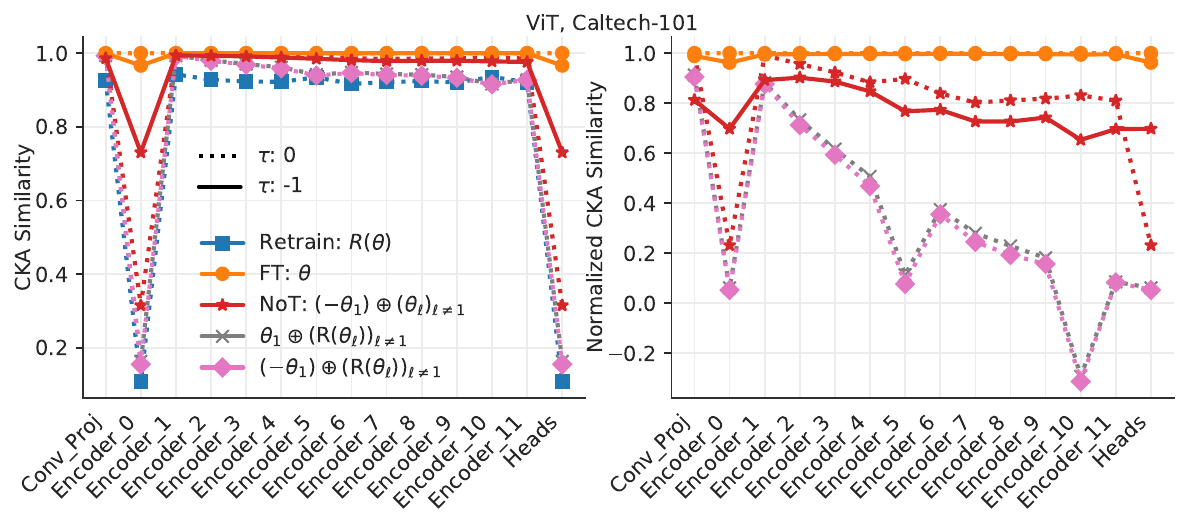}
\end{tabularx}
\vspace*{-5mm}
\caption{CKA (left) and normalized CKA (right) for \textbf{ResNet-18} (top) and \textbf{ViT} (bottom) layer activations, compared to the original model ($\theta=\theta^*$) before fine-tuning (FT@$\tau$:0). $\tau$:0 and -1 denote the first and last communication rounds. Models with negated first-layer weights ($\ell$:0 for ResNet-18 and $\ell$:1 for ViT) and randomized (reinitialized) remaining layers are denoted as $(-\theta_0)\oplus \left( \mathrm{R}(\theta_\ell)\right)_{\ell\neq 0}$, where $R(\cdot)$ refers to reintitializing. We normalize the CKA via $\frac{\mathrm{CKA}_{Q}-\mathrm{CKA}_{Retrain@\tau:0}}{1-\mathrm{CKA}_{Retrain@\tau:0}}$ to clarify these results, where $Q$ is any model.} 
\label{fig:cka_resnet_vit}
\end{figure*}

\subsection{Direct Test of Layer-wise Optimality}
\label{appendix:LWO}

By definition, Layer-wise Optimality (LWO) requires that any layer $\ell$ in a model $\mathcal N^{\theta}$ can be frozen while the remaining layers are reinitialized, and the model can still be fine-tuned to achieve optimal performance. We validate this property by directly comparing the accuracy of models subjected to this process with models trained from scratch.

We test the LWO of $\mathcal N^{\theta'}$, defined as:
 \begin{equation}\theta':=(-\theta_{\ell})_{\ell \in \Lneg}\oplus (\theta_\ell)_{\ell\in\mathscr L\setminus \Lneg}.\end{equation} 
with $\mathcal N \in \{\text{CNN}, \text{ResNet-18}, \text{ViT}\}$; $\Lneg$ only contains the first layer for CNN and ResNet-18, and only the convolutional projection for ViT. In all cases, layer-wise optimality of $\ell \in \mathscr L\setminus \Lneg$ is obvious since they are layers of an optimal model $\mathcal N^\theta$. Therefore, we only need to test layer-wise optimality of $\ell \in \Lneg$. 

\textbf{Procedure:} We freeze $\theta_{\ell}$ for $\ell \in \Lneg$, randomize (reinitialize) $\theta_{\ell'}$ for $\ell'\notin \Lneg$, and fine-tune. \textbf{Results:} Table~\ref{table_lwo} shows that NR-Freeze achieves performance on par with models trained from scratch (Retrain), confirming the theoretical prediction that negation preserves layer-wise optimality. The same stopping conditions as for Table~\ref{table_dif_dataset_arch} to compute the costs are used. 

\begin{table*}[t!]
\centering
\caption{\textbf{Layer-Wise Optimality (LWO)} test of negated models via reinitialization of non-negated layers followed by fine-tuning. Retrain denotes a model trained from scratch, and NR-freeze denotes a model obtained via \textbf{Negating \& Freezing $\Lneg$, \textbf{Reinitializing} $\mathscr L\setminus \Lneg$, and fine-tuning}. Results confirm that negation preserves layer-wise optimality, confirming our theoretical prediction that negation is Layer-Wise Optimality Preserving. }
\vspace{-3mm}
\resizebox{17.5cm}{!}{
\begin{tabular}{llllllllll}
\toprule
\multirow{2}{*}{\shortstack[c]{\textbf{Dataset} \\ \textbf{ \& Model}}} & \multirow{2}{*}{\shortstack[c]{\textbf{Method}}} & \multicolumn{3}{c}{\textbf{Accuracy (\%)}} & \multicolumn{1}{c}{\textbf{Privacy (\%)}} & \multicolumn{1}{c}{\multirow{2}{*}{\shortstack[c]{\textbf{Avg.} \\ \textbf{Gap $\downarrow$}}}} & \multicolumn{2}{c}{\textbf{Cost (Bytes \& FLOPs)}} \\
\cmidrule(r){3-5} \cmidrule(r){6-6} \cmidrule(r){8-9}  
&& \multicolumn{1}{c}{\textbf{Retain (\textcolor{blue}{$\Delta \downarrow$})}} & \multicolumn{1}{c}{\textbf{Forget (\textcolor{blue}{$\Delta \downarrow$})}} & \multicolumn{1}{c}{\textbf{Test (\textcolor{blue}{$\Delta \downarrow$})}} & \multicolumn{1}{c}{\textbf{MIA (\textcolor{blue}{$\Delta \downarrow$})}} & & \multicolumn{1}{c}{\textbf{Comm. $\downarrow$}} & \multicolumn{1}{c}{\textbf{Comp. $\downarrow$}} \\
\midrule
\multirow{2}{*}{\shortstack[c]{CIFAR-10 \\ CNN}} & Retrain & 91.66\tiny{$\pm$ 0.12} \normalsize{(\textcolor{blue}{0.00}}) & 83.05\tiny{$\pm$ 0.23} \normalsize{(\textcolor{blue}{0.00}}) & 82.32\tiny{$\pm$ 0.30} \normalsize{(\textcolor{blue}{0.00}}) & 50.23\tiny{$\pm$ 0.39} \normalsize{(\textcolor{blue}{0.00}}) & \normalsize{\textcolor{blue}{0.00}} & 1.35$e^{10}$ & 5.81$e^{16}$ \\

\cline{2-9}
&  NR - Freeze & 91.71\tiny{$\pm$ 0.11} \normalsize{(\textcolor{blue}{0.05}}) & 82.93\tiny{$\pm$ 0.17} \normalsize{(\textcolor{blue}{0.12}}) & 82.32\tiny{$\pm$ 0.28} \normalsize{(\textcolor{blue}{0.00}}) & 50.23\tiny{$\pm$ 0.12} \normalsize{(\textcolor{blue}{0.00}}) & \normalsize{\textcolor{blue}{0.04}} & 1.04$e^{10}$ & 4.51$e^{16}$ \\

\hline

\multirow{2}{*}{\shortstack[c]{CIFAR-100 \\ CNN}} & Retrain & 72.32\tiny{$\pm$ 0.11} \normalsize{(\textcolor{blue}{0.00}}) & 53.31\tiny{$\pm$ 0.87} \normalsize{(\textcolor{blue}{0.00}}) & 54.28\tiny{$\pm$ 0.25} \normalsize{(\textcolor{blue}{0.00}}) & 49.70\tiny{$\pm$ 0.64} \normalsize{(\textcolor{blue}{0.00}}) & \normalsize{\textcolor{blue}{0.00}} & 1.38$e^{10}$ & 5.96$e^{16}$ \\
\cline{2-9}
&  NR-Freeze & 72.46\tiny{$\pm$ 0.51} \normalsize{(\textcolor{blue}{0.14}}) & 53.34\tiny{$\pm$ 0.79} \normalsize{(\textcolor{blue}{0.03}}) & 54.53\tiny{$\pm$ 0.48} \normalsize{(\textcolor{blue}{0.25}}) & 49.77\tiny{$\pm$ 0.53} \normalsize{(\textcolor{blue}{0.07}}) & \normalsize{\textcolor{blue}{0.12}} & 1.23$e^{10}$ & 5.30$e^{16}$ \\
\hline
\multirow{2}{*}{\shortstack[c]{CIFAR-10 \\ ResNet-18}} & Retrain & 100.00\tiny{$\pm$ 0.00} \normalsize{(\textcolor{blue}{0.00}}) & 87.66\tiny{$\pm$ 0.64} \normalsize{(\textcolor{blue}{0.00}}) & 87.73\tiny{$\pm$ 0.35} \normalsize{(\textcolor{blue}{0.00}}) & 49.37\tiny{$\pm$ 0.29} \normalsize{(\textcolor{blue}{0.00}}) & \normalsize{\textcolor{blue}{0.00}} & 1.23$e^{12}$ & 5.66$e^{18}$ \\
\cline{2-9}
&   NR - Freeze & 99.98\tiny{$\pm$ 0.00} \normalsize{(\textcolor{blue}{0.02}}) & 87.47\tiny{$\pm$ 0.05} \normalsize{(\textcolor{blue}{0.19}}) & 86.72\tiny{$\pm$ 0.02} \normalsize{(\textcolor{blue}{1.01}}) & 49.70\tiny{$\pm$ 0.70} \normalsize{(\textcolor{blue}{0.33}}) & \normalsize{\textcolor{blue}{0.39}} & 5.98$e^{11}$ & 2.75$e^{18}$ \\

\hline

\multirow{2}{*}{\shortstack[c]{CIFAR-100 \\ ResNet-18}} & Retrain & 99.96\tiny{$\pm$ 0.00} \normalsize{(\textcolor{blue}{0.00}}) & 59.96\tiny{$\pm$ 0.61} \normalsize{(\textcolor{blue}{0.00}}) & 60.66\tiny{$\pm$ 0.63} \normalsize{(\textcolor{blue}{0.00}}) & 50.30\tiny{$\pm$ 0.30} \normalsize{(\textcolor{blue}{0.00}}) & \normalsize{\textcolor{blue}{0.00}} & 7.34$e^{11}$ & 3.38$e^{18}$ \\

\cline{2-9}

& NR - Freeze & 99.98\tiny{$\pm$ 0.13} \normalsize{(\textcolor{blue}{0.01}}) & 60.80\tiny{$\pm$ 0.24} \normalsize{(\textcolor{blue}{0.82}}) & 60.30\tiny{$\pm$ 0.31} \normalsize{(\textcolor{blue}{1.67}}) & 50.00\tiny{$\pm$ 0.22} \normalsize{(\textcolor{blue}{0.90}}) & \normalsize{\textcolor{blue}{0.85}} & 7.43$e^{11}$ & 3.42$e^{18}$ \\

\hline
\multirow{2}{*}{\shortstack[c]{Caltech-101 \\ ViT}} & Retrain & 99.73\tiny{$\pm$ 0.04} \normalsize{(\textcolor{blue}{0.00}}) & 48.29\tiny{$\pm$ 0.44} \normalsize{(\textcolor{blue}{0.00}}) & 48.02\tiny{$\pm$ 0.72} \normalsize{(\textcolor{blue}{0.00}}) & 49.67\tiny{$\pm$ 3.47} \normalsize{(\textcolor{blue}{0.00}}) & \normalsize{\textcolor{blue}{0.00}} & 1.76$e^{12}$ & 1.37$e^{21}$ \\
\cline{2-9}
&   NR - Freeze & 99.71\tiny{$\pm$ 0.08} \normalsize{(\textcolor{blue}{0.02}}) & 48.41\tiny{$\pm$ 1.03} \normalsize{(\textcolor{blue}{0.12}}) & 48.10\tiny{$\pm$ 0.37} \normalsize{(\textcolor{blue}{0.08}}) & 49.73\tiny{$\pm$ 1.59} \normalsize{(\textcolor{blue}{0.06}}) & \normalsize{\textcolor{blue}{0.07}} & 1.55$e^{12}$ & 1.21$e^{21}$ \\

\bottomrule
\end{tabular}
}
\label{table_lwo}
\end{table*}

\subsection{Spectral Content of Gradient Covariance}
\label{appendix:dim_reduction}

\def\din{d_{\mathrm{in}}}
\def\dout{d_{\mathrm{out}}}
Given a model $\mathcal N^\theta:\mathbb R^{d_{\mathrm{in}}}\rightarrow \mathbb R^{d_{\mathrm{out}}}$ with $\theta\in \mathbb R^d$, define $B$ the minibatch random variable with value in $\mathbb R^{\din \times b}$ with $b$ the bath size, $(\theta^t)_{t\in [0,t_{\mathrm{max}}]}$ is the parameter vector of the model across gradient descent and $\Sigma:= \mathrm{Cov}(\nabla_\theta \mathcal L_{B}(\mathcal N^{\theta^T}))$ the covariance matrix of the random variable whose value is the gradient of the model on a random batch $B$ at a random time $T\in [0,t_{\mathrm{max}}]$. For our simulations, $t_\mathrm{max}=0$.
Consider $\mathrm{Sp}(\Sigma)=\{\lambda_1,\cdots,\lambda_d\}$ the spectrum of the covariance matrix with $\lambda_1>\lambda_2>\cdots>\lambda_d$. 
We are interested in visualizing the spectral content curves: \begin{equation}
    \Psi:\alpha \mapsto \frac{\sum_{i=1}^{[\alpha d]}\lambda_i}{\sum_{i=1}^d \lambda_i},\quad \alpha \in [0,1].\label{equ:spectral_content}
\end{equation}
This spectral content measures the fraction of eigenvectors generating the linear space generated by the gradients (up to some noise threshold). More precisely, for $\beta = 95\%$, define $\alpha_\beta :=\min  \{\alpha ~|~ \Psi(\alpha)\geq \beta\}$. The quantity $[\alpha_\beta d]$ measures the minimal number of dimension for a linear subspace to contain $95\%$ of the Euclidean norm squared of the noise of the gradient across training.

The most direct approach to evaluation the function $\Psi$ is based on a direct use of the spectrum of the empirical covariance matrix \citep{anderson1963asymptotic}. First, we evaluate the gradient $(X_{i})_{i=1..b} := (\nabla_\theta \mathcal L_{B_i}(\mathcal N^{\theta^0}))_{i=1..b}$ for $b\in \mathbb N$ randomly sampled mini-batchs $(B_i)_{i=1}^b$; and where $\theta^\tau$ is the parameter vector of the model at communication round $\tau$. We sample a random subsets $K_1,\cdots,K_p\subset \{1,\cdots,d\} $ of model parameters and project each $X_i$ onto $\mathbb R^{K_j}$ to get $X^{K_j}_i$. 
Then, we compute the spectrum $S_j:=\mathrm{Sp}(\mathrm{Cov}(X^{K_j}_{i},X^{K_j}_{i}))=\{\lambda_1,\cdots,\lambda_k\}$  of the empirical covariance matrix. Finally, we compute our estimators $\widehat\Psi_j$ with Equation~\eqref{equ:spectral_content} and $\widehat \Psi := \frac{1}{p}\sum_{j=1}^p \widehat \Psi_j$.

Beware that, this direct estimator of the spectrum of $\Sigma$ is inconsistent \citep{geman1980limit,yin1988limit,101214009117906000000917} as we are in the setting of very high dimensional random vectors (long-story short, covariance matrices converge toward Gaussian ensembles, the  spectrum thus converges toward Tracy-Widom distribution \citep{mingo2017free}). As a result, although still barely meaningful for shallow CNN, it yields inconclusive results for larger models (ResNet and ViT). We thus implement the $L^\infty$ version of the algorithm described in \citep{10.1214/07-AOS581} to extract meaningful information from the empirical spectrum  of $\Sigma$ obtained by the method above. 

\begin{figure*}[t!]
\setlength\tabcolsep{1pt}
\begin{tabularx}{\textwidth}{c}
\includegraphics[width=0.33\linewidth]{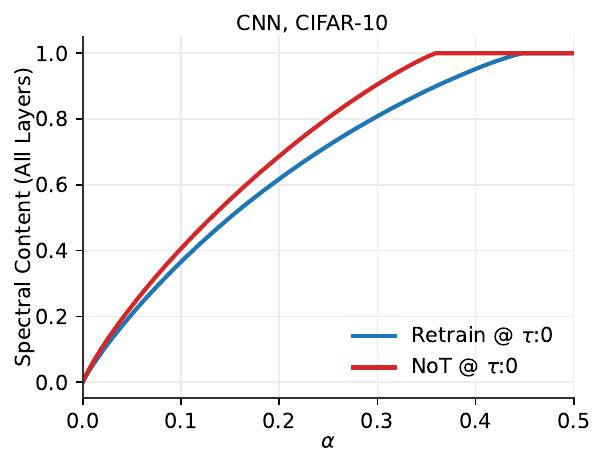} 
\includegraphics[width=0.33\linewidth]{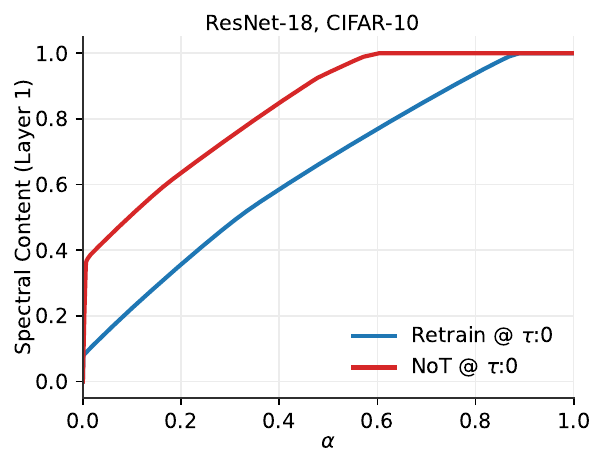} 
\includegraphics[width=0.33\linewidth]{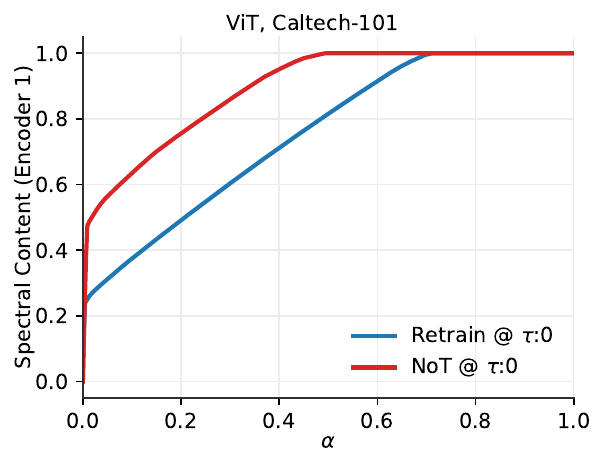} 
\end{tabularx}
\vspace*{-5mm}
\caption{Spectral content of CNN (left), ResNet-18 (mid) and ViT (right) models. CNN:  batch size $64$, $b=2048$, $|K|=2000$ and $p=100$ with $K$ with indices in all layers. ResNet: batch size $64$, $b=128$, $|K|=2000$ and $p=100$ with $K$ with indices in only the first residual block. ViT: batch size $16$, $b=128$, $|K|=2000$ and $p=100$ with $K$ with indices in only the first encoder. In all cases, we observe that the spectral content reaches 100\% for NoT with between 20\% and 30\% less dimensions: $\Psi^{-1}_{\mathrm{NoT}}(1^-)/\Psi^{-1}_{\mathrm{Retrain}}(1^-)\in [0.7, 0.8]$. For ViT and ResNet, we focus on the first layer after the convolution projection on which dimensionality reduction is the strongest and because the size of these models makes it increasingly difficult to compute an accurate spectrum estimator.
}
\label{fig:spectral_content_CNN_resnet}
\end{figure*}

\section{Further Implementation Details}
\label{further_imp_details}
This section details the hyperparameters used for the experiments in this paper. 
\subsection{Federated Learning} 
\label{fl_further_imp_details}
The global model was trained until convergence using the local data of all 10 clients using SGD with 0.9 momentum, 0.001 learning rate, and 5e-4 for weight decay. In every communication round, all clients participate in the federation. For all the experiments, except for CNN architectures, we upscaled the images to 256 and utilized random cropping to 224 and random horizontal flip with 0.5 probability. For CNN experiments, we used the original image sizes of 32 with random horizontal flip with 0.5 probability. Local client training is set to 1. The global model was trained for communication rounds $\mathcal{T}$ = 2000 (CNN + CIFAR-10), 2000 (CNN + CIFAR-100), 800 (ResNet-18 + CIFAR-10), 1000 (ResNet-18 + CIFAR-100), 300 (ViT + Caltech-101), and 500 (ViT + CIFAR-100).

Below are the parameters we used for each unlearning algorithm, where we select client 0 to be the target client and assume unlearning starts at round $\mathcal{T}$ and lasts for another $\mathcal{T}$ number of rounds. We utilized the same parameters for training of the global model unless stated otherwise below. We fine-tune all the hyperparameters to get the lowest average gap for the same number of communication rounds as the training of the global model.
\begin{itemize}

\item \textbf{Retrain.} We use the same global values for training and save the model after $\mathcal{T}$ rounds.

\item \textbf{FT.} We use the same global values for training.

\item \textbf{FedEraser.} For \textit{CNN}: We use the same global values for training.

\item \textbf{FUKD.} For \textit{CNN}: learning rate for unlearning is set to 1e-4, momentum to 0.9, distillation epochs to 3, and temperature to 3. 

\item \textbf{PGD.} For \textit{CNN}: learning rate for unlearning is set to 0.1, unlearning iterations to 10, momentum to 0.9, updates per iteration to 5, distance threshold to 3, and clip gradient to 2. For \textit{ResNet-18}: learning rate for unlearning is set to 0.001, unlearning iterations to 2, momentum to 0.9, updates per iteration to 7, distance threshold to 2.2, and clip gradient to 1. \textit{ViT}: learning rate for unlearning is set to 0.001, unlearning iterations to 10, momentum to 0.9, updates per iteration to 5, distance threshold to 1, and clip gradient to 4.

\item \textbf{MoDE.} For \textit{CNN}: memory guidance rounds is set to 3 with 2 MoDE rounds, MoDe coefficient to 0.4, learning rate for the models is 0.005, and the learning rate for the degradation model is 0.1. For \textit{ResNet-18}: memory guidance rounds is set to 9 with 8 MoDE rounds, MoDe coefficient to 0.95, learning rate for the models is 0.05, and the learning rate for the degradation model is 0.1. \textit{ViT}: memory guidance rounds is set to 15 with 12 MoDE rounds, MoDe coefficient to 0.9, and learning rate for all models is 0.001.

\item \textbf{FCU.} An Adam optimizer is used where the momentum terms are set to 0.9 and 0.99, and a ReduceLROnPlateau scheduler with learning rate 0.1 as the starting point while reducing it 1e-7 with a factor of 0.1 and patience of 2. For \textit{CNN}: learning rate for unlearning is set to 0.01, unlearning iterations to 20, fusion interval to 10, and low-frequency to 0.9. For \textit{ResNet-18}: learning rate for unlearning is set to 0.01, unlearning iterations to 10, fusion interval to 10, and low-frequency to 0.7. \textit{ViT}: learning rate for unlearning is set to 0.01, unlearning iterations to 2000, fusion interval to 9, and low-frequency to 0.7.

\item \textbf{NoT.} For all architectures we negate the first layer, except efficient net we negate the weights of all convolution and dense layers with indices $\leq150$. For \textit{ViT}, we negate the convolution projection (conv\_proj) layer.
\end{itemize}
\subsection{Centralized Training} 
\label{cl_further_imp_details}

In the following, we provide the hyperparameters we used for our baselines in Table~\ref{tbl_centralized_results}. For the base model, we used ResNet-18 with CIFAR-10 dataset where we used the official train set to train the base model using SGD with 0.9 momentum, 0.01 learning rate, and 5e-4 for weight decay. We also used cosine annealing for 200 epochs with a minimum of 0.0001 and saved the best model with the highest accuracy on the test set. For all the experiments, we upscaled the CIFAR-10 images to 256 and utilized random cropping to 224 and random horizontal flip with 0.5 probability. 

Below are the parameters we used for each unlearning algorithm, where we randomly select 10\% of the data as forget data. We utilized the same parameters for training of the base model unless stated otherwise below. We fine-tune all the hyperparameters to get the lowest average gap for 50 epochs.

\begin{itemize}
     
\item \textbf{Retrain.} We use the same base values for training and save the model with highest test accuracy.

\item \textbf{FT.} We use the same base values for training.

\item \textbf{RandL.} We use the same base values for training. Further, at each round, we iterativly train the model on random labels for one epoch and then one epoch on retain data.

\item \textbf{GA.} We use the same base values for training, and we do one epoch of GA and 49 epochs of fine-tuning.

\item \textbf{BadT.} We use 1 as our temperature for BadT.

\item \textbf{$\ell_1$-sparse.} We do 2 epochs with $\ell_1$ loss where we use 0.002 as our coefficient. Then, we do 48 epochs of fine-tuning.

\item \textbf{SSD.} We search over the space of dampening constants and selection weights. We got the best results for 0.5 for both dempening constant and and selection weights.

\item \textbf{SalUn.} We masked 20\% of the model weights for SalUn.

\item \textbf{NoT.} We negate only the weights of the first layer of ResNet-18 since it had the best performance. We tune the learning rate and select 0.1 as the starting point while reducing it 1e-5 in 10 epochs using cosine annealing.

\end{itemize}

%% file: main.bbl
\begin{thebibliography}{71}
\providecommand{\natexlab}[1]{#1}
\providecommand{\url}[1]{\texttt{#1}}
\expandafter\ifx\csname urlstyle\endcsname\relax
  \providecommand{\doi}[1]{doi: #1}\else
  \providecommand{\doi}{doi: \begingroup \urlstyle{rm}\Url}\fi

\bibitem[Amari(1997)]{amari1997information}
SI Amari.
\newblock Information geometry.
\newblock \emph{Contemporary Mathematics}, 203:\penalty0 81--96, 1997.

\bibitem[Anderson(1963)]{anderson1963asymptotic}
Theodore~Wilbur Anderson.
\newblock Asymptotic theory for principal component analysis.
\newblock \emph{The Annals of Mathematical Statistics}, 34\penalty0 (1):\penalty0 122--148, 1963.

\bibitem[Bachlechner et~al.(2021)Bachlechner, Majumder, Mao, Cottrell, and McAuley]{bachlechner2021rezero}
Thomas Bachlechner, Bodhisattwa~Prasad Majumder, Henry Mao, Gary Cottrell, and Julian McAuley.
\newblock Rezero is all you need: Fast convergence at large depth.
\newblock In \emph{Uncertainty in Artificial Intelligence}, pages 1352--1361. PMLR, 2021.

\bibitem[Burkholz and Dubatovka(2019)]{burkholz2019initialization}
Rebekka Burkholz and Alina Dubatovka.
\newblock Initialization of relus for dynamical isometry.
\newblock \emph{Advances in Neural Information Processing Systems}, 32, 2019.

\bibitem[Chen et~al.(2018)Chen, Pennington, and Schoenholz]{chen2018dynamical}
Minmin Chen, Jeffrey Pennington, and Samuel Schoenholz.
\newblock Dynamical isometry and a mean field theory of rnns: Gating enables signal propagation in recurrent neural networks.
\newblock In \emph{International Conference on Machine Learning}, pages 873--882. PMLR, 2018.

\bibitem[Chundawat et~al.(2023)Chundawat, Tarun, Mandal, and Kankanhalli]{chundawat2023can}
Vikram~S Chundawat, Ayush~K Tarun, Murari Mandal, and Mohan Kankanhalli.
\newblock Can bad teaching induce forgetting? unlearning in deep networks using an incompetent teacher.
\newblock \emph{Proceedings of the AAAI Conference on Artificial Intelligence}, 37\penalty0 (6):\penalty0 7210--7217, 2023.

\bibitem[Cohn(2013)]{cohn2013measure}
Donald~L Cohn.
\newblock \emph{Measure theory}.
\newblock Springer, 2013.

\bibitem[Deng et~al.(2024)Deng, Luo, and Chen]{deng2024enable}
Zhipeng Deng, Luyang Luo, and Hao Chen.
\newblock Enable the right to be forgotten with federated client unlearning in medical imaging.
\newblock \emph{arXiv preprint arXiv:2407.02356}, 2024.

\bibitem[Dosovitskiy(2020)]{dosovitskiy2020image}
Alexey Dosovitskiy.
\newblock An image is worth 16x16 words: Transformers for image recognition at scale.
\newblock \emph{arXiv preprint arXiv:2010.11929}, 2020.

\bibitem[Erhan et~al.(2009)Erhan, Manzagol, Bengio, Bengio, and Vincent]{pmlr-v5-erhan09a}
Dumitru Erhan, Pierre-Antoine Manzagol, Yoshua Bengio, Samy Bengio, and Pascal Vincent.
\newblock The difficulty of training deep architectures and the effect of unsupervised pre-training.
\newblock In \emph{Proceedings of the Twelfth International Conference on Artificial Intelligence and Statistics}, pages 153--160, Hilton Clearwater Beach Resort, Clearwater Beach, Florida USA, 2009. PMLR.

\bibitem[Erhan et~al.(2010)Erhan, Courville, Bengio, and Vincent]{erhan2010does}
Dumitru Erhan, Aaron Courville, Yoshua Bengio, and Pascal Vincent.
\newblock Why does unsupervised pre-training help deep learning?
\newblock In \emph{Proceedings of the thirteenth international conference on artificial intelligence and statistics}, pages 201--208. JMLR Workshop and Conference Proceedings, 2010.

\bibitem[EU(2023)]{union2023complete}
European~Union EU.
\newblock Complete guide to general data protection regulation compliance.
\newblock \emph{https://gdpr.eu/}, 2023.

\bibitem[Fan et~al.(2024)Fan, Liu, Zhang, Wong, Wei, and Liu]{fan2024salun}
Chongyu Fan, Jiancheng Liu, Yihua Zhang, Eric Wong, Dennis Wei, and Sijia Liu.
\newblock Salun: Empowering machine unlearning via gradient-based weight saliency in both image classification and generation.
\newblock In \emph{The Twelfth International Conference on Learning Representations}, 2024.

\bibitem[Foster et~al.(2024)Foster, Schoepf, and Brintrup]{foster2024fast}
Jack Foster, Stefan Schoepf, and Alexandra Brintrup.
\newblock Fast machine unlearning without retraining through selective synaptic dampening.
\newblock \emph{Proceedings of the AAAI Conference on Artificial Intelligence}, 38\penalty0 (11):\penalty0 12043--12051, 2024.

\bibitem[Futuyma(2013)]{FUTUYMA201370}
D.J. Futuyma.
\newblock Coevolution.
\newblock In \emph{Brenner's Encyclopedia of Genetics (Second Edition)}, pages 70--75. Academic Press, San Diego, second edition edition, 2013.

\bibitem[Gao et~al.(2024)Gao, Ma, Wang, Sun, Li, Ji, Cheng, and Chen]{gao2024verifi}
Xiangshan Gao, Xingjun Ma, Jingyi Wang, Youcheng Sun, Bo Li, Shouling Ji, Peng Cheng, and Jiming Chen.
\newblock Verifi: Towards verifiable federated unlearning.
\newblock \emph{IEEE Transactions on Dependable and Secure Computing}, 2024.

\bibitem[Geman(1980)]{geman1980limit}
Stuart Geman.
\newblock A limit theorem for the norm of random matrices.
\newblock \emph{The Annals of Probability}, 8\penalty0 (2):\penalty0 252--261, 1980.

\bibitem[Golatkar et~al.(2020{\natexlab{a}})Golatkar, Achille, and Soatto]{golatkar2020eternal}
Aditya Golatkar, Alessandro Achille, and Stefano Soatto.
\newblock Eternal sunshine of the spotless net: Selective forgetting in deep networks.
\newblock In \emph{Proceedings of the IEEE/CVF Conference on Computer Vision and Pattern Recognition}, pages 9304--9312, 2020{\natexlab{a}}.

\bibitem[Golatkar et~al.(2020{\natexlab{b}})Golatkar, Achille, and Soatto]{golatkar2020forgetting}
Aditya Golatkar, Alessandro Achille, and Stefano Soatto.
\newblock Forgetting outside the box: Scrubbing deep networks of information accessible from input-output observations.
\newblock In \emph{Computer Vision--ECCV 2020: 16th European Conference, Glasgow, UK, August 23--28, 2020, Proceedings, Part XXIX 16}, pages 383--398. Springer, 2020{\natexlab{b}}.

\bibitem[Gu et~al.(2019)Gu, Dolan-Gavitt, and Garg]{gu2019badnets}
Tianyu Gu, Brendan Dolan-Gavitt, and Siddharth Garg.
\newblock Badnets: Identifying vulnerabilities in the machine learning model supply chain, 2019.

\bibitem[Halimi et~al.(2022)Halimi, Kadhe, Rawat, and Baracaldo]{halimi2022federated}
Anisa Halimi, Swanand Kadhe, Ambrish Rawat, and Nathalie Baracaldo.
\newblock Federated unlearning: How to efficiently erase a client in fl?
\newblock \emph{arXiv preprint arXiv:2207.05521}, 2022.

\bibitem[He et~al.(2016)He, Zhang, Ren, and Sun]{he2016deep}
Kaiming He, Xiangyu Zhang, Shaoqing Ren, and Jian Sun.
\newblock Deep residual learning for image recognition.
\newblock In \emph{Proceedings of the IEEE conference on computer vision and pattern recognition}, pages 770--778, 2016.

\bibitem[He et~al.(2019)He, Girshick, and Doll{\'a}r]{he2019rethinking}
Kaiming He, Ross Girshick, and Piotr Doll{\'a}r.
\newblock Rethinking imagenet pre-training.
\newblock In \emph{Proceedings of the IEEE/CVF international conference on computer vision}, pages 4918--4927, 2019.

\bibitem[Hinton et~al.(2012)Hinton, Srivastava, Krizhevsky, Sutskever, and Salakhutdinov]{hinton2012improving}
Geoffrey~E Hinton, Nitish Srivastava, Alex Krizhevsky, Ilya Sutskever, and Ruslan~R Salakhutdinov.
\newblock Improving neural networks by preventing co-adaptation of feature detectors.
\newblock \emph{arXiv preprint arXiv:1207.0580}, 2012.

\bibitem[Hochreiter(1998)]{hochreiter1998vanishing}
Sepp Hochreiter.
\newblock The vanishing gradient problem during learning recurrent neural nets and problem solutions.
\newblock \emph{International Journal of Uncertainty, Fuzziness and Knowledge-Based Systems}, 6\penalty0 (02):\penalty0 107--116, 1998.

\bibitem[Hu et~al.(2022)Hu, Salcic, Sun, Dobbie, Yu, and Zhang]{hu2022membership}
Hongsheng Hu, Zoran Salcic, Lichao Sun, Gillian Dobbie, Philip~S Yu, and Xuyun Zhang.
\newblock Membership inference attacks on machine learning: A survey.
\newblock \emph{ACM Computing Surveys (CSUR)}, 54\penalty0 (11s):\penalty0 1--37, 2022.

\bibitem[Jacot-Guillarmod(2022)]{jacot2022theory}
Arthur~Ulysse Jacot-Guillarmod.
\newblock Theory of deep learning: Neural tangent kernel and beyond.
\newblock Technical report, EPFL, 2022.

\bibitem[Jia et~al.(2023)Jia, Liu, Ram, Yao, Liu, Liu, Sharma, and Liu]{jia2023model}
Jinghan Jia, Jiancheng Liu, Parikshit Ram, Yuguang Yao, Gaowen Liu, Yang Liu, Pranay Sharma, and Sijia Liu.
\newblock Model sparsity can simplify machine unlearning.
\newblock In \emph{Thirty-seventh Conference on Neural Information Processing Systems}, 2023.

\bibitem[Karakida et~al.(2019)Karakida, Akaho, and Amari]{pmlr-v89-karakida19a}
Ryo Karakida, Shotaro Akaho, and Shun-ichi Amari.
\newblock Universal statistics of fisher information in deep neural networks: Mean field approach.
\newblock In \emph{Proceedings of the Twenty-Second International Conference on Artificial Intelligence and Statistics}, pages 1032--1041. PMLR, 2019.

\bibitem[Karoui(2007)]{101214009117906000000917}
Noureddine~El Karoui.
\newblock {Tracy–Widom limit for the largest eigenvalue of a large class of complex sample covariance matrices}.
\newblock \emph{The Annals of Probability}, 35\penalty0 (2):\penalty0 663 -- 714, 2007.

\bibitem[Karoui(2008)]{10.1214/07-AOS581}
Noureddine~El Karoui.
\newblock {Spectrum estimation for large dimensional covariance matrices using random matrix theory}.
\newblock \emph{The Annals of Statistics}, 36\penalty0 (6):\penalty0 2757 -- 2790, 2008.

\bibitem[Khalil et~al.(2024)Khalil, Estiri, Beitollahi, Asadi, Hemati, Li, Zhang, and Chen]{khalil2024dfml}
Yasser~H Khalil, Amir~H Estiri, Mahdi Beitollahi, Nader Asadi, Sobhan Hemati, Xu Li, Guojun Zhang, and Xi Chen.
\newblock Dfml: Decentralized federated mutual learning.
\newblock \emph{arXiv preprint arXiv:2402.01863}, 2024.

\bibitem[Kim and Ahn(2000)]{10.1007/978-1-4471-0509-1_20}
Daijin Kim and Sunha Ahn.
\newblock An optimal vq codebook design using the co-adaptation of learning and evolution.
\newblock In \emph{Soft Computing in Industrial Applications}, pages 225--239, London, 2000. Springer London.

\bibitem[Kornblith et~al.(2019)Kornblith, Norouzi, Lee, and Hinton]{pmlr-v97-kornblith19a}
Simon Kornblith, Mohammad Norouzi, Honglak Lee, and Geoffrey Hinton.
\newblock Similarity of neural network representations revisited.
\newblock In \emph{Proceedings of the 36th International Conference on Machine Learning}, pages 3519--3529. PMLR, 2019.

\bibitem[Krizhevsky et~al.(2014)Krizhevsky, Nair, Hinton, et~al.]{krizhevsky2014cifar}
Alex Krizhevsky, Vinod Nair, Geoffrey Hinton, et~al.
\newblock The cifar-10 dataset.
\newblock \emph{online: http://www. cs. toronto. edu/kriz/cifar. html}, 55\penalty0 (5):\penalty0 2, 2014.

\bibitem[Li et~al.(2022)Li, Andreeto, Ranzato, and Perona]{li_andreeto_ranzato_perona_2022}
Fei-Fei Li, Marco Andreeto, Marc'Aurelio Ranzato, and Pietro Perona.
\newblock Caltech 101, 2022.

\bibitem[Liao et~al.(2018)Liao, Drummond, Reid, and Carneiro]{liao2018approximate}
Zhibin Liao, Tom Drummond, Ian Reid, and Gustavo Carneiro.
\newblock Approximate fisher information matrix to characterize the training of deep neural networks.
\newblock \emph{IEEE transactions on pattern analysis and machine intelligence}, 42\penalty0 (1):\penalty0 15--26, 2018.

\bibitem[Liu et~al.(2021)Liu, Ma, Yang, Wang, and Liu]{liu2021federaser}
Gaoyang Liu, Xiaoqiang Ma, Yang Yang, Chen Wang, and Jiangchuan Liu.
\newblock Federaser: Enabling efficient client-level data removal from federated learning models.
\newblock In \emph{2021 IEEE/ACM 29th international symposium on quality of service (IWQOS)}, pages 1--10. IEEE, 2021.

\bibitem[Liu et~al.(2023)Liu, Jiang, Shen, Peng, Lam, and Yuan]{liu2023survey}
Ziyao Liu, Yu Jiang, Jiyuan Shen, Minyi Peng, Kwok-Yan Lam, and Xingliang Yuan.
\newblock A survey on federated unlearning: Challenges, methods, and future directions.
\newblock \emph{arXiv preprint arXiv:2310.20448}, 2023.

\bibitem[Liu et~al.(2024)Liu, Ye, Chen, and Lam]{liu2024threats}
Ziyao Liu, Huanyi Ye, Chen Chen, and Kwok-Yan Lam.
\newblock Threats, attacks, and defenses in machine unlearning: A survey.
\newblock \emph{arXiv preprint arXiv:2403.13682}, 2024.

\bibitem[Merel et~al.(2013)Merel, Fox, Jebara, and Paninski]{merel2013multi}
Josh~S Merel, Roy Fox, Tony Jebara, and Liam Paninski.
\newblock A multi-agent control framework for co-adaptation in brain-computer interfaces.
\newblock \emph{Advances in Neural Information Processing Systems}, 26, 2013.

\bibitem[Mingo and Speicher(2017)]{mingo2017free}
James~A Mingo and Roland Speicher.
\newblock \emph{Free probability and random matrices}.
\newblock Springer, 2017.

\bibitem[Nguyen et~al.(2021)Nguyen, Ding, Pathirana, Seneviratne, Li, and Poor]{nguyen2021federated}
Dinh~C Nguyen, Ming Ding, Pubudu~N Pathirana, Aruna Seneviratne, Jun Li, and H~Vincent Poor.
\newblock Federated learning for internet of things: A comprehensive survey.
\newblock \emph{IEEE Communications Surveys \& Tutorials}, 23\penalty0 (3):\penalty0 1622--1658, 2021.

\bibitem[Nicolae et~al.(2019)Nicolae, Sinn, Tran, Buesser, Rawat, Wistuba, Zantedeschi, Baracaldo, Chen, Ludwig, Molloy, and Edwards]{nicolae2019adversarial}
Maria-Irina Nicolae, Mathieu Sinn, Minh~Ngoc Tran, Beat Buesser, Ambrish Rawat, Martin Wistuba, Valentina Zantedeschi, Nathalie Baracaldo, Bryant Chen, Heiko Ludwig, Ian~M. Molloy, and Ben Edwards.
\newblock Adversarial robustness toolbox v1.0.0, 2019.

\bibitem[Noci et~al.(2022)Noci, Anagnostidis, Biggio, Orvieto, Singh, and Lucchi]{noci2022signal}
Lorenzo Noci, Sotiris Anagnostidis, Luca Biggio, Antonio Orvieto, Sidak~Pal Singh, and Aurelien Lucchi.
\newblock Signal propagation in transformers: Theoretical perspectives and the role of rank collapse.
\newblock \emph{Advances in Neural Information Processing Systems}, 35:\penalty0 27198--27211, 2022.

\bibitem[Pan et~al.(2024)Pan, Wang, Li, Zheng, Wang, Tang, and Zhao]{pan2024federated}
Zibin Pan, Zhichao Wang, Chi Li, Kaiyan Zheng, Boqi Wang, Xiaoying Tang, and Junhua Zhao.
\newblock Federated unlearning with gradient descent and conflict mitigation.
\newblock \emph{arXiv preprint arXiv:2412.20200}, 2024.

\bibitem[Pennington et~al.(2017)Pennington, Schoenholz, and Ganguli]{pennington2017resurrecting}
Jeffrey Pennington, Samuel Schoenholz, and Surya Ganguli.
\newblock Resurrecting the sigmoid in deep learning through dynamical isometry: theory and practice.
\newblock \emph{Advances in neural information processing systems}, 30, 2017.

\bibitem[Rattray et~al.(1998)Rattray, Saad, and Amari]{rattray1998natural}
Magnus Rattray, David Saad, and Shun-ichi Amari.
\newblock Natural gradient descent for on-line learning.
\newblock \emph{Physical review letters}, 81\penalty0 (24):\penalty0 5461, 1998.

\bibitem[Ren et~al.(2024)Ren, Yu, Peng, Tang, Li, Gao, Tan, Zhao, Li, Li, et~al.]{ren2024advances}
Chao Ren, Han Yu, Hongyi Peng, Xiaoli Tang, Anran Li, Yulan Gao, Alysa~Ziying Tan, Bo Zhao, Xiaoxiao Li, Zengxiang Li, et~al.
\newblock Advances and open challenges in federated learning with foundation models.
\newblock \emph{arXiv preprint arXiv:2404.15381}, 2024.

\bibitem[Revuz and Yor(2013)]{revuz2013continuous}
Daniel Revuz and Marc Yor.
\newblock \emph{Continuous martingales and Brownian motion}.
\newblock Springer Science \& Business Media, 2013.

\bibitem[Romandini et~al.(2024)Romandini, Mora, Mazzocca, Montanari, and Bellavista]{romandini2024federated}
Nicol{\`o} Romandini, Alessio Mora, Carlo Mazzocca, Rebecca Montanari, and Paolo Bellavista.
\newblock Federated unlearning: A survey on methods, design guidelines, and evaluation metrics.
\newblock \emph{arXiv preprint arXiv:2401.05146}, 2024.

\bibitem[Russakovsky et~al.(2015)Russakovsky, Deng, Su, Krause, Satheesh, Ma, Huang, Karpathy, Khosla, Bernstein, et~al.]{russakovsky2015imagenet}
Olga Russakovsky, Jia Deng, Hao Su, Jonathan Krause, Sanjeev Satheesh, Sean Ma, Zhiheng Huang, Andrej Karpathy, Aditya Khosla, Michael Bernstein, et~al.
\newblock Imagenet large scale visual recognition challenge.
\newblock \emph{International journal of computer vision}, 115:\penalty0 211--252, 2015.

\bibitem[Sato et~al.(2019)Sato, Ishikawa, Liu, and Tanaka]{sato2019breaking}
Ikuro Sato, Kohta Ishikawa, Guoqing Liu, and Masayuki Tanaka.
\newblock Breaking inter-layer co-adaptation by classifier anonymization.
\newblock \emph{arXiv preprint arXiv:1906.01150}, 2019.

\bibitem[Saxe et~al.(2013)Saxe, McClelland, and Ganguli]{saxe2013exact}
Andrew~M Saxe, James~L McClelland, and Surya Ganguli.
\newblock Exact solutions to the nonlinear dynamics of learning in deep linear neural networks.
\newblock \emph{arXiv preprint arXiv:1312.6120}, 2013.

\bibitem[Shi et~al.(2012)Shi, Tu, Zhang, Liu, and Wei]{shi2012survey}
Zhiguo Shi, Jun Tu, Qiao Zhang, Lei Liu, and Junming Wei.
\newblock A survey of swarm robotics system.
\newblock In \emph{Advances in Swarm Intelligence: Third International Conference, ICSI 2012, Shenzhen, China, June 17-20, 2012 Proceedings, Part I 3}, pages 564--572. Springer, 2012.

\bibitem[Stephan et~al.(2017)Stephan, Hoffman, Blei, et~al.]{stephan2017stochastic}
Mandt Stephan, Matthew~D Hoffman, David~M Blei, et~al.
\newblock Stochastic gradient descent as approximate bayesian inference.
\newblock \emph{Journal of Machine Learning Research}, 18\penalty0 (134):\penalty0 1--35, 2017.

\bibitem[Tarnowski et~al.(2019)Tarnowski, Warcho{\l}, Jastrzobski, Tabor, and Nowak]{tarnowski2019dynamical}
Wojciech Tarnowski, Piotr Warcho{\l}, Stanis{\l}aw Jastrzobski, Jacek Tabor, and Maciej Nowak.
\newblock Dynamical isometry is achieved in residual networks in a universal way for any activation function.
\newblock In \emph{The 22nd International Conference on Artificial Intelligence and Statistics}, pages 2221--2230. PMLR, 2019.

\bibitem[Tarun et~al.(2023)Tarun, Chundawat, Mandal, and Kankanhalli]{tarun2023fast}
Ayush~K Tarun, Vikram~S Chundawat, Murari Mandal, and Mohan Kankanhalli.
\newblock Fast yet effective machine unlearning.
\newblock \emph{IEEE Transactions on Neural Networks and Learning Systems}, 2023.

\bibitem[Thudi et~al.(2022)Thudi, Deza, Chandrasekaran, and Papernot]{thudi2022unrolling}
Anvith Thudi, Gabriel Deza, Varun Chandrasekaran, and Nicolas Papernot.
\newblock Unrolling sgd: Understanding factors influencing machine unlearning.
\newblock In \emph{2022 IEEE 7th European Symposium on Security and Privacy (EuroS\&P)}, pages 303--319. IEEE, 2022.

\bibitem[Villani(2021)]{villani2021topics}
C{\'e}dric Villani.
\newblock \emph{Topics in optimal transportation}.
\newblock American Mathematical Soc., 2021.

\bibitem[Wan et~al.(2024)Wan, Qu, Ni, Xiang, Gao, and Hossain]{wan2024data}
Yichen Wan, Youyang Qu, Wei Ni, Yong Xiang, Longxiang Gao, and Ekram Hossain.
\newblock Data and model poisoning backdoor attacks on wireless federated learning, and the defense mechanisms: A comprehensive survey.
\newblock \emph{IEEE Communications Surveys \& Tutorials}, 2024.

\bibitem[Wang et~al.(2022)Wang, Cai, Han, Zhou, and Gong]{wang2022stnet}
Mingjie Wang, Hao Cai, Xian-Feng Han, Jun Zhou, and Minglun Gong.
\newblock Stnet: Scale tree network with multi-level auxiliator for crowd counting.
\newblock \emph{IEEE Transactions on Multimedia}, 25:\penalty0 2074--2084, 2022.

\bibitem[Wu et~al.(2022)Wu, Zhu, and Mitra]{wu2022federated}
Chen Wu, Sencun Zhu, and Prasenjit Mitra.
\newblock Federated unlearning with knowledge distillation.
\newblock \emph{arXiv preprint arXiv:2201.09441}, 2022.

\bibitem[Xiao et~al.(2018)Xiao, Bahri, Sohl-Dickstein, Schoenholz, and Pennington]{pmlr-v80-xiao18a}
Lechao Xiao, Yasaman Bahri, Jascha Sohl-Dickstein, Samuel Schoenholz, and Jeffrey Pennington.
\newblock Dynamical isometry and a mean field theory of {CNN}s: How to train 10,000-layer vanilla convolutional neural networks.
\newblock In \emph{Proceedings of the 35th International Conference on Machine Learning}, pages 5393--5402. PMLR, 2018.

\bibitem[Xu et~al.(2024)Xu, Wu, Wang, and Jia]{xu2024machine}
Jie Xu, Zihan Wu, Cong Wang, and Xiaohua Jia.
\newblock Machine unlearning: Solutions and challenges.
\newblock \emph{IEEE Transactions on Emerging Topics in Computational Intelligence}, 2024.

\bibitem[Yao et~al.(2021)Yao, Yu, Gong, and Liu]{yao2021understanding}
Yu Yao, Baosheng Yu, Chen Gong, and Tongliang Liu.
\newblock Understanding how pretraining regularizes deep learning algorithms.
\newblock \emph{IEEE Transactions on Neural Networks and Learning Systems}, 34\penalty0 (9):\penalty0 5828--5840, 2021.

\bibitem[Yin et~al.(1988)Yin, Bai, and Krishnaiah]{yin1988limit}
Yong-Qua Yin, Zhi-Dong Bai, and Pathak~R Krishnaiah.
\newblock On the limit of the largest eigenvalue of the large dimensional sample covariance matrix.
\newblock \emph{Probability theory and related fields}, 78:\penalty0 509--521, 1988.

\bibitem[Yosinski et~al.(2014)Yosinski, Clune, Bengio, and Lipson]{yosinski2014transferable}
Jason Yosinski, Jeff Clune, Yoshua Bengio, and Hod Lipson.
\newblock How transferable are features in deep neural networks?
\newblock \emph{Advances in neural information processing systems}, 27, 2014.

\bibitem[Zhang et~al.(2023)Zhang, Zhu, Zhang, Xiong, and Zhou]{zhang2023fedrecovery}
Lefeng Zhang, Tianqing Zhu, Haibin Zhang, Ping Xiong, and Wanlei Zhou.
\newblock Fedrecovery: Differentially private machine unlearning for federated learning frameworks.
\newblock \emph{IEEE Transactions on Information Forensics and Security}, 18:\penalty0 4732--4746, 2023.

\bibitem[Zhang et~al.(2022)Zhang, Gao, He, Zhang, Krishnamachari, and Avestimehr]{zhang2022federated}
Tuo Zhang, Lei Gao, Chaoyang He, Mi Zhang, Bhaskar Krishnamachari, and A~Salman Avestimehr.
\newblock Federated learning for the internet of things: Applications, challenges, and opportunities.
\newblock \emph{IEEE Internet of Things Magazine}, 5\penalty0 (1):\penalty0 24--29, 2022.

\bibitem[Zhao et~al.(2023)Zhao, Wang, Qi, Huang, Wei, and Zhang]{zhao2023federated}
Yian Zhao, Pengfei Wang, Heng Qi, Jianguo Huang, Zongzheng Wei, and Qiang Zhang.
\newblock Federated unlearning with momentum degradation.
\newblock \emph{IEEE Internet of Things Journal}, 2023.

\end{thebibliography}
